\documentclass[11pt]{article}

\usepackage[margin=1in]{geometry}
\usepackage{natbib}
\usepackage{authblk}
\usepackage[T1]{fontenc}    % use 8-bit T1 fonts
\usepackage{hyperref}       % hyperlinks
\usepackage{url}            % simple URL typesetting
\usepackage{booktabs}       % professional-quality tables
\usepackage{amsfonts}       % blackboard math symbols
\usepackage{nicefrac}       % compact symbols for 1/2, etc.
\usepackage{microtype}      % microtypography
\usepackage{xcolor}         % colors
\usepackage{amsmath}
\usepackage{amssymb,amsthm}
\usepackage{mathtools}
\usepackage{algorithm}
\usepackage{algpseudocode}
\makeatletter
\algblockdefx[Sub]{Subroutine}{EndSubroutine}[2]%
  {\textbf{Subroutine} #1(#2)}
  {\textbf{End Subroutine}}   
\makeatother
\usepackage[font=small]{subcaption}
\usepackage{graphicx}

\usepackage{amsmath,amsfonts,bm}

\def\eqref#1{equation~\ref{#1}}
\def\1{\bm{1}}

\DeclareMathAlphabet{\mathsfit}{\encodingdefault}{\sfdefault}{m}{sl}
\SetMathAlphabet{\mathsfit}{bold}{\encodingdefault}{\sfdefault}{bx}{n}

\DeclareMathOperator*{\argmax}{arg\,max}

\usepackage{authblk}
\usepackage{comment}
\usepackage{thmtools}
\usepackage{longtable}
\usepackage{threeparttable}
\usepackage{enumitem}
\usepackage[T1]{fontenc}
\usepackage{tikz}
\usepackage{subcaption}

\definecolor{PhiShade}{HTML}{7C3AED}
\definecolor{OTShade}{HTML}{D97706}

\theoremstyle{plain}
\newtheorem{theorem}{Theorem}[section]
\newtheorem{proposition}[theorem]{Proposition}
\newtheorem{lemma}[theorem]{Lemma}

\theoremstyle{definition}
\newtheorem{definition}[theorem]{Definition}

\theoremstyle{remark}
\newtheorem{remark}[theorem]{Remark}

\newcommand{\StartBlock}[1]{%
  \tikz[remember picture,overlay] \coordinate (start#1) at (0, 1.9ex);%
}
\newcommand{\EndBlock}[2]{%
  \tikz[remember picture,overlay] \coordinate (end#1) at (0, -0.6ex);%
  \begin{tikzpicture}[remember picture, overlay]
    \fill[#2, opacity=0.15] 
      ([xshift=-2em]start#1) % Left edge
      rectangle 
      ([xshift=\linewidth-0em]start#1 |- end#1); % Right edge (Locked to Start's X!)
  \end{tikzpicture}%
}

\algdef{SE}[SUBALG]{Indent}{EndIndent}{}{}
\algtext*{Indent}
\algtext*{EndIndent}

\graphicspath{{figures/}}

\makeatletter
\renewcommand{\eqref}[1]{\hyperref[#1]{(\ref*{#1})}}
\makeatother

\title{Adversarial Training for Deep Hedging in Nonstationary Markets}

\author{%
  Philipp J.~Schneider\textsuperscript{$\star$,1}\quad
  Lukas Looser\textsuperscript{$\star$,1}\quad
  Antoine Garin\textsuperscript{$\dagger$,1}\quad
  Shuhan Liu\textsuperscript{$\dagger$,2}\quad
  Daniel Kuhn\textsuperscript{1}%
  \par\vspace{1em}
  \textsuperscript{1}Risk Analytics and Optimization Chair, EPFL, Lausanne, Switzerland\\
  \textsuperscript{2}University of Waterloo, Waterloo, Canada%
}

\allowdisplaybreaks

\begin{document}
\maketitle
\addtocontents{toc}{\protect\setcounter{tocdepth}{-1}}

\renewcommand{\thefootnote}{\fnsymbol{footnote}}\footnotetext[1]{These authors contributed equally to this work.}\footnotetext[2]{These authors contributed equally to this work.}\renewcommand{\thefootnote}{}\footnotetext[3]{Corresponding author: \texttt{philipp.schneider@epfl.ch}}\renewcommand{\thefootnote}{\arabic{footnote}}

\begin{abstract}
Deep hedging learns trading policies from historical or simulated market trajectories, yet under nonstationarity these training paths may not represent future market conditions. We propose \textsc{WRAP} (Wasserstein--Reweighting Adversarial Perturbation), a drift-aware adversarial training framework derived from a two-budget distributionally robust optimization (DRO) formulation. The formulation is anchored to a weighted empirical reference distribution whose fixed baseline weights are chosen to balance sampling uncertainty against temporal drift. Around this reference distribution, the ambiguity set addresses two complementary forms of distributional misspecification by allowing an adversary to reweight the observed trajectories subject to a $\phi$-divergence constraint and perturb their paths subject to an optimal-transport (OT) constraint. We derive a joint first-order expansion in which the leading-order increase over the nominal expected loss decomposes into a reweighting contribution determined by the dispersion of hedging losses across trajectories and a transport contribution determined by the sensitivity of the loss to path perturbations. This expansion yields an explicit finite-dimensional adversarial attack that replaces the distributional inner supremum with a tractable first-order approximation. Across stationary and nonstationary Heston dynamics and a generalized affine diffusion (GAD), the experiments show complementary benefits from reweighting and transport, with joint adversarial training providing the largest gains under nonstationarity.
\end{abstract}

\section{Introduction}
\label{sec:introduction}

A hedging strategy dynamically trades financial assets to reduce the risk associated with a contingent claim. Deep hedging parameterizes such strategies with neural networks and learns the resulting sequence of trading decisions from simulated or historical market trajectories \citep{buehler2019deep}. The out-of-sample performance of the learned policy therefore depends critically on whether the training distribution represents the market conditions encountered over the hedging horizon. Unlike a static prediction, a hedge consists of a sequence of state-dependent trading decisions, so changes in market dynamics may affect positions throughout the life of the claim \citep{cont2006model,glasserman2014robust}. A policy may consequently achieve a low in-sample hedging loss yet remain vulnerable to market dynamics poorly represented in the training data.

Distributionally robust optimization (DRO) and adversarial training protect a policy against perturbations around a prespecified empirical distribution. Maintaining a representative training distribution is particularly difficult under nonstationarity, where returns, volatility, and dependence structures evolve over time \citep{stuaricua2005nonstationarities,ang2012regime,capponi2025nonstationarity}. A larger historical sample reduces sampling uncertainty but also includes older observations that may be less representative of the next hedging period, whereas emphasizing recent observations improves temporal relevance at the cost of greater sampling variability. Constructing the reference distribution entails a fundamental trade-off between \emph{statistical precision} and \emph{temporal relevance}. Robust deep hedging under nonstationarity must therefore address two questions: \emph{how should the reference distribution be constructed from temporally heterogeneous observations}, and \emph{how should the hedge be protected against residual misspecification around that distribution}?

We distinguish three channels through which the training distribution may fail to represent the next hedging period. First, historical trajectories differ in their temporal relevance, as older observations may be less informative about future market dynamics. Second, the reference distribution may assign inappropriate probability mass to the sampled trajectories or the regimes they represent. Third, future trajectories may differ from those observed in the training sample. We address these channels through \emph{non-uniform baseline weights}, \emph{adversarial probability reweighting}, and \emph{optimal-transport perturbations}, respectively. The baseline weights are selected independently of the hedging losses and determine the drift-aware center of the ambiguity set. Adversarial reweighting depends on the current policy and reallocates probability mass toward trajectories on which that policy incurs high losses, whereas optimal-transport perturbations move the trajectories themselves in locally adverse directions. These three components constitute our main contributions, summarized as follows.

\noindent\textbf{Drift-aware reference distributions for deep hedging.}
We provide a statistical foundation for constructing reference distributions from nonstationary market data. Our finite-sample guarantees separate sampling uncertainty, governed by the effective sample size induced by the baseline weights, from temporal misspecification, governed by a weighted drift index that captures the influence assigned to older observations. This decomposition formalizes the trade-off between statistical precision and temporal relevance and provides a basis for calibrating the Wasserstein radius to both weight concentration and temporal drift.

\noindent\textbf{First-order expansion of the robust objective.}
We derive a joint first-order expansion of the worst-case hedging loss under probability and trajectory misspecification. The expansion shows that the $\phi$-divergence component penalizes the dispersion of hedging losses across trajectories (\emph{variance}), whereas the optimal-transport component penalizes the sensitivity of the hedging loss to perturbations of market paths (\emph{variation}). We further show that a structured empirical attack attains the same first-order value as the full two-budget robust problem.

\noindent\textbf{\textsc{WRAP}: two-budget distributional adversarial training.}
Building on this expansion, we introduce the Wasserstein--Reweighting Adversarial Perturbation (\textsc{WRAP}), an implementable training method for the two-budget robust problem. The two budgets govern complementary deviations from the drift-aware reference distribution: adversarial probability reweighting and trajectory perturbations. The latter are defined through a time- and component-weighted ground metric that permits separate scaling across trading dates and market variables. We derive a one-step attack from the first-order solution and study a multistep variant that iteratively refines the adversarial distribution.

\noindent Existing approaches address individual components of this problem. Closest to our setting, \citet{he2025distributional} develop Wasserstein attacks for deep hedging, \citet{keehan2025don} construct drift-aware Wasserstein sets around weighted empirical laws, and \citet{namkoong2016stochastic} use \(\phi\)-divergence DRO to adversarially reweight fixed empirical observations.

Several approaches combine likelihood and optimal-transport perturbations, including holistic robustness, joint OT--divergence constructions, and OT-regularized divergences \citep{bennouna2022holistic,blanchet2023unifying,liu2023smoothed,birrell2025optimal}. We derive a joint first-order expansion for a sequential, path-dependent objective and construct an explicit transport-and-reweighting attack that is first-order accurate for the full DRO problem. Appendix~\ref{app:extended-related-work} provides the extended literature discussion.

\section{Deep Hedging Under Nonstationarity}
\label{sec:short-setup}
We consider a discrete-time financial market with trading dates \(t=0,\dots,T\). The market contains \(r\) tradable assets with price process \(S=(S_t)_{t=0}^T\), where \(S_t\in\mathbb R^r\). The observed market process \(X=(X_t)_{t=0}^T\) is a random element of an admissible trajectory space $\Xi\subseteq\mathbb R^{d(T+1)}$. Here \(X_t\in\mathbb R^d\) contains the tradable asset prices together with any exogenous market features used for hedging, such as volatility estimates or macroeconomic indicators; in particular, \(d\ge r\). We take \(S_t\) to be the first \(r\) components of \(X_t\) and write \(X_{0:t}\coloneqq(X_0,\dots,X_t)\). Given \(N\) observed or simulated market trajectories, we denote the \(n\)-th trajectory by \(X_n=(X_{n,0},\dots,X_{n,T})\in\Xi\).

The agent aims to hedge a contingent claim with terminal payoff function \(\Phi:\Xi\to\mathbb R\) at maturity. Let \(\theta=(\theta_0,\dots,\theta_{T-1})\in\Theta\subseteq\mathbb R^{kT}\) parametrize \(T\) neural networks. At time \(t=0,\dots,T-1\), the network parametrized by \(\theta_t\) maps \(X_{0:t}\) to holdings \(\delta_t^\theta(X_{0:t})\in\mathbb R^r\) and thus depends on $\theta$ only through $\theta_t$. Since each decision depends only on the observed history \(X_{0:t}\), the strategy is \emph{non-anticipative} and fully characterized by \(\theta\in\Theta\). Assuming a self-financing strategy and a fixed initial premium $p_0\in\mathbb{R}$ for selling the derivative, the terminal \emph{profit and loss} (PnL) is given by
\begin{equation}
\label{eq:pnl}
\mathrm{PnL}(\theta,X)
\;=\;
p_0 +\sum_{t=0}^{T-1}\delta_t^\theta(X_{0:t})^\top (S_{t+1}-S_t) - \Phi(X).
\end{equation}
Following \citet{buehler2019deep}, the risk of the terminal profit and loss $Z$ is assessed through an Optimized Certainty Equivalent (OCE) risk measure \citep{ben1986expected,ben2007old},
\begin{equation}
\label{eq:oce}
\rho(Z)
\;=\;
\inf_{m\in\mathbb{R}}
\left\{
m+\mathbb{E}\big[\ell(-Z-m)\big]
\right\},
\end{equation}
where \(\ell:\mathbb R\to\mathbb R\) is a convex, nondecreasing loss function satisfying $\ell(0)=0$ and 
$\ell(u)\ge u$ for all $u\in\mathbb{R}$. To train a hedging policy, we optimize jointly over the hedging network parameters and the OCE offset by setting $\vartheta=(\theta,m)\in\Theta\times\mathbb{R}$ and defining the pathwise deep-hedging loss
\begin{equation}
\label{eq:dh_loss}
\ell_{\mathrm{DH}}(\vartheta; X)
\;\coloneqq\;
m+\ell\!\left(-\mathrm{PnL}(\theta, X)-m\right).
\end{equation}
The numerical experiments consider Conditional Value at Risk (CVaR) and entropic risk (\textit{cf.}~Appendix~\ref{app:oce-risk-measures}). Let \(\mathcal P(\Xi)\) denote the set of Borel probability measures on \(\Xi\). For a reference measure \(\mu\in\mathcal P(\Xi)\), we train a nominal deep-hedging model by optimizing
\begin{equation}
\label{eq:deep-hedging-problem}
\min_{\vartheta\in\Theta\times\mathbb R}
\mathbb E_{X\sim\mu}
\big[\ell_{\mathrm{DH}}(\vartheta;X)\big].
\end{equation} 
Further details on the market setup and regularity are provided in Appendix~\ref{app:additional-market-regularity-details}.

\subsection{Trajectory Space and Wasserstein Robust Value}
The training problem in \eqref{eq:deep-hedging-problem} evaluates the hedge under a fixed reference law \(\mu\). To account for distribution shifts and sampling error, we replace this fixed law with a Wasserstein ambiguity set centered at \(\mu\). Specifying the Wasserstein ambiguity set requires a transportation cost between market trajectories. We therefore equip the admissible trajectory set \(\Xi\subseteq\mathbb R^{d(T+1)}\) with an anisotropic path metric that allows these costs to vary across market variables and trading dates. Let \(p\in(1,\infty)\), fix time weights \(\kappa_t>0\) for \(t=0,\dots,T\), and component scales \(\lambda_i>0\) for \(i=1,\dots,d\). For an observed trajectory \(X\in\Xi\) and a perturbed trajectory \(\widehat X\in\Xi\), define
\begin{equation}
\label{eq:kappa-lambda-metric}
d_{\kappa,\lambda}(X,\widehat X)
\coloneqq
\Bigg(
\sum_{i=1}^d
\Big(\lambda_i\,\max_{t=0,\dots,T}\kappa_t
\,|X_t^i-\widehat X_t^i|\Big)^p
\Bigg)^{1/p}.
\end{equation}
The inner \(\ell_\infty\) term measures the largest weighted displacement over time within each component, while the outer \(\ell_p\) norm aggregates these component-wise displacements. The scales \(\lambda_i\) and weights \(\kappa_t\) control how perturbations across market variables and trading dates enter the transportation cost. We impose the following condition on the resulting metric space, which ensures finite \(p\)-th moments for all probability laws on \(\Xi\) with respect to \(d_{\kappa,\lambda}\).

\begin{restatable}[Path space]{assumption}{CoveragePathSpace}
\label{ass:path-space}
The admissible trajectory set \(\Xi\), equipped with the metric \(d_{\kappa,\lambda}\) in \eqref{eq:kappa-lambda-metric}, is a Polish metric space with finite diameter: $
\operatorname{diam}_{\kappa,\lambda}(\Xi)
\coloneqq
\sup_{X,\widehat X\in\Xi} d_{\kappa,\lambda}(X,\widehat X)
<\infty.$
\end{restatable}
For $p\in (1,\infty]$, let $W_p$ denote the $p$-Wasserstein distance induced by $d_{\kappa,\lambda}$. For a reference law \(\mu\in\mathcal P(\Xi)\) and radius $\varepsilon\ge0$, define the Wasserstein ambiguity set as the ball $\mathcal B_\varepsilon(\mu)\!\coloneqq\!\{\eta\in\mathcal P(\Xi):W_p(\eta,\mu)\le\varepsilon\}$. The Wasserstein robust value $V_\vartheta$ and the associated deep-hedging problem are
\begin{equation}
\label{eq:short-wasserstein-budget-dro}
\inf_{\vartheta\in\Theta\times\mathbb R} \bigg\{V_{\vartheta}(\varepsilon) \coloneqq 
\sup_{\eta\in\mathcal B_{\varepsilon}(\mu)}
\mathbb E_{X\sim\eta}
\big[\ell_{\mathrm{DH}}(\vartheta;X)\big]\bigg\}.
\end{equation}

\subsection{Drift-Aware Coverage of the Next-Period Market Law}
\label{subsec:short-coverage}
Under temporal nonstationarity, the reference law and Wasserstein radius should reflect the evolution of the data-generating distribution. Let \(X_1,\ldots,X_N\in\Xi\) denote historical trajectories ordered from oldest to most recent, with \(X_n\sim P_n \in\mathcal{P}(\Xi)\), and let \(P_{N+1}\) denote the unknown next-period market law. We model temporal drift through a uniform bound \(\varrho\) on one-step changes in the market law.

\begin{restatable}[Nonstationary sampling and bounded drift]{assumption}{NonstationarySamplingDrift}
\label{ass:nonstationary-drift}
The trajectories \(X_1,\dots,X_N\) are independent, with \(X_n\sim P_n\) for \(n=1,\dots,N\). Moreover, there exists \(\varrho\ge0\) such that
\[
W_\infty(P_n,P_{n+1})\le\varrho,
\qquad n=1,\dots,N.
\]
\end{restatable}

In applications, \(\varrho\) may be calibrated using regime-switching models, change-point methods, or rolling estimates of distributional distance. Under this drift model, older observations may be less representative of the next-period law because they are separated from \(P_{N+1}\) by more distributional shifts. The uniform empirical measure \(\mu\!\coloneqq\!\frac{1}{N}\sum_{n=1}^N\delta_{X_n}\) does not account for this difference in recency. We therefore use the weighted empirical reference measure $\mu_w\!\coloneqq\!\sum_{n=1}^Nw_n\delta_{X_n}$, where $\delta_x$ denotes the Dirac measure at $x$ and $w\!\in\!\Delta_N\!\coloneqq\!\{w\!\in\!\mathbb R_+^N:\sum_{n=1}^Nw_n\!=\!1\}$. The weights $w_n$ allow older observations to be downweighted, are chosen independently of the hedging losses, and remain fixed during adversarial training. 
Choosing \(w\) entails a trade-off between recency and statistical efficiency. Sampling uncertainty is governed by the effective sample size $N_{\mathrm{eff}}(w)\coloneqq1/\sum_{n=1}^Nw_n^2$, whereas exposure to cumulative drift is governed by $D_p(w)\coloneqq(\sum_{n=1}^Nw_n(N-n+1)^p)^{1/p}$.
Here $N-n+1$ counts the number of one-step distributional shifts from $P_n$ to $P_{N+1}$, so that $D_p(w)\varrho$ controls the resulting drift bias. Concentrating weight on recent trajectories reduces drift bias, whereas spreading weight across more trajectories increases $N_{\mathrm{eff}}(w)$ and reduces sampling uncertainty. The following theorem links the baseline weights to a calibrated radius and a next-period risk bound.

\begin{restatable}[Drift-aware coverage and next-period risk]{theorem}{DriftAwareNextPeriodRisk}
\label{thm:drift-aware-next-period-risk}
Let $p\in(1,\infty)$ and suppose Assumptions~\ref{ass:path-space} and~\ref{ass:nonstationary-drift} hold. There exist constants $c_1,c_2>0$ and $\nu\in(0,1/2)$, depending only on $d(T+1)$, $p$, $\operatorname{diam}_{\kappa,\lambda}(\Xi)$, such that the following holds for every $w\in\Delta_N$ and $\beta\in(0,1)$. Define
\[
\varepsilon_\varrho(w,\beta)
\coloneqq
D_p(w)\varrho
+
\left(
c_2N_{\mathrm{eff}}(w)^{-\nu}
+
\sqrt{\frac{\log(1/\beta)}{c_1N_{\mathrm{eff}}(w)}}
\right)^{1/p}.
\]
With probability at least $1-\beta$, we have $P_{N+1}\in\mathcal B_{\varepsilon_\varrho(w,\beta)}(\mu_w)$ and $\mathbb E_{X\sim P_{N+1}}[\ell_{\mathrm{DH}}(\vartheta;X)]\le V_{\vartheta}(\varepsilon_\varrho(w,\beta))$ for every $\vartheta\in\Theta\times\mathbb R$, where $V_{\vartheta}$ uses reference law $\mu_w$.
\end{restatable}

The coverage bound underlying Theorem~\ref{thm:drift-aware-next-period-risk} specializes the nonstationary Wasserstein concentration result of \citet{keehan2025don} to the trajectory space \((\Xi,d_{\kappa,\lambda})\). The proofs are given in Appendix~\ref{app:proofs-drift-aware-wasserstein-dro}. For fixed weights, the theorem yields a calibrated radius. Alternatively, for a target radius \(\varepsilon_\varrho\), we select the weights by minimizing the failure-probability bound by solving
\[
w
\in
\argmax_{w\in\Delta_N}
\left\{
N_{\mathrm{eff}}(w)
(\varepsilon_\varrho-D_p(w)\varrho)_+^{2p}
\right\}.
\]
The optimizer therefore shifts mass toward more recent observations as the drift bound $\varrho$ increases, whereas smaller values of $\varrho$ favor a broader allocation across the historical sample. Moreover, \citet[Theorem~2]{keehan2025don} show that this generally nonconvex problem admits an optimal solution of the truncated-polynomial form $w_n\!=\!\bigl(a_1-a_2(N-n+1)^p\bigr)_+$, for some $a_1,a_2\ge0$ such that the $w_n$ sum to one. The resulting weights define the drift-aware empirical distribution $\mu_w$ that serves as the fixed center of the two-budget adversarial problem.

\section{Two-Budget OT--\texorpdfstring{$\phi$}{phi} Adversarial Training}
\label{sec:short-two-budget}

The Wasserstein formulation in Section~\ref{sec:short-setup} perturbs observed trajectories but leaves their relative probabilities unchanged. To also capture uncertainty in scenario probabilities, we adapt the joint transport-and-reweighting construction of \citet[Section~5]{blanchet2023unifying} to weighted market trajectories, with separate budgets for path perturbations and probability reweighting.

\subsection{A Two-Budget Ambiguity Set}
\label{subsec:two-budget-ambiguity-set}

We represent both forms of uncertainty through joint laws of reference and perturbed trajectories. 
For a coupling \(\pi\in\mathcal P(\Xi\times\Xi)\), let \(\pi_1\) and \(\pi_2\) denote its left and right marginals, respectively. Probability reweighting is constrained by \(\phi\)-divergence.

\begin{definition}[Entropy function and $\phi$-divergence]
\label{def:phi-divergence}
A function \(\phi:[0,\infty)\to[0,\infty]\) is an \emph{entropy function} if it
is proper, convex, lower semicontinuous, and satisfies \(\phi(1)=0\). For
\(\pi',\pi\in\mathcal P(\Xi\times\Xi)\), define the $\phi$-divergence of $\pi'$ with respect to~$\pi$ induced by $\phi$ as
\[
D_\phi(\pi'\|\pi)
\coloneqq
\begin{cases}
\displaystyle\int_{\Xi\times\Xi}
\phi\!\left(\frac{d\pi'}{d\pi}\right)\,d\pi, & \pi'\ll\pi,\\
+\infty, & \pi'\not\ll\pi.
\end{cases}
\]
\end{definition}

Standard choices include \(\phi_{\mathrm{KL}}(u)=u\log u-u+1\) for the Kullback--Leibler divergence and \(\phi_{\chi^2}(u)=(u-1)^2\) for the Pearson \(\chi^2\)-divergence. Combining probability reweighting with transportation yields the following ambiguity set.

\begin{definition}[Two-budget OT--$\phi$ ambiguity set]
\label{def:two_budget_phi_ot_clean}
For transport budget $\varepsilon\ge0$ and reweighting budget $\tau\ge0$, define
\begin{equation}
\label{eq:short-two-budget-set}
\mathcal B_{\varepsilon,\tau}(\mu_w)
\coloneqq
\left\{
\eta\in\mathcal P(\Xi):
\begin{array}{l}
\text{there exist }\pi,\pi'\in\mathcal P(\Xi\times\Xi)
\text{ such that}\\[0.4ex]
\pi_1=\mu_w,\quad \pi'_2=\eta,\quad
D_\phi(\pi'\|\pi)\le\tau,\\[0.8ex]
\displaystyle
\int_{\Xi\times\Xi}
d_{\kappa,\lambda}(x,\widehat x)^p\,
d\pi'(x,\widehat x)\le\varepsilon^p
\end{array}
\right\}.
\end{equation}
\end{definition}

The first marginal of \(\pi\) anchors the construction at the reference law \(\mu_w\), while the second marginal of \(\pi'\) defines the adversarial law \(\eta\). The divergence budget \(\tau\) limits probability reweighting relative to the reference coupling, whereas the transport budget \(\varepsilon\) limits displacement in the trajectory space.

\begin{remark}[Coverage inheritance]
\label{rem:two-budget-coverage}
For every \(\tau\ge0\),
\(\mathcal B_\varepsilon(\mu_w)\subseteq\mathcal B_{\varepsilon,\tau}(\mu_w)\),
since any Wasserstein-feasible coupling is feasible in \eqref{eq:short-two-budget-set} by taking \(\pi'\!=\!\pi\). Hence the transport radius calibrated in Section~\ref{subsec:short-coverage} retains its next-period coverage guarantee under the two-budget formulation. This inclusion does not justify reducing the calibrated transport radius in exchange for a reweighting budget.
\end{remark}

The corresponding two-budget robust value $V_\vartheta(\varepsilon,\tau)$ and the associated deep-hedging problem are
\begin{equation}
\label{eq:short-two-budget-dro}
\inf_{\vartheta\in\Theta\times\mathbb R} \bigg\{V_{\vartheta}(\varepsilon,\tau) \coloneqq 
\sup_{\eta\in\mathcal B_{\varepsilon,\tau}(\mu_w)}
\mathbb E_{ X\sim\eta}
\big[\ell_{\mathrm{DH}}(\vartheta; X)\big]\bigg\}.
\end{equation}

Observations assigned zero baseline weight are excluded from the support of \(\mu_w\), so subsequent sums may be restricted to indices with \(w_n>0\). The formulation contains the two corresponding single-budget models as special cases: \(\tau=0\) reduces to Wasserstein DRO, while \(\varepsilon=0\) yields pure \(\phi\)-divergence reweighting on the empirical support.

\subsection{First-Order Approximation}

For small transport and reweighting budgets, we derive a tractable local approximation of the inner problem in \eqref{eq:short-two-budget-dro}. The transport contribution is determined by the dual geometry of the path metric.

\begin{restatable}{proposition}{WeightedPathDualGeometry}
\textnormal{(Dual geometry of the weighted path metric)}
\label{prop:dual-geometry-main}
Let $p\in(1,\infty)$, $q=p/(p-1)$, and $\kappa_t,\lambda_i>0$. For $\Delta X,G\in\mathbb R^{d(T+1)}$, define
\begin{equation}
\label{eq:dual_norm_main}
\|\Delta X\|_{\kappa,\lambda,p}
\coloneqq\Bigl(\sum_i(\lambda_i\max_t\kappa_t|\Delta X_t^i|)^p\Bigr)^{1/p},
\quad
\|G\|_{\kappa,\lambda,p}^{\ast}
\coloneqq\Bigl(\sum_i\Bigl(\sum_t\frac{|G_t^i|}{\lambda_i\kappa_t}\Bigr)^q\Bigr)^{1/q}.
\end{equation}
The first is a norm inducing the metric in \eqref{eq:kappa-lambda-metric} through \(d_{\kappa,\lambda}(X,\widehat X)=\|X-\widehat X\|_{\kappa,\lambda,p}\), and the second is its dual norm under the Euclidean inner product. Moreover, there exists a Borel map \(h:\mathbb R^{d(T+1)}\to\mathbb R^{d(T+1)}\), with \(h(0)=0\), such that \(\|h(G)\|_{\kappa,\lambda,p}=1\) and \(\langle G,h(G)\rangle=\|G\|_{\kappa,\lambda,p}^{\ast}\) whenever \(G\ne0\).
\end{restatable}

% The dual norm is the maximum of \(\langle G,\Delta X\rangle\) over \(\|\Delta X\|_{\kappa,\lambda,p}\le1\), with \(h(G)\) a Borel measurable choice of maximizer.
For \(n=1,\dots,N\), write \(L_n(\vartheta)\!\coloneqq\!\ell_{\mathrm{DH}}(\vartheta;X_n)\) and \(g_n(\vartheta)\!\coloneqq\!\nabla_X\ell_{\mathrm{DH}}(\vartheta;X_n)\). The local approximation of the robust value $V_\vartheta(\varepsilon,\tau)$ in \eqref{eq:short-two-budget-dro} requires the loss to admit the following first-order expansion. 

\begin{restatable}[Loss regularity]{assumption}{LossRegularity}
\label{ass:loss-regularity}
The trajectories $X_n$ lie in the interior of $\Xi$, and for every $\vartheta\in\Theta\times\mathbb R$, the loss $\ell_{\mathrm{DH}}(\vartheta;\cdot)$ is Borel measurable, bounded on $\Xi$ and admits at each $X_n$ the expansion
\[
\ell_{\mathrm{DH}}(\vartheta;X_n+\Delta)
=L_n(\vartheta)+\langle g_n(\vartheta),\Delta\rangle+o(\|\Delta\|_{\kappa,\lambda,p}), \quad \|\Delta\|_{\kappa,\lambda,p}\to0.
\]
\end{restatable}

The local conditions provide the first-order loss expansion and ensure that perturbations remain feasible near the observed trajectories, while boundedness controls losses away from the empirical support. For the reweighting term, we restrict attention to \emph{locally regular} entropy functions:  \(\phi\) is twice continuously differentiable near \(1\), with \(\phi''(1)>0\). This local condition is needed only for the small-budget expansion and is satisfied by the Kullback--Leibler and Pearson \(\chi^2\) divergences above. Let \(\bar L_w(\vartheta)\coloneqq\sum_{n=1}^Nw_nL_n(\vartheta)\) denote the weighted empirical loss, and define the loss dispersion and transport sensitivity by
\[
\sigma_w(\vartheta)\!\coloneqq\!\Bigl(\sum_{n=1}^Nw_n(L_n(\vartheta)-\bar L_w(\vartheta))^2\Bigr)^{1/2},
\qquad
\Upsilon_w(\vartheta)\!\coloneqq\!\Bigl(\sum_{n=1}^Nw_n\|g_n(\vartheta)\|_{\kappa,\lambda,p}^{\ast q}\Bigr)^{1/q}.
\]

\begin{restatable}{theorem}{JointSensitivityAnalysis}
\textnormal{(Joint first-order approximation)}
\label{thm:two_budget_ot_phi_taylor}
Suppose $\phi$ is locally regular,  $p\in(1,\infty)$ and Assumption~\ref{ass:loss-regularity} holds. For every fixed $\vartheta\in\Theta\times\mathbb R$, as $(\varepsilon,\tau)\to(0,0)$ 
\begin{equation}
\label{eq:two_budget_joint_taylor}
V_{\vartheta}(\varepsilon,\tau)
=
\bar L_w(\vartheta)
+
\Upsilon_w(\vartheta)\,\varepsilon
+
\sqrt{\frac{2}{\phi''(1)}}\,\sigma_w(\vartheta)\sqrt{\tau}
+
o\big(\varepsilon+\sqrt{\tau}\big).
\end{equation}
\end{restatable}

The expansion separates the leading-order effects of the two uncertainty mechanisms. The coefficient \(\Upsilon_w(\vartheta)\) measures sensitivity to local deformation of trajectories under \(d_{\kappa,\lambda}\), whereas \(\sqrt{2/\phi''(1)}\,\sigma_w(\vartheta)\) measures sensitivity to reallocating probability mass across scenarios. Thus, transport captures within-scenario sensitivity and reweighting captures cross-scenario loss heterogeneity. Cross-effects enter only at higher order, although the two mechanisms remain coupled in the exact adversarial problem because the transport allocation depends on the reweighted scenario probabilities. Along \(\tau=0\) and \(\varepsilon=0\), respectively,
\eqref{eq:two_budget_joint_taylor} recovers the weighted Wasserstein and weighted \(\phi\)-divergence first-order limits \citep{Bartl2021Sen,duchi2021statistics}.

\subsection{First-Order Adversarial Construction}
To attain the first-order expansion in Theorem~\ref{thm:two_budget_ot_phi_taylor}, we construct an explicit first-order attack. Fix $\vartheta$ and let $Y_n\coloneqq L_n-\bar L_w$ denote the centered loss of the $n$-th trajectory. If $\sigma_w>0$, define
\[
\alpha_\tau
\coloneqq
\sup\Big\{
\alpha\ge0:
1+\alpha Y_n\ge0\ \text{for all }n,\quad
\sum_{n=1}^Nw_n\phi(1+\alpha Y_n)\le\tau
\Big\},
\]
and set $a_n\coloneqq1+\alpha_\tau Y_n$. If $\sigma_w=0$, set $\alpha_\tau=0$ and $a_n=1$. Since $1+\alpha_\tau Y_n \ge 0$ for all $n$ and $\sum_{n=1}^Nw_nY_n=0$, the reweighted probabilities $\widehat w_n \coloneq w_na_n$ remain nonnegative and normalized. The sign of the centered loss \(Y_n=L_n-\bar L_w\) determines whether trajectory \(n\) gains or loses probability mass and \(\alpha_\tau\) sets the largest feasible step. The resulting tilt therefore reallocates mass toward trajectories with above-average loss while preserving normalization and satisfying the budget.

Let $\Upsilon_{\widehat w}\coloneqq(\sum_{n=1}^N\widehat w_n\|g_n\|_{\kappa,\lambda,p}^{\ast q})^{1/q}$ denote the transport sensitivity under the reweighted probabilities. If $\Upsilon_{\widehat w}>0$, define
\begin{equation}
\label{eq:structured-perturbation-short}
\widehat X_n
\coloneqq
X_n+
\varepsilon h(g_n)
\|g_n\|_{\kappa,\lambda,p}^{\ast(q-1)}
\Upsilon_{\widehat w}^{\,1-q};
\end{equation}
otherwise, set $\widehat X_n=X_n$. The direction $h(g_n)$ maximizes the local loss increase per unit transport distance, while $\|g_n\|_{\kappa,\lambda,p}^{\ast(q-1)}$ allocates larger perturbations to trajectories with greater local sensitivity. The normalization by $\Upsilon_{\widehat w}$ scales the perturbations to the transport budget under the reweighted probabilities. Thus, the construction first tilts scenario probabilities according to realized loss and then allocates the transport budget according to local pathwise sensitivity. The probability multipliers and perturbed trajectories define the adversarial law $\widehat\eta_{\varepsilon,\tau}\coloneqq\sum_{n=1}^N \widehat w_n\delta_{\widehat X_n}$.

\begin{restatable}[First-order accuracy of the adversary]{theorem}{FirstOrderStructuredOTPhiAttack}
\label{thm:structured_phi_ot_first_order_attack_main}
Suppose $\phi$ is locally regular, $p\in(1,\infty)$ and Assumption~\ref{ass:loss-regularity} holds. For sufficiently small $(\varepsilon,\tau)$,
$\widehat\eta_{\varepsilon,\tau}\in \mathcal B_{\varepsilon,\tau}(\mu_w)$, and as $(\varepsilon,\tau)\to(0,0)$,
\begin{equation}
\label{eq:short-attack-accuracy}
V_{\vartheta}(\varepsilon,\tau)
-
\mathbb E_{X\sim\widehat\eta_{\varepsilon,\tau}}
\big[\ell_{\mathrm{DH}}(\vartheta;X)\big]
=
o\big(\varepsilon+\sqrt{\tau}\big).
\end{equation}
\end{restatable}

Hence the two-budget robust value is attained to first order by an adversary consisting of one probability multiplier and one perturbed trajectory per observation. Its proof is in Appendix~\ref{appsubsec:two_budget_adversarial}.

\section{\textsc{WRAP}: Nonstationary Two-Budget Adversarial Training}
\label{sec:algorithms}

The first-order analysis in Section~\ref{sec:short-two-budget} yields a tractable approximation to the two-budget robust objective. Training alternates between an approximate inner maximization over scenario weights and trajectory perturbations and an outer minimization over the hedge parameters. The drift-aware weights from Section~\ref{subsec:short-coverage} remain fixed throughout training and define the reference distribution.

\textbf{Inner step: Adversarial construction.}
For fixed \(\vartheta\), Algorithm~\ref{alg:drift-aware-one-step-attack} implements the structured first-order attack characterized in Theorem~\ref{thm:structured_phi_ot_first_order_attack_main}. Reweighting shifts probability mass toward scenarios with above-average losses, while transport perturbs trajectories along locally adverse directions, yielding a two-budget adversary around the drift-aware reference distribution. The one-step Wasserstein Projection Gradient Descent (WPGD) attack of \citet[Section~4]{he2025distributional} is recovered as the special case with zero reweighting budget and uniform reference weights. Appendix~\ref{app:iterative-attack} develops the corresponding multistep two-budget variant, which similarly reduces to their iterative transport attack under the same restriction. Appendix~\ref{app:additional-results} compares the one-step and multistep variants.

\begin{algorithm}[t]
\footnotesize
\caption{\textbf{W}asserstein \textbf{R}eweighting \textbf{A}dversarial \textbf{P}erturbation (\textsc{WRAP})}
\label{alg:drift-aware-one-step-attack}
\footnotesize
\begin{algorithmic}[1]
\Require Weighted sample \(\{(X_n,w_n)\}_{n=1}^N\), batch \(\mathcal I\subseteq\{1,\dots,N\}\) with \(|\mathcal I|=M\), hedge parameter \(\vartheta\), transport budget \(\varepsilon\), reweighting budget \(\tau\)
\Ensure Adversarial weighted batch \(\{(\widehat X_n,\widehat w_n)\}_{n\in\mathcal I}\)

\State Set \(\bar w_n\gets w_n/\sum_{j\in\mathcal I}w_j\) and compute \(L_n\gets\ell_{\mathrm{DH}}(\vartheta;X_n)\), \(g_n\gets\nabla_X\ell_{\mathrm{DH}}(\vartheta;X_n)\) for all \(n\in\mathcal I\).

\Statex \StartBlock{Phi} \textbf{\(\phi\)-divergence reweighting}
\Indent 
        \State \textbf{if} \(\tau=0\) \textbf{then} Set \(a_n\gets1\) for all \(n\in\mathcal I\).
        \State \textbf{else} \ Set \(Y_n\gets L_n-\bar L_{\mathcal I}\) for all \(n\in\mathcal I\), where \(\bar L_{\mathcal I}= \sum_{n\in\mathcal I}\bar w_nL_n\).
        \State \quad \quad Choose the largest \(\alpha_\tau\ge0\) such that \(1+\alpha_\tau Y_n\ge0\) for all \(n\in\mathcal I\) and \(\sum_{n\in\mathcal I}\bar w_n\phi(1+\alpha_\tau Y_n)\le\tau\).
        \State \quad \quad Set \(a_n\gets1+\alpha_\tau Y_n\) for all \(n\in\mathcal I\).
        
        \State Set \(\widehat w_n\gets\bar w_na_n\) for all \(n\in\mathcal I\).
\EndBlock{Phi}{PhiShade}

\Statex \StartBlock{OT} \textbf{Optimal-transport perturbation}
    \State \textbf{if} \(\varepsilon=0\) \textbf{then} Set \(\widehat X_n\gets X_n\) for all \(n \in \mathcal{I}\).
    \State \textbf{else} \ Set \(\Upsilon_{\widehat w}^{\mathcal I}\gets (\sum_{n\in\mathcal I}\widehat w_n \|g_n\|_{\kappa,\lambda,p}^{\ast q})^{1/q}\).
    \State \quad \quad Set \(\widehat X_n\gets X_n+\varepsilon h(g_n)\|g_n\|_{\kappa,\lambda,p}^{\ast(q-1)} (\Upsilon_{\widehat w}^{\mathcal I})^{1-q}\) for all \(n\in\mathcal I\).
\EndBlock{OT}{OTShade}
\EndIndent
\State \Return \(\{(\widehat X_n,\widehat w_n)\}_{n\in\mathcal I}\).
\end{algorithmic}
\end{algorithm}

\textbf{Outer step: Loss minimization.}
Given the adversarial weighted batch \(\mathcal{I}\), the outer step minimizes
\[
\min_{\vartheta\in\Theta\times\mathbb R}
\Big\{
\mathcal J_{\mathrm{adv}}^{\mathcal I}(\vartheta)
\coloneqq
\sum_{n\in\mathcal I}
\widehat w_n
\ell_{\mathrm{DH}}
\bigl(\vartheta;\widehat X_n\bigr)
\Big\}.
\]
The gradient with respect to $\vartheta$ is evaluated while holding the adversarial weights and perturbed trajectories fixed. Under sufficient regularity, a Danskin argument shows that this update agrees, to leading order in the ambiguity budgets, with the gradient of the corresponding robust value function. After each outer update, the inner adversary is recomputed at the updated hedge parameters, so that the scenario reweighting and transport perturbations adapt to the evolving loss landscape.

\section{Numerical Experiments}
\label{sec:numerical-experiments}
We conduct numerical experiments to evaluate the two central components of the proposed framework. First, we assess whether drift-aware reference weights improve out-of-sample hedging performance under temporal nonstationarity. Second, we isolate the effects of probability reweighting and trajectory perturbations and examine how the corresponding ambiguity budgets affect adversarial paths and hedging performance. We compare uniform reference weights (\textsc{-U}) with drift-aware weights (\textsc{-W}). For each reference distribution, clean training (\textsc{ERM}) applies no adversarial perturbation, ``$\phi$-'' applies probability reweighting, ``\textsc{OT-}'' applies trajectory perturbations, and ``\textsc{WRAP}'' combines both mechanisms. Across all experiments, the network architecture and optimization settings are held fixed as reported in Table~\ref{tab:experiment_hyperparameters}. Ambiguity budgets are selected separately for each experimental setting and training specification using the validation sample, while all reported performance measures are evaluated on the held-out test sample. Appendix~\ref{app:experiment-setup} provides additional results, including the stationary Heston benchmark and extensions with proportional transaction costs and a Mean--CVaR (M--CVaR) objective.

\subsection{Nonstationary Heston Experiments}
We evaluate the robustness of a deep-hedging strategy for a European call option trained under the OCE risk measure CVaR in a nonstationary Heston market. Temporal nonstationarity is introduced by varying the stock-price drift and long-run variance linearly across eight consecutive blocks of simulated trajectories. The first five blocks are used for training, the sixth for hyperparameter and model selection, and the final two for evaluating near-future (Near) and far-future (Far) performance under progressively larger distribution shifts. We compare uniform reference weights with the drift-aware weights from Section~\ref{subsec:short-coverage}, which increasingly emphasize recent observations according to the validation-selected ratio \((\varrho/\varepsilon_\varrho)^*\). The reference-weight radius \(\varepsilon_\varrho\) is distinct from the adversarial transport budget \(\varepsilon\) (\textit{cf.}~Appendix~\ref{subsec:heston-results}).

\begin{table*}[!ht]
\centering
\scriptsize
\begin{threeparttable}
\caption{Results for the nonstationary Heston model.}
\label{tab:short-nonstationary}
\setlength{\tabcolsep}{4pt}
\begin{tabular}{@{}lccccccc@{}}
\toprule
Method & $\varepsilon$ & $\tau$ & \shortstack{Val.~CVaR$(\downarrow)$} & \shortstack{Near CVaR$(\downarrow)$} & \shortstack{Far CVaR$(\downarrow)$} & \shortstack{Far Mean PnL$(\uparrow)$} & \shortstack{Far Turnover$(\downarrow)$} \\
\midrule
\textsc{ERM-U} & -- & -- & $6.868 \pm 0.028$ & $7.952 \pm 0.137$ & $8.992 \pm 0.236$ & $\underline{-4.528} \pm 0.055$ & $10.450 \pm 0.955$ \\
$\phi$\textsc{-U} & -- & $1.00$ & $6.324 \pm 0.034$ & $7.239 \pm 0.118$ & $8.128 \pm 0.248$ & $-4.557 \pm 0.038$ & $10.717 \pm 1.774$ \\
\textsc{OT-U} & $0.20$ & -- & $6.245 \pm 0.093$ & $7.050 \pm 0.180$ & $7.811 \pm 0.397$ & $-4.594 \pm 0.044$ & $10.056 \pm 1.572$ \\
\textsc{WRAP-U} & $0.30$ & $0.60$ & $6.005 \pm 0.058$ & $6.641 \pm 0.136$ & $7.268 \pm 0.193$ & $-4.571 \pm 0.029$ & $10.437 \pm 2.500$ \\
\midrule
\textsc{ERM-W} & -- & -- & $6.653 \pm 0.044$ & $7.735 \pm 0.126$ & $8.707 \pm 0.249$ & $-4.532 \pm 0.042$ & $7.995 \pm 0.650$ \\
$\phi$\textsc{-W} & -- & $1.20$ & $6.259 \pm 0.036$ & $7.097 \pm 0.169$ & $7.919 \pm 0.305$ & $-4.553 \pm 0.042$ & $8.033 \pm 0.455$ \\
\textsc{OT-W} & $0.20$ & -- & $6.199 \pm 0.093$ & $6.984 \pm 0.139$ & $7.772 \pm 0.232$ & $-4.590 \pm 0.034$ & $\underline{7.412} \pm 0.672$ \\
\textsc{WRAP-W} & $0.40$ & $1.00$ & $\underline{5.954} \pm 0.029$ & $\underline{6.600} \pm 0.095$ & $\underline{7.221} \pm 0.186$ & $-4.584 \pm 0.034$ & $7.483 \pm 1.125$\\
\bottomrule
\end{tabular}
\begin{tablenotes}[flushleft]
\scriptsize
\item \textit{Note.} Values are mean $\pm$ standard deviation over five seeds.
\end{tablenotes}
\end{threeparttable}
\end{table*}

\textbf{Results.}~As shown in Table~\ref{tab:short-nonstationary}, under uniform weights, reweighting and transport alone reduce far-future CVaR by \(9.6\%\) and \(13.1\%\), respectively, relative to \textsc{ERM-U}, while \textsc{WRAP-U} reduces it by \(19.2\%\). Drift-aware weighting further lowers CVaR for all four methods in every evaluation period. \textsc{WRAP-W}, with \((\varepsilon,\tau)=(0.40,1.00)\), attains the lowest validation, near-future, and far-future CVaR (\textit{cf.}~Figure~\ref{fig:heston-nonstationary-budget-sensitivity} for the validation landscape). The extensions with transaction costs and M--CVaR likewise favor joint training, which attains the lowest risk in both settings (\textit{cf.}~Tables~\ref{tab:heston_ns_u_tc_optimal}--\ref{tab:heston_ns_u_tc_mc_optimal}).

\begin{figure}[H]
\centering
\includegraphics[width=\textwidth]{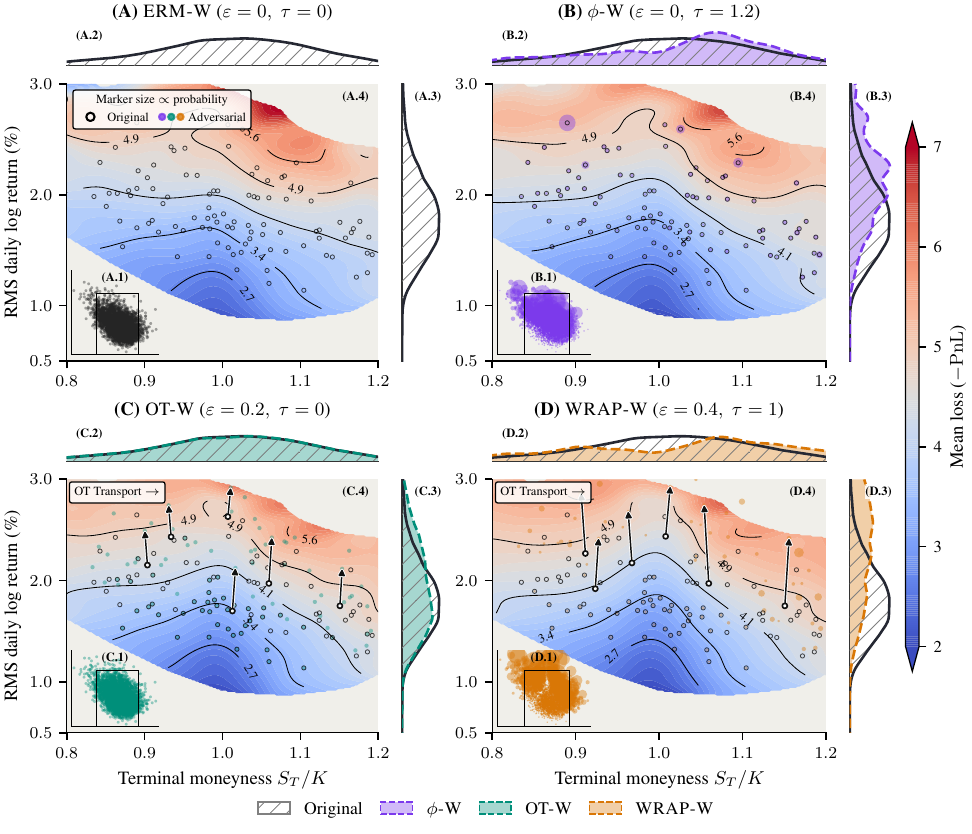}
\caption{Loss landscapes across deep-hedging training schemes from Table~\ref{tab:short-nonstationary}: (A) \textsc{ERM-W}, (B) $\phi$\textsc{-W}, (C) \textsc{OT-W}, and (D) \textsc{WRAP-W}. Each surface depicts the mean hedging loss of a separately trained network on the same reference trajectories $X_n$. The axes show terminal moneyness and the root mean square (RMS) of daily log returns in percent. Markers represent reference paths and adversarial trajectories $\widehat X_n$, with sizes proportional to their respective probabilities $w_n$ and $\widehat w_n$. Arrows indicate selected transport displacements. Within each main panel, subpanels 1--4 report the full feature range, the moneyness marginal, the log-return marginal, and the loss landscape.}
\label{fig:heston-path-distributions}
\end{figure}

Figure~\ref{fig:heston-path-distributions} shows how the two adversarial mechanisms reshape the training distribution and the resulting loss landscape. Panel~(A) shows that under \textsc{ERM-W}, highly volatile paths generate a broad region of high losses. In Panel~(B), $\phi$\textsc{-W} leaves the trajectories unchanged but reallocates probability toward adverse scenarios, particularly in regions of high realized variation (B.1). This concentration is also visible in the moneyness and log-return marginals (B.2--B.3), where probability mass shifts toward selected groups of adverse paths. Panel~(C) isolates the effect of transport. The selected displacements mainly increase realized price variation while inducing comparatively small changes in terminal moneyness (C.1). For a European call, these perturbations stress the sequence of trading gains without requiring a comparable change in the terminal payoff. Accordingly, the marginals remain relatively close to the reference distribution (C.2--C.3).

Panel~(D) combines both mechanisms. At the selected budgets (\textit{cf.}~Table~\ref{tab:short-nonstationary}), \textsc{WRAP-W} reallocates probability toward adverse scenarios while simultaneously transporting trajectories toward locally unfavorable regions (D.1--D.3). The corresponding loss landscape (D.4), together with that of \textsc{OT-W}, exhibits a smaller high-loss region around the strike, although highly volatile paths remain difficult to hedge. Since all loss surfaces are evaluated on the same reference trajectories, these differences reflect the learned hedging policies. The figure thus provides a qualitative counterpart to Theorem~\ref{thm:two_budget_ot_phi_taylor}: reweighting acts on loss dispersion across scenarios, whereas transport acts on local sensitivity to perturbations within trajectories.

Figure~\ref{fig:attack-anatomy-main} contrasts reweighting and transport at the level of individual trajectories: the highest-weight path differs from the path most strongly displaced by optimal transport. Reweighting emphasizes high-loss scenarios, whereas transport primarily increases realized variation. Figures~\ref{fig:heston-nonstationary-uniform-policy}--\ref{fig:heston-nonstationary-policy} show the corresponding hedging policies, and Appendix~\ref{subsec:attack-anatomy} provides further details.

\begin{figure}[H]
\centering
\includegraphics[width=\textwidth]{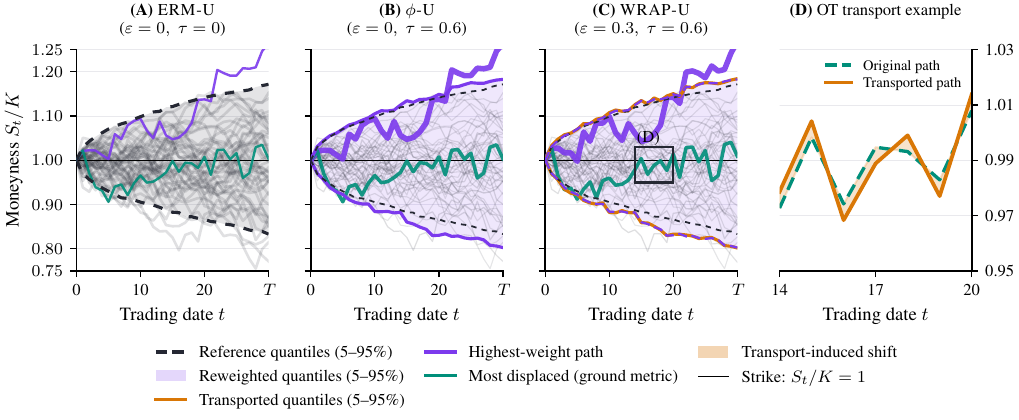}
\caption{Pathwise effects in nonstationary Heston:
(A) \textsc{ERM-U}, (B) $\phi$\textsc{-U}, (C) \textsc{WRAP-U}, and (D) magnified view of the original and transported paths highlighted in (C). Panels~(A)--(C) show sampled trajectories and highlight the highest-weight and most-displaced paths.}
\label{fig:attack-anatomy-main}
\end{figure}

\subsection{Market-Data Experiments}
\label{sec:gad_market_descriptive_stats}
We test whether adversarial training improves robustness to changes in market conditions estimated from real equity prices. Following the GAD benchmark of \citet{he2025distributional} and the parameter-uncertainty framework of \citet{lutkebohmert2022robust}, we hedge a 30-step arithmetic Asian call on AAPL, AMZN, BRK-B, GOOGL, and MSFT under an entropic risk objective. Training, validation, and test trajectories are simulated from GAD models fitted separately to prices from 2008--2019, 2020--2023, and 2024--2025. We consider two training specifications: \textsc{Fix} uses the final parameter estimate from the training period, whereas \textsc{Robust} resamples complete rolling estimates at each trading step. Unlike stress-test designs centered on exogenously imposed crash scenarios, we ask whether adversarial training remains beneficial under empirically estimated regime shifts or becomes overly conservative relative to the risk objective. The historical training and validation windows exhibit more severe drawdowns than the test window (\textit{cf.}~Appendix~\ref{subsec:market-results} and Table~\ref{tab:gad_market_descriptive_stats}).

\textbf{Results.}~\textsc{WRAP-U} attains lower validation and test entropic risk than either single-budget method across all ten stock-training-scheme combinations. Relative to \textsc{ERM-U}, \textsc{WRAP-U} lowers test entropic risk in six of ten comparisons, while in three additional cases coinciding with \textsc{ERM-U}. Table~\ref{tab:gad_synthetic_validation_optimal} reports the complete results. This indicates that joint adversarial training can add robustness beyond parameter uncertainty in the training generator, although the gains remain stock dependent.
\vspace{0mm}
\section{Conclusion}
\label{sec:conclusion}
We develop a drift-aware distributionally robust framework for deep hedging under potential shifts in market regimes. The framework separates the construction of the reference distribution from the subsequent robustification step. Non-uniform baseline weights balance sampling precision against temporal misspecification from older observations, while separate \(\phi\)-divergence and optimal-transport budgets capture uncertainty in scenario probabilities and market trajectories, respectively.

We establish finite-sample coverage guarantees for the weighted empirical law and derive a joint small-budget expansion of the robust hedging objective. The first-order correction separates scenario-reweighting effects, governed by loss dispersion across trajectories, from transport effects, governed by sensitivity to path perturbations.

\noindent\textbf{Limitations and Future Research.}
The baseline weights considered in this paper depend only on the temporal ordering of the observations. A richer formulation could condition them on observable market states or latent regimes, for example through context-dependent optimization or a hidden Markov model, so that historical trajectories are weighted by economic similarity as well as temporal proximity. An adaptive extension could jointly infer the prevailing regime, update the baseline weights, and estimate the drift bound \(\varrho\) as new observations arrive.

The transport and reweighting budgets are currently selected through validation. A more principled calibration could determine $\varepsilon$ from the drift-aware coverage guarantee and relate $\tau$ to finite-sample uncertainty in the weighted scenario probabilities. Finally, a sufficiently large panel of realized market trajectories would permit direct training and evaluation without first fitting a parametric data-generating model.

\newpage

\subsection*{AI use statement}
In this work, we used generative AI tools to aid and polish the writing. In particular, we used these tools to cross-validate the consistency of mathematical notation throughout the paper and to verify that the paper's descriptions match the code implementation as closely as possible, ensuring readability. Additionally, upon finalization of the paper, AI was utilized to discover additional related works that were unknown to the authors during development, which were manually verified before inclusion in the related works section. We have reviewed all AI-assisted work and take responsibility for the final content of this work, including text, claims, or artifacts produced with the aid of generative AI.

\subsection*{Acknowledgments}

This work was supported as part of NCCR Automation, a National Centre
of Competence in Research, funded by the Swiss National Science Foundation (grant number 51NF40\_225155).

\bibliography{iclr2027_conference}
\bibliographystyle{iclr2027_conference}

\appendix

\newpage

\addtocontents{toc}{\protect\setcounter{tocdepth}{3}}
\renewcommand{\contentsname}{Appendix}
\tableofcontents
\newpage

\section{Extended Related Work}
\label{app:extended-related-work}

\textit{Deep hedging and model uncertainty.}
Deep hedging formulates derivative risk management as a model-free, data-driven sequential decision problem \citep{buehler2019deep}; related reviews include \citet{ruf2020neural} on neural networks for pricing and \citet{hambly2023recent} on reinforcement learning in finance. Classical robust hedging optimizes over families of market dynamics, including uncertain-volatility models \citep{avellaneda1995,lyons1995uncertain}. Recent work studies robustness in deep hedging under parameter uncertainty, alternative data-generating processes, generative adversaries, and distributional perturbations \citep{lutkebohmert2022robust,wu2023robust,limmer2024robust,he2025distributional,wiese2020quant,cont2026tailgan}. We extend this literature by adapting the reference law to temporal drift and distinguishing uncertainty about scenario probabilities from uncertainty about market trajectories.

\textit{Distributionally robust optimization and adversarial training.}
Distributionally robust optimization protects decisions against perturbations of a reference law. Wasserstein ambiguity sets transport probability mass according to a ground metric and may therefore generate adverse distributions whose support differs from that of the empirical law \citep{mohajerin2018data,gao2023distributionally,kuhn2025distributionally}. This construction is particularly well suited to trajectory-valued data because the transport cost can encode temporal and component structure \citep{blanchet2019quantifying}. It also provides a distributional interpretation of adversarial training \citep{sinha2018certifying}. First-order sensitivity analysis yields tractable approximations of worst-case objectives and identifies locally adverse perturbation directions \citep{Bartl2021Sen, gao2022Wasserstein}; related methods have been developed for neural networks and deep hedging \citep{bai2023wasserstein,he2025distributional}.

By contrast, \(\phi\)-divergence ambiguity sets preserve the support of the reference law and alter only the relative probability masses assigned to its outcomes \citep{ben2013robust,bayraksan2015data}. They are connected to statistical calibration, variance regularization, and robustness to shifts in the prevalence of high-loss observations or subpopulations \citep{lam2019recovering,duchi2021statistics,van2021data,namkoong2016stochastic, lam2016robust,duchi2023distributionally}. Wasserstein and \(\phi\)-divergence ambiguity sets therefore address distinct forms of misspecification: the former changes the outcomes themselves according to a ground geometry, whereas the latter changes their relative likelihoods.

\noindent Several approaches combine likelihood and optimal-transport perturbations. Holistic robustness protects against statistical error, diffuse data perturbations, and contamination, with a corresponding specialization to neural networks \citep{bennouna2022holistic,bennouna2023certified}. Joint OT--divergence constructions include \citet{blanchet2023unifying} and the smoothed \(\phi\)-divergences of \citet{liu2023smoothed}; our ambiguity set can be viewed as a drift-aware, weighted path-space analogue of the latter, with separate transport and divergence budgets. \textsc{Armor} instead combines the two costs in a single infimal convolution of OT and an information divergence and develops a global adversarial-training method \citep{birrell2025optimal}, whereas Sinkhorn DRO entropy-regularizes the optimal-transport plan \citep{wang2025sinkhorn}. In contrast to these works, we derive a joint small-budget expansion for a sequential, path-dependent objective and construct an explicit transport-and-reweighting attack that is first-order accurate for the full DRO problem.

\textit{Nonstationarity and weighted empirical laws.}
Decision-making under nonstationarity must balance the statistical benefit of additional observations against the misspecification introduced by data generated under earlier distributions \citep{besbes2015non}. \citet{keehan2025don} provide the closest statistical precedent by centering Wasserstein ambiguity sets at weighted empirical laws and separating effective sample size from temporal drift. We integrate this principle with deep hedging and a structured path geometry, and extend it with adversarial probability reweighting under a two-budget training framework.

\section{Technical Assumptions and Definitions}
\label{app:technical-foundations}

This appendix records the assumptions and definitions used in Sections~\ref{sec:short-setup} and~\ref{sec:short-two-budget}. The concentration and sensitivity results built from these conditions are stated and proved in Appendix~\ref{app:proofs}.

\subsection{Market Measurability and Payoff Regularity}
\label{app:additional-market-regularity-details}

We use the market setup and hedging strategy from Section~\ref{sec:short-setup}. All stochastic processes are defined on a filtered probability space $(\Omega,\mathcal F,(\mathcal F_t)_{t=0}^T,\mathbb P)$. The observed process $X$ is adapted to this filtration, and hence $\sigma(X_0,\dots,X_t)\subseteq\mathcal F_t$ for $t=0,\dots,T$. Since each $\delta_t^\theta$ is measurable, the holdings $\delta_t^\theta(X_{0:t})$ are $\mathcal F_t$-measurable, and the terminal PnL in \eqref{eq:pnl} is $\mathcal F_T$-measurable. The negative payoff term $\Phi(X)$ corresponds to a short position in the claim; examples of the payoffs used in the experiments are given in Appendix~\ref{app:claims}.

\begin{remark}[Payoff and loss regularity]
\label{rem:payoff-regularity}
Many vanilla payoffs, such as calls or puts, are globally Lipschitz in $S_T$ and hence in $X$. If differentiability everywhere is desired, kinked terms such as $\max(\cdot,0)$ or $|\cdot|$ can be replaced by smooth approximations (\textit{e.g.}~softplus for the positive part). The same approach applies to nonsmooth OCE losses, such as the CVaR loss, while preserving convexity, monotonicity, and the OCE normalization. These approximations preserve the structure of the DRO formulation.
\end{remark}

\subsection{Drift-Aware Coverage Assumptions}
\label{subsec:drift-aware-dro}
We restate the assumptions from Section~\ref{sec:short-setup} and record their interpretation and the coupling form used in the proofs.

\CoveragePathSpace*

\begin{remark}
Assumption~\ref{ass:path-space} is satisfied when all $\lambda_i$ and $\kappa_t$ are positive and \(\Xi\) is compact. In implementations, compactness corresponds to working with normalized and bounded features or clipping extreme inputs. The Polish-space condition ensures that the Wasserstein ambiguity sets and couplings are well defined, while the finite-diameter condition is used in the concentration bounds and sensitivity analysis.
\end{remark}

% ------------------------------------------------------------
\NonstationarySamplingDrift*

\begin{remark}
Following the nonstationary Wasserstein DRO model of \citet{keehan2025don}, the drift condition in Assumption~\ref{ass:nonstationary-drift} can equivalently be written in coupling form: for each \(n=1,\dots,N\), there exists \(\pi_n\in\Pi(P_n,P_{n+1})\) such that \(d_{\kappa,\lambda}(X,\widehat X)\le\varrho\) \(\pi_n\)-almost surely. Here \(\Pi(P,Q)\) denotes the set of probability measures on \(\Xi\times\Xi\) with marginals \(P\) and \(Q\); see Appendix~\ref{app:wasserstein_distances}. The proof of the drift-aware concentration result uses this coupling form directly.
\end{remark}

% ------------------------------------------------------------
\subsection{Two-Budget Loss Regularity}
\label{sec:two-budget-adversarial}
\label{subsec:transport-geometry-regularity}

This section records the loss-regularity conditions underlying the two-budget formulation and first-order analysis in Section~\ref{sec:short-two-budget}.

Let $\|\cdot\|_{\kappa,\lambda,p}^\ast$ and the maximizing map $h$ be as defined in Proposition~\ref{prop:dual-geometry-main}. Recall $L_n(\vartheta)=\ell_{\mathrm{DH}}(\vartheta;X_n)$ and $g_n(\vartheta)=\nabla_X\ell_{\mathrm{DH}}(\vartheta;X_n)$. The following assumption specifies the regularity and local feasibility conditions required for the first-order analysis.

\LossRegularity*

\begin{remark}
Assumption~\ref{ass:loss-regularity} is local at the observed trajectories: the differentiability and feasibility requirements are imposed only at the empirical support. The boundedness condition is global and is used to control transport mass that may move away from that support. The feasibility condition ensures that the explicit perturbations in \eqref{eq:structured-perturbation-short} remain in \(\Xi\) for sufficiently small transport radii.
\end{remark}

\section{Proofs}
\label{app:proofs}

This appendix contains the proofs of the theoretical results stated in the paper. The arguments follow the logical development of the paper: metric and dual-norm properties of the weighted trajectory distance, drift-aware coverage, and first-order sensitivity expansions for the Wasserstein and two-budget ambiguity sets. For ease of reference, each claim is restated before its proof.

\subsection{Weighted Trajectory Metric and Duality}
\label{app:wasserstein}

This section develops the norm-duality structure underlying the Wasserstein perturbation analysis. After recalling Wasserstein distances on the trajectory space, we specialize the ground metric to the weighted trajectory distance used in the paper. We then identify its induced norm, dual norm, and associated maximizer for the linearized perturbation problem that determines the first-order adversarial direction.
% ------------------------------------------------------------
\subsubsection{Wasserstein Distances on Path Space}
\label{app:wasserstein_distances}
We recall the Wasserstein distance for a generic ground metric before specializing to the weighted trajectory distance. Let \((\Xi,d)\) be a Polish metric space and let \(\mathcal{P}(\Xi)\) denote the set of Borel probability measures on \(\Xi\). For \(p\in[1,\infty)\), the \(p\)-Wasserstein distance between \(\mu,\nu\in\mathcal{P}(\Xi)\) is
\begin{equation}
\label{eq:Wp_def_app}
W_p(\mu,\nu)
\;\coloneqq\;
\left(
\inf_{\pi\in\Pi(\mu,\nu)}
\int_{\Xi\times\Xi} d(x,\widehat x)^p\, \pi(d x, d\widehat x)
\right)^{1/p},
\end{equation}
where \(\Pi(\mu,\nu)\) is the set of couplings of \(\mu\) and \(\nu\).
For \(p=\infty\), the \(\infty\)-Wasserstein distance is
\begin{equation}
\label{eq:Winfty_def_app}
W_\infty(\mu,\nu)
\;\coloneqq\;
\inf_{\pi\in\Pi(\mu,\nu)}
\operatorname*{ess\,sup}_{(x,\widehat x)\sim\pi} d(x,\widehat x).
\end{equation}

\paragraph{Standing convention.}
Unless stated otherwise, all Wasserstein distances in this appendix are computed with the weighted trajectory ground metric \(d_{\kappa,\lambda}\) defined in \eqref{eq:kappa-lambda-metric}.
% ============================================================
\subsubsection{Norm Representation, Dual Norm, and Maximizer Map}
\label{app:dual_norm_h}

We restate and prove here the results about the dual geometry.

\WeightedPathDualGeometry*
\begin{proof}[Proof of Proposition~\ref{prop:dual-geometry-main}]
\textit{Primal norm.}\ 

Homogeneity and definiteness follow directly from the positivity of the weights
\(\kappa_t\) and \(\lambda_i\). For the triangle inequality, define, for each component \(i\),
\[
\|\Delta X^i\|_{\infty,\kappa}
\coloneqq
\max_{t=0,\dots,T}\kappa_t|\Delta X_t^i|,
\qquad
\Delta X^i=(\Delta X_0^i,\dots,\Delta X_T^i).
\]
The weighted sup-norm satisfies
\[
\|\Delta X^i+\Delta Y^i\|_{\infty,\kappa}
\le \|\Delta X^i\|_{\infty,\kappa}
+ \|\Delta Y^i\|_{\infty,\kappa}.
\]
Therefore,
\[
\begin{aligned}
\|\Delta X+\Delta Y\|_{\kappa,\lambda,p}
&=
\left(
\sum_{i=1}^d
\left(
\lambda_i
\|\Delta X^i+\Delta Y^i\|_{\infty,\kappa}
\right)^p
\right)^{1/p} \\
&\le
\left(
\sum_{i=1}^d
\left(
\lambda_i\|\Delta X^i\|_{\infty,\kappa}
+
\lambda_i\|\Delta Y^i\|_{\infty,\kappa}
\right)^p
\right)^{1/p}.
\end{aligned}
\]
Applying Minkowski's inequality to the outer \(\ell_p\) aggregation, with
\(a_i=\lambda_i\|\Delta X^i\|_{\infty,\kappa}\) and
\(b_i=\lambda_i\|\Delta Y^i\|_{\infty,\kappa}\), gives
\[
\begin{aligned}
\|\Delta X+\Delta Y\|_{\kappa,\lambda,p}
&\le
\left(
\sum_{i=1}^d
\left(
\lambda_i\|\Delta X^i\|_{\infty,\kappa}
\right)^p
\right)^{1/p}
+
\left(
\sum_{i=1}^d
\left(
\lambda_i\|\Delta Y^i\|_{\infty,\kappa}
\right)^p
\right)^{1/p} \\
&=
\|\Delta X\|_{\kappa,\lambda,p}
+
\|\Delta Y\|_{\kappa,\lambda,p}.
\end{aligned}
\]
Thus the triangle inequality holds, and \(\|\cdot\|_{\kappa,\lambda,p}\) is a norm.

\smallskip\noindent\textit{Dual norm.} For \(G,\Delta X\in\mathbb{R}^{d(T+1)}\), write
\[
\langle G,\Delta X\rangle
\coloneqq
\sum_{i=1}^d\sum_{t=0}^T G_t^i\Delta X_t^i.
\]
Let \(q\) be the H\"older conjugate of \(p\in(1,\infty)\), so that \(1/p+1/q=1\).

For each component \(i\), define
\(b_i\coloneqq \sum_{t=0}^T |G_t^i|/\kappa_t\). For any
\(\Delta X\in\mathbb{R}^{d(T+1)}\), the componentwise pairing satisfies
\[
\sum_{t=0}^T G_t^i\Delta X_t^i
\le
\sum_{t=0}^T |G_t^i|\,|\Delta X_t^i|
\le
\|\Delta X^i\|_{\infty,\kappa}\, b_i,
\]
where
\(\|\Delta X^i\|_{\infty,\kappa}\coloneqq\max_{t=0,\dots,T}\kappa_t|\Delta X_t^i|\).
Summing over \(i\), we obtain
\[
\langle G,\Delta X\rangle
\le
\sum_{i=1}^d
\left(\lambda_i\|\Delta X^i\|_{\infty,\kappa}\right)
\frac{b_i}{\lambda_i}.
\]
H\"older's inequality in \(\mathbb{R}^d\) gives
\[
\langle G,\Delta X\rangle
\le
\|\Delta X\|_{\kappa,\lambda,p}
\left(
\sum_{i=1}^d
\left(\frac{b_i}{\lambda_i}\right)^q
\right)^{1/q}.
\]
Taking the supremum over all \(\Delta X\) with
\(\|\Delta X\|_{\kappa,\lambda,p}\le 1\) yields
\begin{equation}
\label{eq:dual_norm_upper_bound}
\|G\|_{\kappa,\lambda,p}^\ast
\le
\left(
\sum_{i=1}^d
\left(
\sum_{t=0}^T
\frac{|G_t^i|}{\lambda_i\kappa_t}
\right)^q
\right)^{1/q}.
\end{equation}

It remains to show that the upper bound is tight. Since the dual norm is the
supremum over all feasible perturbations, it is at least the supremum over any
feasible subclass. Choose
\(\Delta X_t^i=\operatorname{sgn}(G_t^i)\alpha_i/\kappa_t\), with
\(\alpha_i\ge 0\). Then
$\|\Delta X^i\|_{\infty,\kappa}\le\alpha_i$,
$\|\Delta X\|_{\kappa,\lambda,p}
\le(\sum_{i=1}^d(\lambda_i\alpha_i)^p)^{1/p}$, and
\(\langle G,\Delta X\rangle=\sum_{i=1}^d\alpha_i b_i\). Hence
\[
\|G\|_{\kappa,\lambda,p}^\ast
\ge
\sup_{\left(\sum_{i=1}^d(\lambda_i\alpha_i)^p\right)^{1/p}\le 1}
\sum_{i=1}^d \alpha_i b_i.
\]
Setting \(c_i=\lambda_i\alpha_i\), the right-hand side becomes
\[
\sup_{\|c\|_p\le 1}
\sum_{i=1}^d c_i\frac{b_i}{\lambda_i}
=
\left(
\sum_{i=1}^d
\left(\frac{b_i}{\lambda_i}\right)^q
\right)^{1/q},
\]
by the standard duality between \(\ell_p\) and \(\ell_q\). Substituting
\(b_i=\sum_{t=0}^T |G_t^i|/\kappa_t\) gives
\begin{equation}
\label{eq:dual_norm_lower_bound}
\|G\|_{\kappa,\lambda,p}^\ast
\ge
\left(
\sum_{i=1}^d
\left(
\sum_{t=0}^T
\frac{|G_t^i|}{\lambda_i\kappa_t}
\right)^q
\right)^{1/q}.
\end{equation}
Combining (\ref{eq:dual_norm_upper_bound}) and
(\ref{eq:dual_norm_lower_bound}) proves
(\ref{eq:dual_norm_main}).

\smallskip\noindent\textit{Maximizer map.} The maximizer map \(h\) is applied to dual variables, such as pathwise gradients
\(G\in\mathbb{R}^{d(T+1)}\), rather than to trajectories. For
\(G\in\mathbb{R}^{d(T+1)}\), define
\(a_i(G)\coloneqq \sum_{t=0}^T |G_t^i|/(\lambda_i\kappa_t)\) and
\(A(G)\coloneqq (\sum_{j=1}^d a_j(G)^q)^{1/q}\). By
(\ref{eq:dual_norm_main}), \(A(G)=\|G\|_{\kappa,\lambda,p}^\ast\).

For
\(G\neq 0\), define
\(r_i(G)\coloneqq a_i(G)^{q-1}/A(G)^{q-1}\) and set
\begin{equation}
\label{eq:h_explicit_app}
h(G)_t^i
\coloneqq
\frac{r_i(G)}{\lambda_i\kappa_t}\operatorname{sign}(G_t^i),
\qquad
i=1,\dots,d,\quad t=0,\dots,T.
\end{equation}
For \(G=0\), we set
\(h(0)=0\).

By construction, $\kappa_t|h(G)_t^i|\le r_i(G)/\lambda_i$. If $r_i(G)>0$, some $G_t^i$ is nonzero, so equality is attained at that date; if $r_i(G)=0$, the entire component vanishes. Thus $\lambda_i\max_t\kappa_t|h(G)_t^i|=r_i(G)$. Hence
\[
\|h(G)\|_{\kappa,\lambda,p}
=
\left(
\sum_{i=1}^d r_i(G)^p
\right)^{1/p}.
\]
From \(r_i(G)=a_i(G)^{q-1}/A(G)^{q-1}\) and the identity
\((q-1)p=q\), we obtain
\[
\sum_{i=1}^d r_i(G)^p
=
\frac{\sum_{i=1}^d a_i(G)^q}{A(G)^q}
=
1,
\]
so \(\|h(G)\|_{\kappa,\lambda,p}=1\).

The same construction attains the dual norm. Substituting \eqref{eq:h_explicit_app} into the pairing gives
\[
\langle G,h(G)\rangle
=
\sum_{i=1}^d r_i(G)
\sum_{t=0}^T
\frac{|G_t^i|}{\lambda_i\kappa_t}
=
\sum_{i=1}^d r_i(G)a_i(G).
\]
Substituting the definition of \(r_i(G)\), this becomes
\[
\sum_{i=1}^d r_i(G)a_i(G)
=
\frac{\sum_{i=1}^d a_i(G)^q}{A(G)^{q-1}}
=
A(G).
\]
Since \(A(G)=\|G\|_{\kappa,\lambda,p}^\ast\) by
(\ref{eq:dual_norm_main}), we conclude that
\(\langle G,h(G)\rangle=\|G\|_{\kappa,\lambda,p}^\ast\).
\end{proof}
% ============================================================
\subsection{Drift-Aware Coverage and Calibration}
\label{app:proofs-drift-aware-wasserstein-dro}

This section establishes the drift-aware concentration result underlying Section~\ref{subsec:short-coverage} and its calibration consequences. The analysis separates statistical sampling error, controlled by the effective sample size \(N_{\mathrm{eff}}(w)\), from temporal drift, controlled by the cumulative drift term \(D_p(w)\varrho\). We first derive the stationary weighted concentration bound, then transfer it to the nonstationary setting through a coupling argument, and finally obtain the calibrated ambiguity radius and the corresponding next-period risk guarantee.

We model nonstationarity as in Assumption~\ref{ass:nonstationary-drift}: the observed trajectories are independent, \(X_n\sim P_n\), and consecutive laws admit couplings with \(d_{\kappa,\lambda}\)-distance at most \(\varrho\) almost surely. In particular, Assumption~\ref{ass:nonstationary-drift} implies \(W_\infty(P_n,P_{n+1})\le\varrho\) for \(n=1,\dots,N\).

Recall from Section~\ref{subsec:short-coverage} the effective sample size and weighted cumulative drift index,
\begin{equation}
\label{eq:Neff_app}
N_{\mathrm{eff}}(w)\coloneqq\frac{1}{\sum_{n=1}^N w_n^2},
\qquad
D_p(w)\coloneqq\left(\sum_{n=1}^N w_n (N-n+1)^p\right)^{1/p}.
\end{equation}
Here \(N_{\mathrm{eff}}(w)\) captures the statistical information retained by the weighted sample, while \(D_p(w)\) captures the exposure of the weighting scheme to temporal drift toward the next-period law.

\subsubsection{Weighted Empirical Concentration}
We first isolate the sampling error from the drift error by considering the stationary benchmark in which \(X_1,\ldots,X_N\) are i.~i.~d.~with law $P$. The following two results show that the weighted empirical measure concentrates around \(P\) at a rate governed by \(N_{\mathrm{eff}}(w)\), rather than by the nominal sample size \(N\). Recall the definition of the simplex  $\Delta_N\!\coloneqq\!\{w\!\in\!\mathbb R_+^N:\sum_{n=1}^Nw_n\!=\!1\}.$

\begin{lemma}[Expectation bound with effective sample size]
\label{lem:expectation-bound}
Suppose Assumption~\ref{ass:path-space} holds. For
\(p\in(1,\infty)\) and \(P\in\mathcal P(\Xi)\), let \(X_1,\dots,X_N\sim P\) be i.i.d., and define
\(\mu_w\coloneqq\sum_{n=1}^N w_n\delta_{X_n}\) for \(w\in\Delta_N\). Then
there exist constants \(c_0>0\) and \(\nu\in(0,1/2)\), depending only on the
ambient dimension \(d(T+1)\), \(p\), and
\(\operatorname{diam}_{\kappa,\lambda}(\Xi)\), such that
\[
\mathbb E\!\left[
W_p(\mu_w,P)^p
\right]
\le
c_0 N_{\mathrm{eff}}(w)^{-\nu}.
\]
\end{lemma}

\begin{proof}[Proof of Lemma~\ref{lem:expectation-bound}]
The proof reduces the weighted empirical measure to a convex combination of ordinary empirical measures and then uses the standard Wasserstein convergence rate for i.i.d.~samples.

Since \(X_1,\dots,X_N\) are i.i.d., the law of
\(\sum_{n=1}^N w_n\delta_{X_n}\) is invariant under permutations of the weights. Since \(N_{\mathrm{eff}}(w)\) is also permutation-invariant, it is enough to prove the claim for weights sorted in nonincreasing order. We therefore write \(w_1\ge\cdots\ge w_N\) and set \(w_{N+1}=0\).
For \(k=1,\dots,N\), define 
\[
\bar P_k \coloneqq \frac1k\sum_{i=1}^k \delta_{X_i},
\qquad
\alpha_k\coloneqq k(w_k-w_{k+1}).
\]
Then \(\mu_w\) admits the decomposition $\mu_w =\sum_{k=1}^N \alpha_k \bar P_k$. 
Indeed, \(\alpha_k\ge0\) and
\[
\sum_{k=1}^N\alpha_k
=\sum_{k=1}^N k(w_k-w_{k+1})
=\sum_{k=1}^N w_k
=1,
\]
so this is a convex combination. Moreover, the coefficient of \(\delta_{X_i}\)
in \(\sum_{k=1}^N\alpha_k\bar P_k\) is
\[
\sum_{k=i}^N \frac{\alpha_k}{k}
=\sum_{k=i}^N (w_k-w_{k+1})
=w_i.
\]
Hence the decomposition is exact.

The map \(Q\mapsto W_p(Q,P)^p\) is convex in its first argument. Therefore,
\[
W_p(\mu_w,P)^p
\le
\sum_{k=1}^N \alpha_k W_p(\bar P_k,P)^p.
\]
Taking expectations gives
\[
\mathbb E\!\left[W_p(\mu_w,P)^p\right]
\le
\sum_{k=1}^N \alpha_k
\mathbb E\!\left[W_p(\bar P_k,P)^p\right].
\]
Since \(\Xi\) has finite diameter under \(d_{\kappa,\lambda}\), standard empirical Wasserstein convergence bounds on compact subsets of finite-dimensional normed spaces, such as \citet{Fournier2015on}, imply that there exist constants
\(C>0\) and \(\nu\in(0,1/2)\) such that
\[
\mathbb E\!\left[W_p(\bar P_k,P)^p\right]
\le
Ck^{-\nu},
\qquad k=1,\dots,N.
\]
The ground metric \(d_{\kappa,\lambda}\) is induced by a weighted norm on a finite-dimensional space, so the usual Euclidean-rate statement applies up to a change of constants. Thus,
\[
\mathbb E\!\left[W_p(\mu_w,P)^p\right]
\le
C\sum_{k=1}^N k(w_k-w_{k+1})k^{-\nu}
=
C\sum_{k=1}^N k^{1-\nu}(w_k-w_{k+1}).
\]
By summation by parts,
\[
\sum_{k=1}^N k^{1-\nu}(w_k-w_{k+1})
=
\sum_{k=1}^N
\big(k^{1-\nu}-(k-1)^{1-\nu}\big)w_k,
\]
where \(0^{1-\nu}=0\). In the remainder of the proof, \(C_\nu<\infty\) denotes a constant depending only on \(\nu\), whose value may change from line to line. Since \(x\mapsto x^{1-\nu}\) is concave and increasing,
\[
k^{1-\nu}-(k-1)^{1-\nu}
\le
C_\nu k^{-\nu}.
\]
Hence
\[
\mathbb E\!\left[W_p(\mu_w,P)^p\right]
\le
C C_\nu \sum_{k=1}^N k^{-\nu}w_k.
\]

It remains to control the last sum by the effective sample size. Let
\(m\coloneqq \lfloor N_{\mathrm{eff}}(w)\rfloor\). Since
\(N_{\mathrm{eff}}(w)\ge1\), we have \(m\ge1\). Split
\[
\sum_{k=1}^N k^{-\nu}w_k
=
\sum_{k=1}^m k^{-\nu}w_k
+
\sum_{k=m+1}^N k^{-\nu}w_k
\coloneqq H+R.
\]
For the head term, Cauchy--Schwarz and the definition of
\(N_{\mathrm{eff}}(w)\) give
\[
H
\le
\left(\sum_{k=1}^m k^{-2\nu}\right)^{1/2}
\left(\sum_{k=1}^m w_k^2\right)^{1/2}
\le
N_{\mathrm{eff}}(w)^{-1/2}
\left(\sum_{k=1}^m k^{-2\nu}\right)^{1/2}.
\]
Because \(\nu<1/2\),
\[
\sum_{k=1}^m k^{-2\nu}
\le
C_\nu m^{1-2\nu}.
\]
Therefore,
\[
H
\le
C_\nu
N_{\mathrm{eff}}(w)^{-1/2}m^{1/2-\nu}
\le
C_\nu N_{\mathrm{eff}}(w)^{-\nu},
\]
where the last inequality uses \(m\le N_{\mathrm{eff}}(w)\).

For the remainder term, since \(k^{-\nu}\) is decreasing and the weights sum to one,
\[
R
\le
(m+1)^{-\nu}\sum_{k=m+1}^N w_k
\le
(m+1)^{-\nu}
\le
N_{\mathrm{eff}}(w)^{-\nu}.
\]
Combining the bounds on \(H\) and \(R\), we obtain
\[
\sum_{k=1}^N k^{-\nu}w_k
\le
C_\nu N_{\mathrm{eff}}(w)^{-\nu}.
\]
Absorbing constants into \(c_0\) gives
\[
\mathbb E\!\left[
W_p(\mu_w,P)^p
\right]
\le
c_0N_{\mathrm{eff}}(w)^{-\nu},
\]
which proves the claim.
\end{proof}

\begin{proposition}[Concentration bound with effective sample size]\label{prop:concentration-bound}
Let \(p\in(1,\infty)\) and assume \(\mathrm{diam}_{\kappa,\lambda}(\Xi)<\infty\). Let
\(P\in\mathcal P(\Xi)\), and let \(\nu\in(0,1/2)\) be the exponent from Lemma~\ref{lem:expectation-bound}. Let \(X_1,\dots,X_N\sim P\) be independent and identically distributed, and define \(\mu_w\coloneqq \sum_{i=1}^N w_i\delta_{X_i}\). There exist constants \(c_1,c_2>0\) such that, for all \(N\in\mathbb N\), \(w\in\Delta_N\), and
\(\varepsilon>0\),
\[
\mathbb P\!\left(W_p(\mu_w,P)\ge \varepsilon\right)
\le
\exp\!\left(-c_1N_{\mathrm{eff}}(w)\Big(\varepsilon^p-c_2N_{\mathrm{eff}}(w)^{-\nu}\Big)_+^2\right).
\]
In particular, whenever
\(\varepsilon \ge 2(c_2 N_{\mathrm{eff}}(w)^{-\nu})^{1/p}\), we have
\[
\mathbb P\!\left(
W_p(\mu_w,P)\ge \varepsilon
\right)
\le
\exp\!\left(
-\frac{c_1}{4}N_{\mathrm{eff}}(w)\varepsilon^{2p}
\right).
\]
\end{proposition}

\begin{proof}[Proof of Proposition~\ref{prop:concentration-bound}]
We apply McDiarmid's inequality to the random variable
\[
f(X_1,\dots,X_N)
\coloneqq
W_p\!\left(\sum_{i=1}^N w_i\delta_{X_i},P\right)^p .
\]
The key step is to verify a bounded-differences property. Changing one observation affects only the mass \(w_j\) attached to that observation; boundedness of \(\Xi\) therefore yields a bounded-differences constant proportional to \(w_j\).

Fix \(j\in\{1,\dots,N\}\), and let \(X_j'\in\Xi\) be an arbitrary replacement of the
\(j\)-th trajectory. Define
\[
\mu_w^{(j)}
\coloneqq
\sum_{i\neq j}w_i\delta_{X_i}+w_j\delta_{X_j'} .
\]
We claim that
\begin{equation}
\label{eq:bounded-difference-wp}
\left|
W_p(\mu_w,P)^p
-
W_p(\mu_w^{(j)},P)^p
\right|
\le
w_j\,\mathrm{diam}_{\kappa,\lambda}(\Xi)^p .
\end{equation}

To prove the claim, let
\(\pi^{(j)}\in\Pi(\mu_w^{(j)},P)\) be arbitrary. Since
\(\mu_w^{(j)}\) is discrete, we may disintegrate \(\pi^{(j)}\) according to
the indexed atoms of \(\mu_w^{(j)}\). Thus, for some conditional laws
\(\beta_1,\dots,\beta_N\in\mathcal P(\Xi)\),
\[
\pi^{(j)}(dx,dy)
=
\sum_{i\neq j}w_i\delta_{X_i}(dx)\,\beta_i(dy)
+
w_j\delta_{X_j'}(dx)\,\beta_j(dy).
\]
If some atoms coincide, we split the mass according to this indexed
representation; this does not change the marginal measure or the transport
cost.

Now construct a coupling \(\pi\in\Pi(\mu_w,P)\) by keeping the same
conditional laws and replacing only the source point \(X_j'\) by \(X_j\):
\[
\pi(dx,dy)
=
\sum_{i\neq j}w_i\delta_{X_i}(dx)\,\beta_i(dy)
+
w_j\delta_{X_j}(dx)\,\beta_j(dy).
\]
This coupling is feasible for \(W_p(\mu_w,P)\). Therefore,
\[
\begin{aligned}
W_p(\mu_w,P)^p
&\le
\int_{\Xi\times\Xi}
d_{\kappa,\lambda}(x,y)^p\,d\pi(x,y) \\
&=
\sum_{i\neq j}w_i
\int_\Xi d_{\kappa,\lambda}(X_i,y)^p\,\beta_i(dy)
+
w_j
\int_\Xi d_{\kappa,\lambda}(X_j,y)^p\,\beta_j(dy).
\end{aligned}
\]
Since \(X_j,X_j'\), and \(y\) all belong to \(\Xi\), we have
\(d_{\kappa,\lambda}(X_j,y)^p\le \mathrm{diam}_{\kappa,\lambda}(\Xi)^p\), and
hence
\[
d_{\kappa,\lambda}(X_j,y)^p
\le
d_{\kappa,\lambda}(X_j',y)^p
+
\mathrm{diam}_{\kappa,\lambda}(\Xi)^p .
\]
Substituting this bound gives
\[
\begin{aligned}
W_p(\mu_w,P)^p
&\le
\sum_{i\neq j}w_i
\int_\Xi d_{\kappa,\lambda}(X_i,y)^p\,\beta_i(dy)
+
w_j
\int_\Xi d_{\kappa,\lambda}(X_j',y)^p\,\beta_j(dy) \\
&\qquad
+
w_j\,\mathrm{diam}_{\kappa,\lambda}(\Xi)^p \\
&=
\int_{\Xi\times\Xi}
d_{\kappa,\lambda}(x,y)^p\,d\pi^{(j)}(x,y)
+
w_j\,\mathrm{diam}_{\kappa,\lambda}(\Xi)^p .
\end{aligned}
\]
Since \(\pi^{(j)}\in\Pi(\mu_w^{(j)},P)\) was arbitrary, taking the infimum
over all such couplings yields
\[
W_p(\mu_w,P)^p
-
W_p(\mu_w^{(j)},P)^p
\le
w_j\,\mathrm{diam}_{\kappa,\lambda}(\Xi)^p .
\]
Exchanging the roles of \(X_j\) and \(X_j'\) gives the reverse inequality, and
therefore proves (\ref{eq:bounded-difference-wp}).

Thus \(f\) satisfies McDiarmid's bounded-differences condition with constants
\[
a_j\coloneqq w_j\,\mathrm{diam}_{\kappa,\lambda}(\Xi)^p,
\qquad j=1,\dots,N.
\]
Hence, for every \(u>0\),
\[
\mathbb P\big(f-\mathbb E[f]\ge u\big)
\le
\exp\!\left(
-\frac{2u^2}{\sum_{j=1}^N a_j^2}
\right).
\]
Since
\[
\sum_{j=1}^N a_j^2
=
\mathrm{diam}_{\kappa,\lambda}(\Xi)^{2p}
\sum_{j=1}^N w_j^2
=
\mathrm{diam}_{\kappa,\lambda}(\Xi)^{2p}
N_{\mathrm{eff}}(w)^{-1},
\]
we obtain
\[
\mathbb P\big(f-\mathbb E[f]\ge u\big)
\le
\exp\!\left(
-c_1N_{\mathrm{eff}}(w)u^2
\right),
\qquad
c_1\coloneqq 2\,\mathrm{diam}_{\kappa,\lambda}(\Xi)^{-2p}.
\]

By Lemma~\ref{lem:expectation-bound}, there exists \(c_2>0\) such that
\[
\mathbb E[W_p(\mu_w,P)^p]
\le
c_2N_{\mathrm{eff}}(w)^{-\nu}.
\]
Therefore, for any \(\varepsilon>0\),
\[
\begin{aligned}
\mathbb P\!\left(W_p(\mu_w,P)\ge \varepsilon\right)
&=
\mathbb P\!\left(W_p(\mu_w,P)^p\ge \varepsilon^p\right) \\
&\le
\mathbb P\!\left(
W_p(\mu_w,P)^p-\mathbb E[W_p(\mu_w,P)^p]
\ge
\varepsilon^p-c_2N_{\mathrm{eff}}(w)^{-\nu}
\right) \\
&\le
\exp\!\left(
-c_1N_{\mathrm{eff}}(w)
\Big(
\varepsilon^p-c_2N_{\mathrm{eff}}(w)^{-\nu}
\Big)_+^2
\right).
\end{aligned}
\]
This proves the first claim.

For the simplified bound, suppose \(\varepsilon \ge 2(c_2N_{\mathrm{eff}}(w)^{-\nu})^{1/p}\).
Then \(c_2N_{\mathrm{eff}}(w)^{-\nu} \le 2^{-p}\varepsilon^p\).
Since \(p\in(1,\infty)\), we have \(1-2^{-p}\ge 1/2\), and therefore
\[
\Big(\varepsilon^p-c_2N_{\mathrm{eff}}(w)^{-\nu}\Big)_+
\ge
\frac12\varepsilon^p.
\]
Substituting this into the first concentration bound gives
\[
\mathbb P\!\left(
W_p(\mu_w,P)\ge \varepsilon
\right)
\le
\exp\!\left(
-\frac{c_1}{4}N_{\mathrm{eff}}(w)\varepsilon^{2p}
\right),
\]
which completes the proof.
\end{proof}

\subsubsection{Coupling Under Temporal Drift}

We first bound the Wasserstein distance between weighted empirical laws by pairing observations with the same index.

\begin{proposition}[Weighted coupling bound]
\label{prop:weighted-coupling}
Let $p\in(1,\infty)$, $\mu_w=\sum_{n=1}^N w_n\delta_{X_n}$, and $\widehat\mu_w=\sum_{n=1}^N w_n\delta_{\widehat X_n}$.
Then,
\begin{equation}
\label{eq:weighted-coupling}
W_p(\mu_w,\widehat\mu_w)^p
\le
\sum_{n=1}^N w_n\, d_{\kappa,\lambda}(X_n,\widehat X_n)^p.
\end{equation}
\end{proposition}

\begin{proof}[Proof of Proposition~\ref{prop:weighted-coupling}]
We construct an explicit coupling by assigning mass \(w_n\) to the pair
\((X_n,\widehat X_n)\) for each \(n=1,\dots,N\), namely
\(\pi_w \coloneqq \sum_{n=1}^N w_n\,\delta_{(X_n,\widehat X_n)}\). The first marginal of \(\pi_w\) is \(\mu_w\), and its second marginal is
\(\widehat\mu_w\). Hence \(\pi_w\in\Pi(\mu_w,\widehat\mu_w)\).

The \(p\)-Wasserstein distance on \((\Xi,d_{\kappa,\lambda})\) is defined by
\[
W_p(\mu_w,\widehat\mu_w)^p
=
\inf_{\pi\in\Pi(\mu_w,\widehat\mu_w)}
\int_{\Xi\times\Xi}
d_{\kappa,\lambda}(x,\widehat x)^p\,d\pi(x,\widehat x).
\]
Since \(\pi_w\) is an admissible coupling, evaluating the transport cost at
\(\pi_w\) gives the upper bound
\[
W_p(\mu_w,\widehat\mu_w)^p
\le
\int_{\Xi\times\Xi}
d_{\kappa,\lambda}(x,\widehat x)^p\,d\pi_w(x,\widehat x).
\]
Finally, substituting the definition of \(\pi_w\) 
\[
\int_{\Xi\times\Xi}
d_{\kappa,\lambda}(x,\widehat x)^p\,d\pi_w(x,\widehat x)
=
\sum_{n=1}^N
w_n\,d_{\kappa,\lambda}(X_n,\widehat X_n)^p.
\]
We therefore obtain the desired bound,
\[
W_p(\mu_w,\widehat\mu_w)^p
\le
\sum_{n=1}^N
w_n\,d_{\kappa,\lambda}(X_n,\widehat X_n)^p.
\]
\end{proof}

We next transfer the stationary concentration estimate to the nonstationary sample. The following coupling argument compares the weighted empirical measure generated by \(X_n\sim P_n\) with a weighted empirical measure drawn from the next-period law \(P_{N+1}\). The cost of this comparison is the drift term \(D_p(w)\varrho\). Older observations incur a larger transport cost because they must be coupled through more one-step distributional shifts before reaching \(P_{N+1}\).

\begin{lemma}[Drift-aware Wasserstein bound]
\label{lem:drift-aware-wasserstein-bound}
Suppose Assumption~\ref{ass:nonstationary-drift} holds, and let \(p\in(1,\infty)\) and \(w\in\Delta_N\).
Then there exist coupled pairs \((X_n,\widetilde X_n)\), independent across
\(n\), with \(X_n\sim P_n\) and \(\widetilde X_n\sim P_{N+1}\), such that, for
\(\mu_w\coloneqq\sum_{n=1}^N w_n\delta_{X_n}\) and
\(\widetilde\mu_w\coloneqq\sum_{n=1}^N w_n\delta_{\widetilde X_n}\),
\[
W_p(\mu_w,P_{N+1})
\le
W_p(\widetilde\mu_w,P_{N+1})+D_p(w)\varrho
\qquad \text{a.s.}
\]
In particular, \(\widetilde X_1,\dots,\widetilde X_N\) are i.i.d.~with law
\(P_{N+1}\).
\end{lemma}

\begin{proof}[Proof of Lemma~\ref{lem:drift-aware-wasserstein-bound}]
By Assumption~\ref{ass:nonstationary-drift}, for each
\(k=1,\dots,N\) there is a coupling of \(P_k\) and \(P_{k+1}\) whose \(d_{\kappa,\lambda}\)-cost is at most \(\varrho\) almost surely. Gluing these couplings from time \(n\) to time \(N+1\) yields a coupling of \(P_n\) and \(P_{N+1}\) under which
\[
d_{\kappa,\lambda}(X_n,\widetilde X_n)
\le
\sum_{k=n}^N \varrho
=
(N-n+1)\varrho
\qquad \text{a.s.}
\]
Choosing these couplings independently across \(n\) gives the coupled pairs in the statement. Since the couplings are independent across \(n\), and each proxy variable has marginal law \(P_{N+1}\), the variables \(\widetilde X_1,\dots,\widetilde X_N\) are i.i.d.~with law \(P_{N+1}\).

We compare \(\mu_w\) with the proxy measure \(\widetilde\mu_w\) using the triangle inequality. By the triangle inequality for \(W_p\),
\[
W_p(\mu_w,P_{N+1})
\le
W_p(\widetilde\mu_w,P_{N+1})
+
W_p(\mu_w,\widetilde\mu_w).
\]
The second term is controlled by the weighted coupling bound in Proposition~\ref{prop:weighted-coupling}. Applying it to the paired samples \((X_n,\widetilde X_n)\) yields
\[
W_p(\mu_w,\widetilde\mu_w)^p
\le
\sum_{n=1}^N w_n d_{\kappa,\lambda}(X_n,\widetilde X_n)^p .
\]
Using the pairwise drift bound and the definition of \(D_p(w)\), we obtain
\[
W_p(\mu_w,\widetilde\mu_w)^p
\le
\varrho^p\sum_{n=1}^N w_n (N-n+1)^p
=
\varrho^p D_p(w)^p .
\]
Hence \(W_p(\mu_w,\widetilde\mu_w)\le D_p(w)\varrho\). Substituting this into the previous triangle inequality gives
\[
W_p(\mu_w,P_{N+1})
\le
W_p(\widetilde\mu_w,P_{N+1})
+
D_p(w)\varrho
\qquad \text{a.s.}
\]
\end{proof}

\subsubsection{Drift-Aware Coverage Bound}

We next combine the weighted empirical concentration bound with the drift coupling. The resulting theorem gives the concentration inequality underlying the calibrated radius in Section~\ref{subsec:short-coverage}, expressed in terms of the effective sample size \(N_{\mathrm{eff}}(w)\) and cumulative drift index \(D_p(w)\). It specializes the nonstationary Wasserstein concentration guarantee of \citet{keehan2025don} to the trajectory space \((\Xi,d_{\kappa,\lambda})\). Let \((u)_+\coloneqq\max\{0,u\}\).

\begin{theorem}[Drift-aware coverage]
\label{thm:dlb-coverage}
Let \(p\in(1,\infty)\) and suppose Assumptions~\ref{ass:path-space} and
\ref{ass:nonstationary-drift} hold. Then there exist constants
\(c_1,c_2>0\) and an exponent \(\nu\in(0,1/2)\) such that, for every
\(w\in\Delta_N\) and every \(\varepsilon>0\),
\begin{equation}
\label{eq:dlb-bound}
\mathbb P\!\left(W_p(\mu_w,P_{N+1}) \ge \varepsilon\right)
\le
\exp\!\left(
-c_1 N_{\mathrm{eff}}(w)
\Big(
\big((\varepsilon-D_p(w)\varrho)_+^p
-
c_2N_{\mathrm{eff}}(w)^{-\nu}\big)_+
\Big)^2
\right).
\end{equation}
In particular, if \(\varepsilon
\ge D_p(w)\varrho + 2\big(c_2N_{\mathrm{eff}}(w)^{-\nu}\big)^{1/p}\),
then
\begin{equation}
\label{eq:dlb-bound-tail}
\mathbb P\!\left(W_p(\mu_w,P_{N+1}) \ge \varepsilon\right)
\le
\exp\!\left(
-\frac{c_1}{4}
N_{\mathrm{eff}}(w)
(\varepsilon-D_p(w)\varrho)_+^{2p}
\right).
\end{equation}
\end{theorem}

\begin{proof}[Proof of Theorem~\ref{thm:dlb-coverage}]
We compare the nonstationary weighted empirical measure \(\mu_w\) to a proxy
weighted empirical measure drawn from the target distribution \(P_{N+1}\). By
Lemma~\ref{lem:drift-aware-wasserstein-bound}, we may construct coupled pairs
\((X_n,\widetilde X_n)\), independent across \(n\), with \(X_n\sim P_n\) and
\(\widetilde X_n\sim P_{N+1}\), such that, for
\(\widetilde\mu_w\coloneqq\sum_{n=1}^N w_n\delta_{\widetilde X_n}\),
\[
W_p(\mu_w,P_{N+1})
\le
W_p(\widetilde\mu_w,P_{N+1})+D_p(w)\varrho
\qquad \text{a.s.}
\]
Consequently, for any \(\varepsilon>0\),
\[
\mathbb P\!\left(W_p(\mu_w,P_{N+1})\ge \varepsilon\right)
\le
\mathbb P\!\left(
W_p(\widetilde\mu_w,P_{N+1})
\ge
(\varepsilon-D_p(w)\varrho)_+
\right).
\]

The proxy variables \(\widetilde X_1,\dots,\widetilde X_N\) are i.i.d.~with law
\(P_{N+1}\). Hence Proposition~\ref{prop:concentration-bound}, applied with
\(P=P_{N+1}\) and with \(\mu_w\) replaced by \(\widetilde\mu_w\), gives, for
all \(t>0\),
\[
\mathbb P\!\left(W_p(\widetilde\mu_w,P_{N+1})\ge t\right)
\le
\exp\!\left(
-c_1N_{\mathrm{eff}}(w)
\big(t^p-c_2N_{\mathrm{eff}}(w)^{-\nu}\big)_+^2
\right).
\]
If \(t=(\varepsilon-D_p(w)\varrho)_+=0\), the desired bound is trivial because
the right-hand side of (\ref{eq:dlb-bound}) equals one. Otherwise, substituting
\(t=(\varepsilon-D_p(w)\varrho)_+\) into the preceding display yields
\[
\mathbb P\!\left(W_p(\mu_w,P_{N+1})\ge \varepsilon\right)
\le
\exp\!\left(
-c_1 N_{\mathrm{eff}}(w)
\Big(
\big((\varepsilon-D_p(w)\varrho)_+^p
-
c_2N_{\mathrm{eff}}(w)^{-\nu}\big)_+
\Big)^2
\right),
\]
which proves (\ref{eq:dlb-bound}).

The simplified tail bound follows by lower-bounding the residual radius. Suppose that
\[
\varepsilon \ge D_p(w)\varrho
+
2(c_2N_{\mathrm{eff}}(w)^{-\nu})^{1/p}.
\]
Then
\[
c_2N_{\mathrm{eff}}(w)^{-\nu}
\le
2^{-p}(\varepsilon-D_p(w)\varrho)_+^p .
\]
Since \(p>1\), we have \(1-2^{-p}\ge 1/2\), and therefore
\[
\Big(
(\varepsilon-D_p(w)\varrho)_+^p
-
c_2N_{\mathrm{eff}}(w)^{-\nu}
\Big)_+
\ge
\frac12(\varepsilon-D_p(w)\varrho)_+^p .
\]
Substituting this lower bound into (\ref{eq:dlb-bound}) gives
\[
\mathbb P\!\left(W_p(\mu_w,P_{N+1}) \ge \varepsilon\right)
\le
\exp\!\left(
-\frac{c_1}{4}
N_{\mathrm{eff}}(w)
(\varepsilon-D_p(w)\varrho)_+^{2p}
\right),
\]
which proves (\ref{eq:dlb-bound-tail}).
\end{proof}

\begin{remark}[Drift-aware truncation weights]
\label{rem:truncation-weights}
Theorem~2 of \citet{keehan2025don} shows that weights optimizing the bias--variance trade-off induced by \(N_{\mathrm{eff}}(w)\) and \(D_p(w)\varrho\) in the concentration bound \eqref{eq:dlb-bound-tail} admit a truncated-polynomial form \(w_n\propto \big(a_1-a_2(N-n+1)^p\big)_+\) for two scalars $a_1$ and $a_2$.
\end{remark}

\subsubsection{Calibration and Next-Period Risk}
\label{app:coverage-consequences}

\DriftAwareNextPeriodRisk*

\begin{proof}[Proof of Theorem~\ref{thm:drift-aware-next-period-risk}]
Choose $c_1,c_2$ and $\nu$ as in Theorem~\ref{thm:dlb-coverage}. Substituting $\varepsilon=\varepsilon_\varrho(w,\beta)$ into \eqref{eq:dlb-bound} gives
\[
\mathbb P\!\left(W_p(\mu_w,P_{N+1})\ge\varepsilon_\varrho(w,\beta)\right)
\le
\exp\!\left(-\log(1/\beta)\right)
=\beta.
\]
Hence $E\coloneqq\{P_{N+1}\in\mathcal B_{\varepsilon_\varrho(w,\beta)}(\mu_w)\}$ has probability at least $1-\beta$. On $E$, the law $P_{N+1}$ is feasible in the supremum defining $V_{\vartheta}(\varepsilon_\varrho(w,\beta))$, so $\mathbb E_{X\sim P_{N+1}}[\ell_{\mathrm{DH}}(\vartheta;X)]\le V_{\vartheta}(\varepsilon_\varrho(w,\beta))$ for every $\vartheta\in\Theta\times\mathbb R$. Since $E$ does not depend on $\vartheta$, this bound holds simultaneously for all $\vartheta$.
\end{proof}

\subsection{Two-Budget OT--\texorpdfstring{$\phi$}{phi} Adversarial Training}
\label{app:two-budget}
This subsection proves the two-budget sensitivity results. We first derive the joint first-order expansion of the OT--\(\phi\) robust value by reducing the ambiguity set to an outer reweighting problem and an inner transport problem. We then prove that the structured empirical attack is feasible and attains the same expansion to first order.

\subsubsection{Exact Reduction and Joint First-order Expansion}
\label{appsubsec:thm_two_budget_taylor}
The proof of Theorem~\ref{thm:two_budget_ot_phi_taylor} proceeds in three steps.
First, Proposition~\ref{prop:two_budget_exact_reduction} rewrites the robust problem as an outer optimization over anchor reweightings \(a\) and an inner transport problem \(T_a(\varepsilon)\). Second, Lemma~\ref{lem:uniform_ot_reweighted} provides a first-order expansion of \(T_a(\varepsilon)\), uniformly for reweightings close to \(\mathbf 1\). Third, Lemma~\ref{lem:phi_reweighting_expansion} gives the small-\(\tau\) expansion of the resulting pure reweighting problem. Since \(Q_\phi(\tau)\) collapses to \(\{\mathbf 1\}\) as \(\tau\downarrow0\), the reweighted transport sensitivity converges uniformly to \(\Upsilon_w(\vartheta)\), and the interaction between the two budgets is of smaller order. Combining these steps yields the claimed joint expansion. 

\begin{proposition}[Exact empirical reduction for the two-budget OT--$\phi$ set]
\label{prop:two_budget_exact_reduction}
For \(p\in(1,\infty)\), define
\[
Q_\phi(\tau)
\coloneqq
\Big\{
 a\in[0,\infty)^N:
 \sum_{n=1}^N w_n a_n = 1,
 \ \sum_{n=1}^N w_n\phi(a_n)\le \tau
\Big\},
\]
and for $a\in Q_\phi(\tau)$ define
\[
T_a(\varepsilon)
\coloneqq
\sup_{\nu_1',\dots,\nu_N'\in\mathcal P(\Xi)}
\Bigg\{
\sum_{n=1}^N w_n a_n\int_\Xi \ell_{\mathrm{DH}}(\vartheta;\widehat{x})\,\nu_n'(d\widehat{x})
:
\sum_{n=1}^N w_n a_n\int_\Xi d_{\kappa,\lambda}(X_n,\widehat{x})^p\,\nu_n'(d\widehat{x})\le \varepsilon^p
\Bigg\}.
\]
The robust value admits the exact reduction
\begin{equation}
\label{eq:two_budget_exact_reduction}
V_{\vartheta}(\varepsilon,\tau)
=
\sup_{a\in Q_\phi(\tau)} T_a(\varepsilon).
\end{equation}
\end{proposition}

\begin{proof}[Proof of Proposition~\ref{prop:two_budget_exact_reduction}]
We prove the claim in two steps. First, we show that any feasible
\(\eta\in\mathcal B_{\varepsilon,\tau}\) admits a representation
\(\eta=\sum_{n=1}^N w_n a_n\nu_n'\), where \(a\) satisfies the
\(\phi\)-divergence constraint and the conditional laws \(\nu_n'\) satisfy the
transport constraint. Second, we show that any such representation induces a
feasible law in \(\mathcal B_{\varepsilon,\tau}\).

Fix any feasible $\eta \in \mathcal{B}_{\varepsilon,\tau}$. By definition, there exist couplings $\pi,\pi' \in \mathcal P(\Xi \times \Xi)$ such that
\[
\pi_1=\mu_w,
\qquad
\pi'_2=\eta,
\qquad
\int_{\Xi\times\Xi} d_{\kappa,\lambda}(x,\widehat{x})^p\,d\pi'(x,\widehat{x})\le \varepsilon^p,
\qquad
D_\phi(\pi'\,\|\,\pi)\le \tau.
\]
Since \(\pi_1=\mu_w=\sum_{n=1}^N w_n\delta_{X_n}\), the first coordinate has support contained in the finite set \(\{X_1,\dots,X_N\}\). Hence \(\pi\) admits the disintegration
\[
\pi(dx,d\widehat{x})
=
\sum_{n=1}^N w_n\,\delta_{X_n}(dx)\,\nu_n(d\widehat{x}),
\]
where $\nu_n \in \mathcal P(\Xi)$ are the reference conditional laws.

Because $D_\phi(\pi'\,\|\,\pi)<\infty$, we necessarily have $\pi'\ll\pi$. Let $W=d\pi'/d\pi$ denote the Radon-Nikodym derivative, and define its slices as
\[
W_n(\widehat{x})\coloneqq W(X_n,\widehat{x}),
\qquad
a_n\coloneqq\int_\Xi W_n(\widehat{x})\,\nu_n(d\widehat{x}).
\]
The quantity $a_n$ represents the total mass assigned by $\pi'$ to the slice $\{x=X_n\}$ relative to $\pi$. For each $n$ such that $a_n > 0$, define the normalized conditional probability measure $\nu_n'$ by
\[
\nu_n'(d\widehat{x})\coloneqq\frac{W_n(\widehat{x})}{a_n}\,\nu_n(d\widehat{x}),
\]
and choose an arbitrary probability measure $\nu_n'\in\mathcal P(\Xi)$ when $a_n=0$. This construction cleanly separates the distorted coupling into a mass component $a_n$ and a normalized conditional slice $\nu_n'$. 

The distorted coupling therefore factors as:
\[
\pi'(dx,d\widehat{x})
=
\sum_{n=1}^N w_n a_n\,\delta_{X_n}(dx)\,\nu_n'(d\widehat{x}).
\]
Taking the second marginal (integrating out $dx$) directly yields:
\[
\pi'_2
=
\eta
=
\sum_{n=1}^N w_n a_n\nu_n',
\qquad\text{with}\quad
\sum_{n=1}^N w_n a_n=1,
\]
thus any feasible \(\eta\) can be written as a finite mixture with weights \(w_n a_n\). Consequently, its expected loss equals the objective appearing in
\(T_a(\varepsilon)\):
\[
\mathbb E_{\widehat{X}\sim\eta}\big[\ell_{\mathrm{DH}}(\vartheta;\widehat{X})\big]
=
\sum_{n=1}^N w_n a_n\int_\Xi \ell_{\mathrm{DH}}(\vartheta;\widehat{x})\,\nu_n'(d\widehat{x}).
\]
Moreover, the transport constraint is preserved under this representation:
\[
\int_{\Xi\times\Xi} d_{\kappa,\lambda}(x,\widehat{x})^p\,d\pi'(x,\widehat{x})
=
\sum_{n=1}^N w_n a_n\int_\Xi d_{\kappa,\lambda}(X_n,\widehat{x})^p\,\nu_n'(d\widehat{x})
\le \varepsilon^p.
\]
Finally, by Jensen's inequality applied to the convex function $\phi$,
\begin{align*}
D_\phi(\pi'\,\|\,\pi)
&=
\sum_{n=1}^N w_n\int_\Xi \phi\big(W_n(\widehat{x})\big)\,\nu_n(d\widehat{x})\\
&\ge
\sum_{n=1}^N w_n\phi\!\left(\int_\Xi W_n(\widehat{x})\,\nu_n(d\widehat{x})\right)
=
\sum_{n=1}^N w_n\phi(a_n),
\end{align*}
so $a\in Q_\phi(\tau)$ since $D_\phi(\pi'\,\|\,\pi)\leq \tau$. Thus, for this specific $a$, the conditional slices $(\nu_1',\dots,\nu_N')$ are feasible for the inner supremum $T_a(\varepsilon)$, attaining the exact same objective value. This proves
\[
V_{\vartheta}(\varepsilon,\tau)
\le
\sup_{a\in Q_\phi(\tau)}T_a(\varepsilon).
\]

Conversely, fix any $a\in Q_\phi(\tau)$ and any $\nu_1',\dots,\nu_N'\in\mathcal P(\Xi)$ feasible in the definition of $T_a(\varepsilon)$. Set
\[
\pi(dx,d\widehat{x})
\coloneqq
\sum_{n=1}^N w_n\,\delta_{X_n}(dx)\,\nu_n'(d\widehat{x}),
\qquad
\pi'(dx,d\widehat{x})
\coloneqq
\sum_{n=1}^N w_n a_n\,\delta_{X_n}(dx)\,\nu_n'(d\widehat{x}).
\]
Then $\pi_1=\mu_w$, $\pi'_2=\sum_{n=1}^N w_n a_n\nu_n'$, and the transport cost evaluates to:
\[
\int_{\Xi\times\Xi} d_{\kappa,\lambda}(x,\widehat{x})^p\,d\pi'(x,\widehat{x})
=
\sum_{n=1}^N w_n a_n\int_\Xi d_{\kappa,\lambda}(X_n,\widehat{x})^p\,\nu_n'(d\widehat{x})
\le \varepsilon^p.
\]
Moreover, $\pi'\ll\pi$.  Where several components overlap,
\(d\pi'/d\pi\) is a weighted average of the corresponding \(a_n\)'s, so applying Jensen's inequality with the convexity of $\phi$ yields
\[
D_\phi(\pi'\,\|\,\pi)
\le
\sum_{n=1}^N w_n\phi(a_n)
\le \tau.
\]
Hence $\eta\coloneqq\sum_{n=1}^N w_n a_n\nu_n'$ belongs to $\mathcal B_{\varepsilon,\tau}$ and attains exactly the same objective value as in $T_a(\varepsilon)$. Taking the supremum over all $a \in Q_\phi(\tau)$ and their corresponding feasible conditional slices $(\nu_1', \dots, \nu_N')$ proves the reverse inequality
\[
V_{\vartheta}(\varepsilon,\tau)
\ge
\sup_{a\in Q_\phi(\tau)}T_a(\varepsilon).
\]

\end{proof}

Next, we show the \emph{uniform} first-order expansion of the Wasserstein ambiguity.
\begin{lemma}[Uniform OT expansion under small reweightings]
\label{lem:uniform_ot_reweighted}
Let \(p\in(1,\infty)\) and suppose Assumption~\ref{ass:loss-regularity} holds, with
\(\|\cdot\|_{\kappa,\lambda,p}^{\ast}\) and \(h\) as in
Proposition~\ref{prop:dual-geometry-main}. For \(\delta\in(0,1)\), define
\begin{equation}
\label{eq:A_delta}
A_\delta
\coloneqq
\Big\{
a\in[0,\infty)^N:
\sum_{n=1}^N w_n a_n=1,\ 
\|a-\mathbf 1\|_\infty\le\delta
\Big\}.
\end{equation}
For each $a \in A_\delta$, set $\widehat w_n\coloneqq w_na_n$ for $n=1,\dots,N$ and define the aggregate first-order sensitivity
\[
\Upsilon_{\widehat w}(\vartheta)
\coloneqq
\Bigg(\sum_{n=1}^N \widehat w_n \|g_n(\vartheta)\|_{\kappa,\lambda,p}^{\ast\,q}\Bigg)^{1/q}.
\]
Here $\widehat w=\widehat w(a)$; its dependence on $a$ is implicit in the notation throughout this lemma and its proof.
Then, for every fixed $\delta \in (0,1)$, the following expansion holds uniformly over $a \in A_\delta$:
\begin{equation}
\label{eq:uniform_ot_reweighted}
\sup_{a\in A_\delta}
\left|
T_a(\varepsilon)
-
\sum_{n=1}^N \widehat w_n L_n(\vartheta)
-
\varepsilon\Upsilon_{\widehat w}(\vartheta)
\right|
=
o(\varepsilon)
\qquad \text{as } \varepsilon \downarrow 0.
\end{equation}
\end{lemma}

\begin{proof}[Proof of Lemma~\ref{lem:uniform_ot_reweighted}]
We establish matching upper and lower bounds uniformly over \(a\in A_\delta\). The upper bound follows from a local linearization of \(\ell_{\mathrm{DH}}\) around the nominal anchors; the lower bound uses an explicit deterministic perturbation along the dual maximizing directions.

Let
\[
B_\ell \coloneqq
\sup_{\widehat x\in\Xi}
|\ell_{\mathrm{DH}}(\vartheta;\widehat x)|<\infty
\]
and fix \(\zeta>0\). By the local expansion in
Assumption~\ref{ass:loss-regularity} and the finiteness of \(N\), there exists
\(r_\zeta>0\) such that, for every \(n=1,\dots,N\),
\begin{equation}
\label{eq:linearization_error}
\big|\ell_{\mathrm{DH}}(\vartheta;\widehat{x}) - L_n(\vartheta) - \langle g_n(\vartheta), \widehat{x}-X_n\rangle\big| \le \zeta d_{\kappa,\lambda}(X_n, \widehat{x})
\end{equation}
for all $\widehat{x} \in \Xi$ with $d_{\kappa,\lambda}(X_n, \widehat{x}) \le r_\zeta$.

To prove the upper bound, fix $a \in A_\delta$ and let $(\nu_1', \dots, \nu_N')$ be any feasible family of conditional slices for the inner problem $T_a(\varepsilon)$. Define the local neighborhood $B_n \coloneqq \{\widehat{x} \in \Xi : d_{\kappa,\lambda}(X_n, \widehat{x}) \le r_\zeta\}$. Since $\Xi = B_n \cup B_n^c$, we decompose the expected loss as 
\[
\int_\Xi \ell_{\mathrm{DH}}(\vartheta; \widehat{x}) \, \nu_n'(d\widehat{x}) = \int_{B_n} \ell_{\mathrm{DH}}(\vartheta; \widehat{x}) \, \nu_n'(d\widehat{x}) + \int_{B_n^c} \ell_{\mathrm{DH}}(\vartheta; \widehat{x}) \, \nu_n'(d\widehat{x}).
\]
On $B_n$, the linearization gives $\ell_{\mathrm{DH}}(\vartheta;\widehat{x}) \le L_n(\vartheta) + \langle g_n(\vartheta), \widehat{x} - X_n\rangle + \zeta d_{\kappa,\lambda}(X_n, \widehat{x})$, while on $B_n^c$, global boundedness implies $\ell_{\mathrm{DH}}(\vartheta;\widehat{x}) - L_n(\vartheta) \le 2B_\ell$.

Integrating these bounds and utilizing the identity $L_n(\vartheta) \nu_n'(B_n) + L_n(\vartheta) \nu_n'(B_n^c) = L_n(\vartheta)$, we obtain
\[
\int_\Xi \ell_{\mathrm{DH}}(\vartheta; \widehat{x})\,\nu_n'(d\widehat{x}) \le L_n(\vartheta) + \int_{B_n}\langle g_n(\vartheta), \widehat{x} - X_n\rangle\,\nu_n'(d\widehat{x}) + \zeta \int_{B_n} d_{\kappa,\lambda}(X_n, \widehat{x})\,\nu_n'(d\widehat{x}) + 2B_\ell\nu_n'(B_n^c).
\]
By the definition of the dual norm, $\langle g_n(\vartheta), \widehat{x} - X_n\rangle \le \|g_n(\vartheta)\|_{\kappa,\lambda,p}^* d_{\kappa,\lambda}(X_n,\widehat{x})$. Since this upper bound is non-negative, we can extend the integration domain from $B_n$ to the entire space $\Xi$ by monotonicity. Taking the weighted sum with $\widehat w_n$ yields the nominal reweighted loss alongside three distinct error terms: a gradient bound ($I_{\mathrm{grad}}$), a linearization error ($I_{\mathrm{lin}}$), and a tail bound ($I_{\mathrm{tail}}$):
\begin{align*}
&\sum_{n=1}^N \widehat w_n
\int_\Xi \ell_{\mathrm{DH}}(\vartheta;\widehat x)\,
\nu_n'(d\widehat x) \\
&\le
\sum_{n=1}^N \widehat w_n L_n(\vartheta)
+ \underbrace{
\sum_{n=1}^N \widehat w_n
\int_\Xi
\|g_n(\vartheta)\|_{\kappa,\lambda,p}^*
d_{\kappa,\lambda}(X_n,\widehat x)\,
\nu_n'(d\widehat x)
}_{I_{\mathrm{grad}}}  \\
&\quad
+ \underbrace{\zeta
\sum_{n=1}^N \widehat w_n
\int_\Xi d_{\kappa,\lambda}(X_n,\widehat x)\,
\nu_n'(d\widehat x)
}_{I_{\mathrm{lin}}}
+
\underbrace{
2B_\ell\sum_{n=1}^N \widehat w_n\nu_n'(B_n^c)
}_{I_{\mathrm{tail}}}.
\end{align*}
For the gradient bound $I_{\mathrm{grad}}$, H\"older's inequality and the feasibility of $(\nu_n')$ imply that  $\sum \widehat w_n \int_\Xi d_{\kappa,\lambda}(X_n, \widehat{x})^p \, \nu_n'(d\widehat{x}) \le \varepsilon^p$, hence we have
\begin{align*}
I_{\mathrm{grad}}
&\le
\left(
\sum_{n=1}^N \widehat w_n
\|g_n(\vartheta)\|_{\kappa,\lambda,p}^{*\,q}
\right)^{1/q}
\left(
\sum_{n=1}^N \widehat w_n
\int_\Xi d_{\kappa,\lambda}(X_n,\widehat x)^p\,\nu_n'(d\widehat x)
\right)^{1/p}  \\
&\le
\Upsilon_{\widehat w}(\vartheta)\varepsilon.
\end{align*}
Similarly, the linearization error satisfies $
I_{\mathrm{lin}}
\le
\zeta\varepsilon$.
Finally, since \(d_{\kappa,\lambda}(X_n,\widehat x)>r_\zeta\) on \(B_n^c\), Markov's inequality gives
\[
\nu_n'(B_n^c)
\le
r_\zeta^{-p}
\int_\Xi d_{\kappa,\lambda}(X_n,\widehat x)^p\,\nu_n'(d\widehat x),
\]
and therefore
\[
I_{\mathrm{tail}}
\le
2B_\ell r_\zeta^{-p}\varepsilon^p.
\]

Substituting these three bounds and taking the supremum over all feasible slices yields the uniform upper bound:
\begin{equation}
\label{eq:uniform_ot_upper}
T_a(\varepsilon) \le \sum_{n=1}^N \widehat w_n L_n(\vartheta) + \varepsilon\Upsilon_{\widehat w}(\vartheta) + \zeta\varepsilon + 2B_\ell r_\zeta^{-p}\varepsilon ^p.
\end{equation}

We now construct a matching lower bound by evaluating a specific discrete attack. If $\Upsilon_w(\vartheta) = 0$, then $\|g_n(\vartheta)\|_{\kappa,\lambda,p}^{\ast\,q}=0$ for every $n$ (since $w_n>0$), meaning $\Upsilon_{\widehat w}(\vartheta)=0$ for all $a\in A_\delta$. Setting $\nu_n'=\delta_{X_n}$ yields $T_a(\varepsilon)\ge \sum_{n=1}^N \widehat w_n L_n(\vartheta)$, which is sufficient. We therefore assume $\Upsilon_w(\vartheta) > 0$.

For every $a\in A_\delta$, the constraint $\|a-\mathbf{1}\|_\infty \le \delta$ implies $a_n \ge 1-\delta$. Thus, $\Upsilon_{\widehat w}(\vartheta)^q = \sum_{n=1}^N w_n a_n \|g_n(\vartheta)\|_{\kappa,\lambda,p}^{\ast\,q} \ge (1-\delta)\sum_{n=1}^N w_n \|g_n(\vartheta)\|_{\kappa,\lambda,p}^{\ast\,q} = (1- \delta)\Upsilon_w(\vartheta)^q$, meaning $\Upsilon_{\widehat w}(\vartheta)$ is uniformly bounded away from $0$. Define the deterministic adversarial perturbations:
\[
\widehat X_n^a \coloneqq X_n + \varepsilon \, h(g_n(\vartheta)) \, \|g_n(\vartheta)\|_{\kappa,\lambda,p}^{*\,(q-1)} \, \Upsilon_{\widehat w}(\vartheta)^{1-q}.
\]
The transport distance is $d_{\kappa,\lambda}(X_n,\widehat X_n^a) = \varepsilon\, \|g_n(\vartheta)\|_{\kappa,\lambda,p}^{*\,(q-1)} \Upsilon_{\widehat w}(\vartheta)^{1-q}$ using that $h(g_n(\vartheta))$ is a unit vector. Recalling that $(q-1)p=q$, the total transport cost evaluates exactly to the budget:
\[
\sum_{n=1}^N \widehat w_n d_{\kappa,\lambda}(X_n,\widehat X_n^a)^p = \varepsilon^p\,\Upsilon_{\widehat w}(\vartheta)^{-q} \sum_{n=1}^N \widehat w_n \|g_n(\vartheta)\|_{\kappa,\lambda,p}^{\ast\,q} = \varepsilon^p\,\Upsilon_{\widehat w}(\vartheta)^{-q} \Upsilon_{\widehat w}(\vartheta)^q = \varepsilon^p.
\]
Hence, the discrete point-mass choice $\nu_n'=\delta_{\widehat X_n^a}$ is feasible for $T_a(\varepsilon)$. Furthermore, $d_{\kappa,\lambda}(X_n,\widehat X_n^a)\to 0$ uniformly in $a\in A_\delta$ as $\varepsilon\downarrow0$, and by Assumption~\ref{ass:loss-regularity} these perturbations remain in $\Xi$.

Using the linearization error \eqref{eq:linearization_error} at the perturbed points yields:
\[
\ell_{\mathrm{DH}}(\vartheta; \widehat X_n^a) \ge L_n(\vartheta) + \langle g_n(\vartheta),\widehat X_n^a-X_n\rangle - \zeta d_{\kappa,\lambda}(X_n,\widehat X_n^a).
\]
Multiplying by $\widehat w_n$ and summing over $n$, we explicitly evaluate the gradient inner product. We proceed in three steps: first, we substitute the definition of the perturbation $\widehat X_n^a$; second, we apply the dual maximizer property $\langle g_n(\vartheta), h(g_n(\vartheta)) \rangle = \|g_n(\vartheta)\|_{\kappa,\lambda,p}^*$; and third, we recognize the definition of $\Upsilon_{\widehat w}(\vartheta)^q = \sum_{n=1}^N \widehat w_n \|g_n(\vartheta)\|_{\kappa,\lambda,p}^{\ast\,q}$. This yields:
\begin{align*}
\sum_{n=1}^N \widehat w_n \langle g_n(\vartheta),\widehat X_n^a-X_n\rangle 
&= \sum_{n=1}^N \widehat w_n \Big\langle g_n(\vartheta), \, \varepsilon h(g_n(\vartheta)) \|g_n(\vartheta)\|_{\kappa,\lambda,p}^{*\,(q-1)} \Upsilon_{\widehat w}(\vartheta)^{1-q} \Big\rangle \\
&= \varepsilon \Upsilon_{\widehat w}(\vartheta)^{1-q} \sum_{n=1}^N \widehat w_n \langle g_n(\vartheta), h(g_n(\vartheta)) \rangle \|g_n(\vartheta)\|_{\kappa,\lambda,p}^{*\,(q-1)} \\
&= \varepsilon \Upsilon_{\widehat w}(\vartheta)^{1-q} \sum_{n=1}^N \widehat w_n \|g_n(\vartheta)\|_{\kappa,\lambda,p}^{*\,q} \\
&= \varepsilon \Upsilon_{\widehat w}(\vartheta)^{1-q} \Upsilon_{\widehat w}(\vartheta)^q \\
&= \varepsilon \Upsilon_{\widehat w}(\vartheta).
\end{align*}
Substituting this into the aggregated linear bound directly confirms the lower bound:
\begin{align*}
T_a(\varepsilon) &\ge \sum_{n=1}^N \widehat w_nL_n(\vartheta) + \varepsilon \Upsilon_{\widehat w}(\vartheta) - \zeta \sum_{n=1}^N \widehat w_n d_{\kappa,\lambda}(X_n,\widehat X_n^a)\\
&\ge \sum_{n=1}^N \widehat w_nL_n(\vartheta) + \varepsilon\Upsilon_{\widehat w}(\vartheta) - \zeta\varepsilon,
\end{align*}
uniformly over $a\in A_\delta$. Combining this with (\ref{eq:uniform_ot_upper}) and noting $\varepsilon^p=o(\varepsilon)$ as $\varepsilon \downarrow 0$ proves (\ref{eq:uniform_ot_reweighted}).

\end{proof}

% ============================================================

\paragraph{Pure reweighting expansion.}
We isolate the effect of the \(\phi\)-divergence budget \(\tau\) by analyzing the pure reweighting problem. Recall that $L_n(\vartheta) \coloneqq \ell_{\mathrm{DH}}(\vartheta; X_n)$ and
\[
\bar L_w(\vartheta) \coloneqq \sum_{n=1}^N w_n L_n(\vartheta),
\qquad
\sigma_w(\vartheta) \coloneqq \Bigg(\sum_{n=1}^N w_n\big(L_n(\vartheta)-\bar L_w(\vartheta)\big)^2\Bigg)^{1/2}.
\]
Define the maximum expected loss achievable through anchor reweighting under budget \(\tau\) by
\[
R_\phi(\tau) \coloneqq \sup_{a\in Q_\phi(\tau)} \sum_{n=1}^N w_n a_n L_n(\vartheta).
\]

With these aggregate statistics defined, we can apply a standard second-order Taylor expansion to the $\phi$-divergence constraint to extract the asymptotic rate of the pure reweighting problem.

\begin{lemma}[Small-budget expansion of the reweighting problem]
\label{lem:phi_reweighting_expansion}
Let \(\phi\) be a locally regular entropy function. Then, as
\(\tau\downarrow0\), the reweighting problem admits the expansion
\begin{equation}
\label{eq:phi_reweighting_expansion}
R_\phi(\tau) = \bar L_w(\vartheta) + \sqrt{\frac{2}{\phi''(1)}}\,\sigma_w(\vartheta)\sqrt{\tau} + o(\sqrt{\tau}).
\end{equation}
\end{lemma}

\begin{proof}[Proof of Lemma \ref{lem:phi_reweighting_expansion}]
We prove the expansion by establishing matching upper and lower bounds in the small-budget regime. The key idea is that for sufficiently small $\tau$, every feasible reweighting stays strictly close to the uniform baseline $a = \mathbf{1}$. This forces the divergence constraint into a local regime where $\phi$ is regular and reduces the optimization to maximizing a linear objective subject to a weighted variance constraint.

If $\sigma_w(\vartheta)=0$, then $L_n(\vartheta) = \bar L_w(\vartheta)$ for all $n$ (since $w_n>0$). For any feasible $a \in Q_\phi(\tau)$, the expected loss is exactly $\sum_{n=1}^N w_n a_n L_n(\vartheta) = \bar L_w(\vartheta) \sum_{n=1}^N w_n a_n = \bar L_w(\vartheta)$, yielding $R_\phi(\tau)=\bar L_w(\vartheta)$ for all $\tau \ge 0$, which trivially satisfies the expansion. We therefore assume $\sigma_w(\vartheta)>0$.

\noindent Define the centered losses $Y_n(\vartheta) \coloneqq L_n(\vartheta)-\bar L_w(\vartheta)$. By definition of the weighted empirical mean and variance, they satisfy $\sum_{n=1}^N w_n Y_n(\vartheta) = 0$ and $\sum_{n=1}^N w_n Y_n(\vartheta)^2 = \sigma_w(\vartheta)^2$.

Fix $0<\zeta<\phi''(1)/2$. Because \(\phi\) is locally regular, \(1\) is its unique minimizer and
\(\phi'(1)=0\). Since $\phi$ is twice continuously differentiable near $1$, a second-order Taylor expansion yields $\phi(1+s) = \frac{1}{2}\phi''(1)s^2 + o(s^2)$. Thus, there exists $\delta_\zeta\in(0,1)$ such that for all $|s|\le \delta_\zeta$, the divergence is locally bounded by quadratics:
\begin{equation}
\label{eq:phi_local_quadratic_bounds}
\left(\frac{\phi''(1)}{2}-\zeta\right)s^2
\le \phi(1+s)
\le \left(\frac{\phi''(1)}{2}+\zeta\right)s^2.
\end{equation}

For this quadratic bound to be useful, we must guarantee that small budgets restrict all feasible weights to this neighborhood. For any feasible $a \in Q_\phi(\tau)$, the constraints $a_n \ge 0$ and $\sum_{n=1}^N w_n a_n = 1$ strictly bound the weights: $a_n \le 1/\min_k w_k$. Define the compact set $K_{\delta_\zeta} \coloneqq [0, 1/\min_k w_k] \setminus (1-\delta_\zeta, 1+\delta_\zeta)$. Since $1 \notin K_{\delta_\zeta}$ and $\phi$ is strictly positive away from $1$, its minimum on this set is strictly positive: $m_\zeta \coloneqq \min_{s\in K_{\delta_\zeta}} \phi(s)>0$. If any weight strays such that $|a_n-1|\ge \delta_\zeta$, then $a_n \in K_{\delta_\zeta}$, forcing $\tau \ge \sum_{k=1}^N w_k\phi(a_k) \ge w_n\phi(a_n) \ge (\min_k w_k) m_\zeta$. Consequently, for all $\tau < (\min_k w_k)m_\zeta$, every feasible $a \in Q_\phi(\tau)$ must satisfy $|a_n-1|<\delta_\zeta$ for all $n$.

We next bound the objective. For any feasible $a$, define the deviations $u_n \coloneqq a_n-1$. Since $\sum_{n=1}^N w_n u_n = 0$, the objective simplifies to $\sum_{n=1}^N w_n a_n L_n(\vartheta) = \sum_{n=1}^N w_n (1+u_n)(\bar L_w(\vartheta)+Y_n(\vartheta)) = \bar L_w(\vartheta) + \sum_{n=1}^N w_n u_n Y_n(\vartheta)$. For sufficiently small $\tau$, we know $|u_n|<\delta_\zeta$. Applying the lower bound in (\ref{eq:phi_local_quadratic_bounds}) with $s=u_n$ bounds the variance of the weights:
\[
\tau
\ge \sum_{n=1}^N w_n\phi(1+u_n)
\ge \left(\frac{\phi''(1)}{2}-\zeta\right)\sum_{n=1}^N w_n u_n^2.
\]
Applying the Cauchy--Schwarz inequality to the objective deviation yields:
\[
\sum_{n=1}^N w_n u_n Y_n(\vartheta) \le \left(\sum_{n=1}^N w_n u_n^2\right)^{1/2} \left(\sum_{n=1}^N w_n Y_n(\vartheta)^2\right)^{1/2} \le \sqrt{\frac{\tau}{\phi''(1)/2-\zeta}}\,\sigma_w(\vartheta).
\]
Taking the supremum over all feasible $a$ gives the uniform upper bound:
\begin{equation}
\label{eq:phi_upper_local}
R_\phi(\tau) \le \bar L_w(\vartheta) + \frac{\sigma_w(\vartheta)}{\sqrt{\phi''(1)/2-\zeta}}\sqrt{\tau}.
\end{equation}

For the matching lower bound, we construct an explicit feasible reweighting. Define the weights $a_n \coloneqq 1 + \alpha_\tau Y_n(\vartheta)$, where $\alpha_\tau \coloneqq \sqrt{\tau} \big/ \big(\sigma_w(\vartheta)\sqrt{\phi''(1)/2+\zeta}\big)$. By construction, $\sum_{n=1}^N w_n a_n=1$. For sufficiently small $\tau$, we have $|\alpha_\tau Y_n(\vartheta)| \le \delta_\zeta$, allowing us to use the upper bound in (\ref{eq:phi_local_quadratic_bounds}) to verify feasibility:
\[
\sum_{n=1}^N w_n\phi(a_n)
\le
\left(\frac{\phi''(1)}{2}+\zeta\right)\alpha_\tau^2 \sum_{n=1}^N w_n Y_n(\vartheta)^2
=
\left(\frac{\phi''(1)}{2}+\zeta\right)\alpha_\tau^2 \sigma_w(\vartheta)^2
=
\tau,
\]
confirming $a \in Q_\phi(\tau)$. Evaluating the expected loss under $a$ yields $\sum_{n=1}^N w_n a_n L_n(\vartheta) = \bar L_w(\vartheta) + \alpha_\tau\sum_{n=1}^N w_n Y_n(\vartheta) L_n(\vartheta)$. Since $\sum_{n=1}^N w_n Y_n(\vartheta) \bar L_w(\vartheta) = 0$, the inner product simplifies to $\sum_{n=1}^N w_n Y_n(\vartheta) (Y_n(\vartheta) + \bar L_w(\vartheta)) = \sigma_w(\vartheta)^2$. Substitution provides the lower bound:
\[
R_\phi(\tau)
\ge
\bar L_w(\vartheta)
+
\frac{\sigma_w(\vartheta)}{\sqrt{\phi''(1)/2+\zeta}}\sqrt{\tau}.
\]

Combining this lower bound with (\ref{eq:phi_upper_local}) and letting $\zeta\downarrow 0$ completes the proof of (\ref{eq:phi_reweighting_expansion}).
\end{proof}

\paragraph{Joint OT--\(\phi\) expansion.}
We now combine the exact reduction, the uniform OT expansion, and the pure reweighting expansion.

\JointSensitivityAnalysis*

\begin{proof}[Proof of Theorem~\ref{thm:two_budget_ot_phi_taylor}] We first use the exact reduction in Proposition~\ref{prop:two_budget_exact_reduction} to rewrite the problem as a supremum over reweightings \(a\in Q_\phi(\tau)\). We then apply the uniform OT expansion from Lemma~\ref{lem:uniform_ot_reweighted} and verify that it applies to all feasible \(a\in Q_\phi(\tau)\) for small \(\tau\). Finally, we separate the reweighting and OT contributions by an elementary supremum inequality and combine the result with the small-\(\tau\) expansion of the pure \(\phi\)-reweighting problem in Lemma~\ref{lem:phi_reweighting_expansion}.
For each \(a\in Q_\phi(\tau)\), set \(\widehat w_n\coloneqq w_na_n\) for \(n=1,\dots,N\); the dependence \(\widehat w=\widehat w(a)\) is implicit throughout this proof.

By Proposition~\ref{prop:two_budget_exact_reduction},
\[
V_{\vartheta}(\varepsilon,\tau)
=
\sup_{a\in Q_\phi(\tau)} T_a(\varepsilon).
\]
We next verify that the uniform OT expansion applies over the entire feasible set. Fix \(\delta\in(0,1)\). We show that \(Q_\phi(\tau)\subseteq A_\delta\)
for all sufficiently small \(\tau\). Define \(w_{\min}=\min_n w_n\). Note that $w_{\min}>0$ and for every
\(a\in Q_\phi(\tau)\), the normalization gives \(0\le a_n\le 1/w_{\min}\). Since \(1\) is the unique zero of \(\phi\), lower semicontinuity implies 
\[
m_\delta
\coloneqq
\inf_{s\in[0,1/w_{\min}]\setminus(1-\delta,1+\delta)}
\phi(s)>0.
\]
If some coordinate satisfied \(|a_j-1|\ge\delta\), then $ \sum_{n=1}^N w_n\phi(a_n) \ge w_j\phi(a_j) \ge w_{\min}m_\delta$. Hence, for \(\tau<w_{\min}m_\delta\), no such coordinate exists, and therefore
\(Q_\phi(\tau)\subseteq A_\delta\).
Therefore, Lemma~\ref{lem:uniform_ot_reweighted} gives
\begin{equation}
\label{eq:sup-elementary}
\sup_{a\in Q_\phi(\tau)}
\left| T_a(\varepsilon)
- \sum_{n=1}^N \widehat w_nL_n(\vartheta)
- \varepsilon\Upsilon_{\widehat w}(\vartheta) \right|
=
o(\varepsilon),
\qquad\text{as }\varepsilon\downarrow0,
\end{equation}
uniformly for all sufficiently small \(\tau\). Equivalently, we may write
\[
T_a(\varepsilon)
=
\sum_{n=1}^N \widehat w_nL_n(\vartheta)
+
\varepsilon\Upsilon_{\widehat w}(\vartheta)
+
r_a(\varepsilon,\tau),
\]
where the remainder satisfies
\begin{equation}
\label{eq:uniform-remainder-Qa}
\sup_{a\in Q_\phi(\tau)} |r_a(\varepsilon,\tau)|
=
o(\varepsilon),
\qquad\text{as }\varepsilon\downarrow0.
\end{equation}
We now separate the contribution of the reweighting term and the OT correction.
Recall the definition of \(R_\phi(\tau)\), and compare
\[
\sup_{a\in Q_\phi(\tau)}
\Big(
\sum_{n=1}^N \widehat w_nL_n(\vartheta)
+ \varepsilon\Upsilon_{\widehat w}(\vartheta)
+ r_a(\varepsilon,\tau)\Big)
\]
with $R_\phi(\tau)+\varepsilon\Upsilon_w(\vartheta)$.
To control this difference, we use the elementary inequality $\Big|
\sup_{a\in Q}(A_a + B_a)
- \sup_{a\in Q} A_a
- b \Big|\le\sup_{a\in Q} |B_a - b|$,
which follows from
$A_a + B_a \le A_a + b + |B_a - b|$
together with the corresponding reverse inequality obtained by writing
\(A_a = (A_a + B_a) - B_a\).

Applying this inequality with $A_a\coloneqq\sum_{n=1}^N \widehat w_nL_n(\vartheta)$, $
B_a\coloneqq\varepsilon\Upsilon_{\widehat w}(\vartheta)+r_a(\varepsilon,\tau)$ and $
b\coloneqq\varepsilon\Upsilon_w(\vartheta)$ 
we obtain
\begin{align*}
\Big|V_{\vartheta}(\varepsilon,\tau)
- R_\phi(\tau)
- \varepsilon\Upsilon_w(\vartheta)\Big|
&\le
\sup_{a\in Q_\phi(\tau)}|r_a(\varepsilon,\tau)|
+
\varepsilon
\sup_{a\in Q_\phi(\tau)}
\big|\Upsilon_{\widehat w}(\vartheta)-\Upsilon_w(\vartheta)\big|.
\end{align*}
By (\ref{eq:uniform-remainder-Qa}), the first term is \(o(\varepsilon)\). The second term is controlled by the collapse of \(Q_\phi(\tau)\) to \(\mathbf 1\). Since \(a\mapsto \Upsilon_{\widehat w(a)}(\vartheta)\) is continuous, and
\(Q_\phi(\tau)\) collapses to \(\{\mathbf 1\}\) as \(\tau\downarrow0\),
in the sense that $
\sup_{a\in Q_\phi(\tau)}\|a-\mathbf 1\|_\infty \to 0 $, we have $\sup_{a\in Q_\phi(\tau)}
\big|
\Upsilon_{\widehat w}(\vartheta)-\Upsilon_w(\vartheta)
\big|
\to 0$. Hence,
\begin{equation}
\label{eq:upsilon-uniform-bound}
\varepsilon
\sup_{a\in Q_\phi(\tau)}
\big|\Upsilon_{\widehat w}(\vartheta)-\Upsilon_w(\vartheta)\big|
= o(\varepsilon).    
\end{equation}

We conclude that
\begin{equation}
\label{eq:joint-reduction-rigorous}
V_{\vartheta}(\varepsilon,\tau)
=
R_\phi(\tau)
+ \Upsilon_w(\vartheta)\varepsilon
+ o(\varepsilon).
\end{equation}

Finally, with Lemma~\ref{lem:phi_reweighting_expansion} we can substitute the expansion of $
R_\phi(\tau)$ into (\ref{eq:joint-reduction-rigorous}) and get
\[
V_{\vartheta}(\varepsilon,\tau)
= \bar L_w(\vartheta)
+ \Upsilon_w(\vartheta)\varepsilon
+ \sqrt{\frac{2}{\phi''(1)}}\,\sigma_w(\vartheta)\,\sqrt{\tau}
+ o(\varepsilon+\sqrt{\tau}).
\]
The budget axes are included: for \(\tau=0\), \(Q_\phi(0)=\{\mathbf 1\}\), so the exact reduction and Lemma~\ref{lem:uniform_ot_reweighted} give the pure Wasserstein expansion directly. For \(\varepsilon=0\), zero transport cost implies \(T_a(0)=\sum_n \widehat w_nL_n(\vartheta)\), so \(V_{\vartheta}(0,\tau)=R_\phi(\tau)\) and Lemma~\ref{lem:phi_reweighting_expansion} applies.
\end{proof}

\subsubsection{Structured Two-Budget Attack}
\label{appsubsec:two_budget_adversarial}

The structured attack combines the leading \(\phi\)-reweighting direction with the Wasserstein steepest-ascent perturbation conditional on the reweighted anchors. The proof verifies feasibility of the reweighting and transport budgets and then evaluates the loss expansion along this constructed law. 

Recall \(Y_n(\vartheta)\), \(a_n\), \(\widehat w_n\coloneqq w_na_n\), \(\Upsilon_{\widehat w}(\vartheta)\),
\(\widehat X_n\), and \(\widehat\eta_{\varepsilon,\tau}\) from
Section~\ref{sec:short-two-budget}.

\FirstOrderStructuredOTPhiAttack*

\begin{proof}[Proof of Theorem~\ref{thm:structured_phi_ot_first_order_attack_main}]
We first verify the reweighting constraint. If \(\sigma_w(\vartheta)=0\), then
\(Y_n(\vartheta)=0\) for all \(n\), because \(w_n>0\), and the convention
\(a=\mathbf 1\) gives \(a\in Q_\phi(\tau)\). Suppose therefore that
\(\sigma_w(\vartheta)>0\). Since $\sum_{n=1}^N w_nY_n(\vartheta)=0,$ we have
\[
\sum_{n=1}^N w_n a_n
= 1+\alpha_\tau\sum_{n=1}^N w_nY_n(\vartheta)
= 1.
\]
The nonnegativity and \(\phi\)-budget constraints hold by the definition of
\(\alpha_\tau\). Hence \(a\in Q_\phi(\tau)\).

We next derive the small-\(\tau\) behavior of \(\alpha_\tau\). Since
\(\phi\) is \(C^2\) around \(1\), with \(\phi(1)=\phi'(1)=0\), we have
\[
\sum_{n=1}^N w_n\phi(1+\alpha Y_n(\vartheta))
=
\frac{\phi''(1)}{2}\alpha^2
\sum_{n=1}^N w_nY_n(\vartheta)^2
+
o(\alpha^2)
=
\frac{\phi''(1)}{2}\alpha^2\sigma_w(\vartheta)^2
+
o(\alpha^2).
\]
Thus the maximal feasible \(\alpha_\tau\) satisfies
\begin{equation}
\label{eq:alpha_tau_asymptotic}
\alpha_\tau
=
\frac{\sqrt{2\tau}}{\sqrt{\phi''(1)}\,\sigma_w(\vartheta)}
+
o(\sqrt{\tau}),
\qquad \tau\downarrow0 .
\end{equation}

We show that the perturbed trajectories are admissible for all
sufficiently small \((\varepsilon,\tau)\): \(\widehat X_n\in\Xi\) . If \(g_n(\vartheta)=0\), then
\(\widehat X_n=X_n\in\Xi\). Otherwise, since $X_n$ is in the interior of $\Xi$, \(X_n\in\operatorname{int}(\Xi)\), and \(\|h(g_n(\vartheta))\|_{\kappa,\lambda,p}=1\), there exists \(\bar t_n>0\) such that
\[
X_n+t h(g_n(\vartheta))\in\Xi,
\qquad 0\le t\le\bar t_n.
\]
The step length in direction \(h(g_n(\vartheta))\) is
\[
t_n(\varepsilon,\tau)
=
\varepsilon
\|g_n(\vartheta)\|_{\kappa,\lambda,p}^{\ast(q-1)}
\Upsilon_{\widehat w}(\vartheta)^{\,1-q}.
\]
Since \(\alpha_\tau\to0\), the constructed vector \(a=a(\tau)\) satisfies \(a\to\mathbf 1\) and \(\widehat w\to w\) as \(\tau\downarrow0\). Hence, the quantities \(\Upsilon_{\widehat w}(\vartheta)\) remain bounded away from zero whenever the transport sensitivity $\Upsilon_w$ is nonzero. Therefore, \(t_n(\varepsilon,\tau)\downarrow0\) as \(\varepsilon\downarrow0\), uniformly over the finitely many samples. Therefore, \(\widehat X_n\in\Xi\) for all sufficiently small \((\varepsilon,\tau)\).

We next verify the transport budget constraint. If \(\Upsilon_{\widehat w}(\vartheta)=0\), then
\(\widehat X_n=X_n\) for all \(n\), so the transport cost is zero.
Assume therefore that \(\Upsilon_{\widehat w}(\vartheta)>0\). By construction, $\widehat X_n-X_n
=
\varepsilon
h(g_n(\vartheta))
\|g_n(\vartheta)\|_{\kappa,\lambda,p}^{\ast(q-1)}
\Upsilon_{\widehat w}(\vartheta)^{\,1-q}$. 
Since \(d_{\kappa,\lambda}\) is induced by
\(\|\cdot\|_{\kappa,\lambda,p}\), and since
\(\|h(g_n(\vartheta))\|_{\kappa,\lambda,p}=1\) whenever \(g_n(\vartheta)\neq0\), we get
\[
d_{\kappa,\lambda}
\bigl(\widehat X_n,X_n\bigr)
=
\varepsilon
\|g_n(\vartheta)\|_{\kappa,\lambda,p}^{\ast(q-1)}
\Upsilon_{\widehat w}(\vartheta)^{\,1-q}.
\]
For \(g_n(\vartheta)=0\), both sides are zero, so the same identity still holds. Taking
the \(p\)-th power and using \((q-1)p=q\) gives
\[
d_{\kappa,\lambda}
\bigl(\widehat X_n,X_n\bigr)^p
=
\varepsilon^p\,
\Upsilon_{\widehat w}(\vartheta)^{-q}
\|g_n(\vartheta)\|_{\kappa,\lambda,p}^{\ast q}.
\]
Therefore, weighting with $\widehat w_n$ and summing over $n$ shows that the transport budget constraint is satisfied:
\[
\begin{aligned}
\sum_{n=1}^N
\widehat w_n
d_{\kappa,\lambda}
\bigl(\widehat X_n,X_n\bigr)^p
&=
\varepsilon^p\,
\Upsilon_{\widehat w}(\vartheta)^{-q}
\sum_{n=1}^N
\widehat w_n
\|g_n(\vartheta)\|_{\kappa,\lambda,p}^{\ast q}  \\
&=
\varepsilon^p\,
\Upsilon_{\widehat w}(\vartheta)^{-q}
\Upsilon_{\widehat w}(\vartheta)^{q}
=
\varepsilon^p .
\end{aligned}
\]
To verify feasibility in the full two-budget ambiguity set, define
\[
\pi
\coloneqq
\sum_{n=1}^N w_n\,\delta_{(X_n,\widehat X_n)},
\qquad
\pi'
\coloneqq
\sum_{n=1}^N \widehat w_n\,\delta_{(X_n,\widehat X_n)}.
\]
Then \(\pi_1=\mu_w\) and \(\pi'_2=\widehat\eta_{\varepsilon,\tau}\). Moreover, \(\pi'\ll\pi\). If several pairs \((X_n,\widehat X_n)\) coincide, the density \(d\pi'/d\pi\) is the corresponding \(w\)-weighted average of the coefficients \(a_n\); hence Jensen's inequality and \(a\in Q_\phi(\tau)\) give
\[
D_\phi(\pi'\|\pi)
\le
\sum_{n=1}^N w_n\phi(a_n)
\le\tau.
\]
Together with the transport-budget calculation above, this shows that \(\pi,\pi'\) are admissible in Definition~\ref{def:two_budget_phi_ot_clean}, and therefore
\[
\widehat\eta_{\varepsilon,\tau}
\in
\mathcal B_{\varepsilon,\tau}(\mu_w).
\]

It remains to evaluate the expected loss under \(\widehat\eta_{\varepsilon,\tau}\). Since \(\alpha_\tau\to0\), the constructed vector \(a=a(\tau)\) satisfies \(a\to\mathbf 1\) and \(\widehat w\to w\) as \(\tau\downarrow0\). Hence, continuity gives \(\Upsilon_{\widehat w}(\vartheta)=\Upsilon_w(\vartheta)+o(1)\) (see \eqref{eq:upsilon-uniform-bound}). If \(\Upsilon_w(\vartheta)=0\), then all \(g_n(\vartheta)=0\), because \(w_n>0\), and the transport
part contributes nothing at first order. Otherwise, \(\Upsilon_{\widehat w}(\vartheta)\) is
bounded away from zero for all sufficiently small \(\tau\). Hence the
constructed perturbations satisfy $
d_{\kappa,\lambda}
\bigl(\widehat X_n,X_n\bigr)
=
O(\varepsilon)$ uniformly in \(n\). By Assumption~\ref{ass:loss-regularity} 
\[
\ell_{\mathrm{DH}}(\vartheta;\widehat X_n)
=
L_n(\vartheta)
+ \left\langle
g_n(\vartheta),\widehat X_n-X_n
\right\rangle
+ o(\varepsilon).
\]
Multiplying by \(\widehat w_n\), summing over \(n\), and using boundedness of
\(a\) for small \(\tau\), gives
\[
\sum_{n=1}^N
\widehat w_n
\ell_{\mathrm{DH}}(\vartheta;\widehat X_n)
=
\sum_{n=1}^N
\widehat w_n L_n(\vartheta)
+ \sum_{n=1}^N \widehat w_n
\left\langle
g_n(\vartheta),\widehat X_n-X_n
\right\rangle
+ o(\varepsilon).
\]
The linear transport term is
\[
\begin{aligned}
\sum_{n=1}^N
\widehat w_n
\left\langle
g_n(\vartheta),\widehat X_n-X_n
\right\rangle
&=
\varepsilon
\Upsilon_{\widehat w}(\vartheta)^{1-q}
\sum_{n=1}^N
\widehat w_n
\|g_n(\vartheta)\|_{\kappa,\lambda,p}^{\ast q} 
=
\varepsilon\Upsilon_{\widehat w}(\vartheta).
\end{aligned}
\]
Therefore, we conclude that
\begin{equation}
\label{eq:structured_two_budget_OT}
\mathbb E_{X\sim\widehat\eta_{\varepsilon,\tau}}
\big[\ell_{\mathrm{DH}}(\vartheta;X)\big]
=
\sum_{n=1}^N
\widehat w_n L_n(\vartheta)
+ \varepsilon\Upsilon_{\widehat w}(\vartheta)
+ o(\varepsilon).    
\end{equation}

If \(\sigma_w(\vartheta)>0\), using $Y_n(\vartheta) = L_n(\vartheta) - \bar L_w(\vartheta)$, the loss evaluates to 
\[
\sum_{n=1}^N w_n a_n L_n(\vartheta)
=
\sum_{n=1}^N w_n \big(1+\alpha_\tau Y_n(\vartheta)\big)\big(\bar L_w(\vartheta)+Y_n(\vartheta)\big)
=
\bar L_w(\vartheta)+\alpha_\tau\sigma_w(\vartheta)^2.
\]
Using (\ref{eq:alpha_tau_asymptotic}), this becomes
\begin{equation}
\label{eq:structured_two_budget_phidiv}
\sum_{n=1}^N \widehat w_n L_n(\vartheta)
=
\bar L_w(\vartheta)
+ \sqrt{\frac{2}{\phi''(1)}}\,\sigma_w(\vartheta)\sqrt{\tau}
+ o(\sqrt{\tau}).    
\end{equation}

If \(\sigma_w(\vartheta)=0\), the same display holds with zero \(\sqrt{\tau}\)-term.
Finally,
\[
\varepsilon\Upsilon_{\widehat w}(\vartheta)
=
\varepsilon\Upsilon_w(\vartheta)
+ o(\varepsilon).
\]
Combining \eqref{eq:structured_two_budget_OT} and \eqref{eq:structured_two_budget_phidiv} yields
\[
\mathbb E_{X\sim\widehat\eta_{\varepsilon,\tau}}
\big[\ell_{\mathrm{DH}}(\vartheta;X)\big]
=
\bar L_w(\vartheta)
+ \Upsilon_w(\vartheta)\varepsilon
+ \sqrt{\frac{2}{\phi''(1)}}\,\sigma_w(\vartheta)\sqrt{\tau}
+ o\big(\varepsilon+\sqrt{\tau}\big).
\]
Comparing this expansion with Theorem~\ref{thm:two_budget_ot_phi_taylor} gives
\[
V_{\vartheta}(\varepsilon,\tau)
-
\mathbb E_{X\sim\widehat\eta_{\varepsilon,\tau}}
\big[\ell_{\mathrm{DH}}(\vartheta;X)\big]
=
o\big(\varepsilon+\sqrt{\tau}\big),
\]
which is \eqref{eq:short-attack-accuracy} and completes the proof.
\end{proof}

\section{Numerical Setup and Supplementary Results}
\label{app:numerical-details}

This appendix provides further details on the numerical experiments. We first specify the Optimized Certainty Equivalent (OCE) risk measures (\ref{app:oce-risk-measures}), contingent-claim payoffs (\ref{app:claims}), the multistep \textsc{WRAP} algorithm (\ref{app:iterative-attack}), and the data-generating processes (\ref{app:financial-time-series-models}). We define the performance metrics (\ref{subsec:performance-metrics}) and present the experimental setup and results (\ref{app:experiment-setup}), covering the Heston experiments (\ref{subsec:heston-results}) and the GAD study (\ref{subsec:market-results}), together with the corresponding calibration data and descriptive statistics. This section concludes with a comparison of the one-step and multistep \textsc{WRAP} variants in the stationary Heston setting (\ref{app:additional-results}).

\subsection{Optimized Certainty Equivalent Risk Measures}
\label{app:oce-risk-measures}

The numerical experiments adopt two standard Optimized Certainty Equivalent (OCE) risk measures. OCEs were introduced by \citet{ben1986expected} and later developed as a class of convex risk measures by \citet{ben2007old}. Let \(\mathcal Z\) denote the domain of real-valued random variables under consideration. A normalized OCE loss is a convex and nondecreasing function \(\ell:\mathbb R\to\mathbb R\) satisfying \(\ell(0)=0\) and \(\ell(u)\ge u\) for all \(u\in\mathbb R\).
The associated OCE risk measure is
\begin{equation}
\label{eq:app-oce}
\rho_\ell(Z)
=
\inf_{m\in\mathbb R}
\left\{
m+\mathbb E\!\left[\ell(-Z-m)\right]
\right\}.
\end{equation}
Here \(Z\) denotes terminal profit and loss, so larger values of \(Z\) are preferred and \(-Z\) is the corresponding loss. For \(Z=\mathrm{PnL}(\theta,X)\), the objective inside the infimum in \eqref{eq:app-oce} gives the samplewise deep-hedging loss \(\ell_{\mathrm{DH}}(\vartheta;X)\) defined in \eqref{eq:dh_loss}.

\paragraph{Conditional Value at Risk.}
Conditional Value at Risk (CVaR), also known as Expected Shortfall, is a coherent risk measure that averages losses beyond a high quantile. The variational representation below follows \citet{rockafellar2000optimization,rockafellar2002conditional}. For a confidence level \(\alpha\in(0,1)\), define
\[
\ell_{\alpha}(u)
=
\frac{(u)_+}{1-\alpha},
\qquad
(u)_+\coloneqq\max\{u,0\}.
\]
Substituting this loss into \eqref{eq:app-oce} yields
\begin{equation}
\label{eq:app-cvar-oce}
\rho_{\alpha}(Z)
=
\inf_{m\in\mathbb R}
\left\{
m+
\frac{1}{1-\alpha}
\mathbb E\!\left[(-Z-m)_+\right]
\right\}
=
\mathrm{CVaR}_{\alpha}(-Z).
\end{equation}
Thus, CVaR evaluates the average of the largest \(1-\alpha\) fraction of losses, with the standard adjustment for probability mass at the quantile. The minimizers in \eqref{eq:app-cvar-oce} are the \(\alpha\)-quantiles of \(-Z\); in particular, one may choose \(m=\mathrm{VaR}_{\alpha}(-Z)\).

\paragraph{Entropic risk.}
The entropic risk measure is obtained from the exponential OCE loss \citep{ben2007old}. For a risk-aversion parameter \(\lambda>0\), define
\[
\ell_{\lambda}(u)
=
\frac{\exp(\lambda u)-1}
{\lambda}.
\]
Whenever
\(\mathbb E[\exp(-\lambda Z)]<\infty\),
\begin{equation}
\label{eq:app-entropic-oce}
\begin{aligned}
\rho_{\lambda}(Z)
&=
\inf_{m\in\mathbb R}
\left\{
m+
\frac{1}{\lambda}
\mathbb E\!\left[
\exp\!\big(\lambda(-Z-m)\big)-1
\right]
\right\} \\
&=
\frac{1}{\lambda}
\log\mathbb E\!\left[\exp(-\lambda Z)\right].
\end{aligned}
\end{equation}
The entropic risk measure is smooth and convex and penalizes large losses exponentially. Unlike CVaR, it is not positively homogeneous.

\paragraph{Convex combinations with expected loss.}
Let \(\rho\) be an OCE risk measure with a convex, nondecreasing loss \(\ell\) satisfying \(\ell(0)=0\) and \(\ell(u)\ge u\) for all \(u\in\mathbb R\). Then, for \(\omega\in[0,1]\), the loss \(\ell_\omega(u)\coloneqq(1-\omega)u+\omega\,\ell(u)\) satisfies the same conditions. Substitution into \eqref{eq:app-oce} yields, for integrable \(Z\) with finite \(\rho(Z)\),
\[
\rho_{\omega}(Z)
=(1-\omega)\mathbb E[-Z]+\omega\rho(Z).
\]
In particular, Mean--CVaR (M--CVaR) admits the representation
\begin{equation}
\label{eq:oce-mean-cvar}
\begin{aligned}
\rho_{\omega,\alpha}(Z)
&=
\inf_{m\in\mathbb R}
\left\{
m+
\mathbb E\!\left[(1-\omega)(-Z-m) + \frac{\omega}{1-\alpha}(-Z-m)_+\right]
\right\} \\
&=
(1-\omega)\mathbb E[-Z] + \omega\,\mathrm{CVaR}_{\alpha}(-Z).
\end{aligned}
\end{equation}

\subsection{Contingent Claims}
\label{app:claims}
We collect the payoff functionals used in the numerical examples. Recall that the agent aims to hedge a contingent claim with payoff function $\Phi:\Xi\to\mathbb{R}$ at maturity. Let \(K>0\) denote the strike price and $S^1 \in \mathbb R ^{T+1}$ denote the price process of the considered asset.

The Heston experiments consider a \emph{European call option} on the stock with
payoff
\[
\Phi^{\mathrm{European}}(X)
=
(S_T^1-K)_+.
\]
The GAD experiments consider an \emph{arithmetic Asian call} with payoff
\[
\Phi^{\mathrm{Asian}}(X)
=
\left(
\frac{1}{T+1}\sum_{t=0}^{T}S^1_t-K
\right)_+.
\]

\subsection{Multistep \textsc{WRAP}}
\label{app:iterative-attack}

\begin{algorithm}[H]
\caption{Multistep \textbf{W}asserstein \textbf{R}eweighting \textbf{A}dversarial  \textbf{P}erturbation (\textsc{WRAP})}
\label{alg:drift-aware-wbpgd}
\footnotesize
\begin{algorithmic}[1]
\Require Weighted sample \(\{(X_n,w_n)\}_{n=1}^N\), batch
\(\mathcal I\subseteq\{1,\dots,N\}\) with \(|\mathcal I|=M\), hedge
parameter \(\vartheta\), transport budget \(\varepsilon\), reweighting
budget \(\tau\), inner steps \(K\), stepsize \(\eta\)
\Ensure Adversarial weighted batch
\(\{(\widehat X_n,\widehat w_n)\}_{n\in\mathcal I}\)

\State Set \(\bar w_n\gets w_n/\sum_{j\in\mathcal I}w_j\),
\(a_n\gets1\), \(\widehat w_n\gets\bar w_n\), \(b_n\gets0\), and
\(d_n\gets0\) for all \(n\in\mathcal I\).

\For{\(k=1,\dots,K\)}
    \State Set \(\widehat X_n\gets X_n+b_nd_n\) for all \(n\in\mathcal I\).
    \State Compute \(L_n\gets\ell_{\mathrm{DH}}(\vartheta;\widehat X_n)\)
    for all \(n\in\mathcal I\).

    \Statex \StartBlock{Phi} \textbf{\(\phi\)-divergence reweighting}

    \If{\(\tau=0\)}
        \State Set \(a_n\gets1\) for all \(n\in\mathcal I\).
    \Else
        \State Set \(\bar L_{\mathcal I}\gets\sum_{n\in\mathcal I}\bar w_nL_n\)
        and \(Y_n\gets L_n-\bar L_{\mathcal I}\) for all \(n\in\mathcal I\).
        \State Choose the largest \(\alpha_\tau\ge0\) such that
        \(1+\alpha_\tau Y_n\ge0\) for all \(n\in\mathcal I\) and
        \(\sum_{n\in\mathcal I}\bar w_n\phi(1+\alpha_\tau Y_n)\le\tau\).
        \State Set \(a_n\gets1+\alpha_\tau Y_n\) for all
        \(n\in\mathcal I\).
    \EndIf
    \EndBlock{Phi}{PhiShade}
    \State Set \(\widehat w_n\gets\bar w_na_n\) for all
    \(n\in\mathcal I\).

    \Statex \StartBlock{OT} \textbf{Optimal-transport perturbation}
    \State Compute \(G_n\gets\nabla_{\widehat X}
    \ell_{\mathrm{DH}}(\vartheta;\widehat X_n)\) for all \(n\in\mathcal I\).

    \State Set \(\widetilde d_n\gets d_n+(\eta/\varepsilon)h(G_n)\) and
    \(d_n\gets\widetilde d_n/\|\widetilde d_n\|_{\kappa,\lambda,p}\) for all
    \(n\in\mathcal I\).

    \State Set \(\Upsilon_{\mathcal I}\gets
    \big(\sum_{n\in\mathcal I}\widehat w_n
    [\langle G_n,d_n\rangle]_+^q\big)^{1/q}\) and update
    \(b_n\gets b_n+\eta[\langle G_n,d_n\rangle]_+^{q-1}
    \Upsilon_{\mathcal I}^{1-q}\) for all \(n\in\mathcal I\).

    \State Rescale
    \(b_n\gets
    \varepsilon b_n/
    \big(\sum_{j\in\mathcal I}\widehat w_jb_j^p\big)^{1/p}\)
    for all \(n\in\mathcal I\).
    \EndBlock{OT}{OTShade}
\EndFor

\State Set \(\widehat X_n\gets X_n+b_nd_n\) and
\(\widehat w_n\gets\bar w_na_n\) for all \(n\in\mathcal I\).
\State \Return \(\{(\widehat X_n,\widehat w_n)\}_{n\in\mathcal I}\).

\end{algorithmic}
\end{algorithm}

The multistep attack refines Algorithm~\ref{alg:drift-aware-one-step-attack} while preserving the same two-budget structure. The reweighting step is unchanged: at each inner iteration, the losses are centered and the adversarial multipliers are updated according to \(a_n=1+\alpha_\tau Y_n\). The transport step differs. Rather than applying the closed-form first-order perturbation once, the algorithm writes \(\widehat X_n=X_n+b_nd_n\), where \(d_n\) is a normalized perturbation direction and \(b_n\ge0\) is the pathwise allocation of the transport budget. Both are updated iteratively and projected back onto the weighted transport constraint.

Algorithm~\ref{alg:drift-aware-wbpgd} is therefore not the exact first-order construction of Theorem~\ref{thm:structured_phi_ot_first_order_attack_main}, but an iterative refinement that builds on \citet{he2025distributional}. The \(\phi\)-block continues to tilt probability toward high-loss scenarios through the linearized \(\alpha_\tau\)-step, while the OT block updates the perturbation directions and budget allocations with learning rate \(\eta\). At each adversarial iterate \(\widehat X_n\), the reweighting is recomputed before the transport update proceeds along the local ascent direction. When \(\tau=0\), the method reduces to a drift-aware iterative optimal-transport attack; when \(\varepsilon=0\), the transport updates vanish and the method reduces to the reweighting-only case of Algorithm~\ref{alg:drift-aware-one-step-attack}.

In the numerical experiments, this multistep procedure serves as the stronger iterative benchmark against the one-step attack in Algorithm~\ref{alg:drift-aware-one-step-attack}.

\subsection{Data-Generating Processes}
\label{app:financial-time-series-models}

We conduct the numerical experiments under two market models: the Heston stochastic-volatility model and the generalized affine diffusion (GAD) model. Both models are sampled at the trading dates \(t_j=j\Delta t\), \(j=0,\ldots,T\), with maturity \(\mathcal T=T\Delta t\). Following the notation introduced in Section~\ref{sec:short-setup}, \(S_j\) denotes the vector of tradable prices, \(X_j\) the observed market state, and \(X=(X_0,\ldots,X_T)\in\Xi\) the resulting trajectory.

\paragraph{Heston model.}
The first specification is the Heston stochastic-volatility model \citep{heston1993closed},
\begin{equation}
\label{eq:heston-model-short}
\begin{aligned}
dS_t
&=
\mu S_t\,dt
+
\sqrt{v_t}\,S_t\,dW_t^{S},\\
dv_t
&=
\kappa(\theta_{\mathrm H}-v_t)\,dt
+
\xi\sqrt{v_t}\,dW_t^{v},
\end{aligned}
\qquad
d\langle W^{S},W^{v}\rangle_t=\rho\,dt .
\end{equation}
Here \(S_0>0\) and \(v_0\geq0\) are the initial stock price and variance, \(\mu\in\mathbb R\) is the stock-price drift, \(\kappa>0\) the mean-reversion rate, \(\theta_{\mathrm H}>0\) the long-run variance, \(\xi>0\) the volatility-of-variance parameter, and \(\rho\in[-1,1]\) the instantaneous correlation between the Brownian motions \(W^S\) and \(W^v\). We use \(\theta_{\mathrm H}\) to distinguish the long-run variance from the hedging-policy parameter \(\theta\).

We sample the variance process from its exact Cox--Ingersoll--Ross transition and approximate the integrated variance entering the stock-price update by the trapezoidal rule. The exact variance transition preserves nonnegativity without requiring the Feller condition.

In the Heston experiments, the claim is written on the stock, while the hedging portfolio may trade both the stock and a variance-linked instrument. At day \(t\), the hedging network takes \(X_{0:t}=(\log S,v)_{0:t}\) as input, and portfolio gains are computed from changes in both tradable prices.

\paragraph{Generalized affine diffusion model.}
The second specification follows the generalized affine diffusion model used by \citet{lutkebohmert2022robust} for robust deep hedging. It extends the classical affine-diffusion case, covered by the affine-process framework of \citet{duffie2003affine}. In continuous-time notation, we write
\begin{equation}
\label{eq:gad-model-short}
dS_t
=
(b_0+b_1S_t)\,dt
+
(a_0+a_1S_t)^\gamma\,dW_t .
\end{equation}
Here \(S_0>0\) is the initial stock price, \(b_0,b_1\in\mathbb R\) determine the affine drift, and \(a_0,a_1,\gamma\) determine the state-dependent volatility. The stock is the sole tradable asset and the only state variable observed by the hedging network, \textit{i.e.}~$X=S$.

\subsection{Performance Metrics}
\label{subsec:performance-metrics}
We evaluate each trained hedge on a held-out test set \(X_1,\dots,X_M\in\Xi\). Recall that the terminal profit and loss generated by a hedging strategy $\theta$ along a path $X$ is
\begin{equation*}
\mathrm{PnL}(\theta,X)
\;=\;
p_0
+\sum_{t=0}^{T-1}\delta_t^\theta(X_{0:t})^\top (S_{t+1}-S_t)
- \Phi(X).
\end{equation*}

\noindent \emph{Mean PnL.}
The average realized profit and loss is
\begin{equation*}
\label{eq:metric-avg-pnl}
\mathrm{Mean PnL}
\coloneqq
\frac{1}{M}\sum_{i=1}^M \mathrm{PnL}(\theta,X_i).
\end{equation*}

\noindent \emph{Empirical CVaR.}
The empirical CVaR of the hedging loss is
\begin{equation*}
\label{eq:metric-cvar-objective}
\mathrm{CVaR}
\coloneqq
m+
\frac{1}{1-\alpha_{\mathrm{cvar}}}
\frac{1}{M}
\sum_{i=1}^M
\big(-\mathrm{PnL}(\theta,X_i)-m\big)_+.
\end{equation*}
Here $m$ is the empirical $\alpha_{\mathrm{cvar}}$-quantile of the test losses $-\mathrm{PnL}(\theta,X_i)$, with $\alpha_{\mathrm{cvar}}=0.95$.

\noindent \emph{Empirical M--CVaR.}
The empirical combination of mean loss and CVaR is
\[
\mathrm{M\text{--}CVaR}
\coloneqq
-(1-\omega)\,\mathrm{Mean PnL}
+\omega\,\mathrm{CVaR}.
\]
Here $\omega\in[0,1]$ controls the weight on CVaR, with $\omega=0.5$ in the M--CVaR experiments.

\noindent \emph{Empirical entropic risk.}
The empirical entropic risk is
\[
\mathrm{EntR}
\coloneqq
\frac{1}{\lambda}
\log\!\left(
\frac{1}{M}
\sum_{i=1}^M
\exp\!\left(
-\lambda\,\mathrm{PnL}(\theta,X_i)
\right)
\right).
\]
Here $\lambda>0$ is the risk-aversion parameter.

\noindent \emph{Transaction costs.}
The average cumulative transaction cost (\textit{cf.}~\eqref{eq:transaction-cost}) per test path is
\[
\mathrm{Total TC}
\coloneqq
\frac{1}{M}\sum_{i=1}^M\operatorname{TC_c}(\theta,X_i),
\]
where $\operatorname{TC_c}(\theta,X_i)$ denotes the cost of opening and rebalancing the hedge.

\noindent \emph{Turnover.}
Additionally, turnover measures how aggressively the learned strategy rebalances
along the test paths. Let
\[
\mathrm{Turnover}
\coloneqq
\frac{1}{MTr}
\sum_{i=1}^{M}
\sum_{t=0}^{T-1}
\left\|
\delta_t^\theta(X_{i,0:t})
-
\delta_{t-1}^\theta(X_{i,0:t-1})
\right\|_1.
\]
With $\delta_{-1}^\theta=0$, turnover is the mean absolute position change per path, trading date, and tradable asset, including the opening trade.

\noindent \emph{Aggregation across random seeds.}
Each metric is first computed on the test paths of one seed. For a scalar metric
having seed-level values \(R_1,\dots,R_B\), we report
\[
\frac1B\sum_{b=1}^B R_b
\ \pm \
\left(
\frac1B\sum_{b=1}^B
\left(R_b-\frac1B\sum_{c=1}^B R_c\right)^2
\right)^{1/2}.
\]

\subsection{Experiments}
\label{app:experiment-setup}
This section supplements Section~\ref{sec:numerical-experiments} with the complete comparisons and budget-sensitivity analysis.\footnote{Code and experiment configurations are available at \url{https://anonymous.4open.science/r/deep-hedging-nonstationary-dro/}.} We first present the stationary and nonstationary Heston experiments (Appendix~\ref{subsec:heston-results}), including their extensions with transaction costs and M--CVaR, followed by the GAD experiments (Appendix~\ref{subsec:market-results}). Appendix~\ref{app:additional-results} compares one-step and multistep \textsc{WRAP}.

We use the four training specifications and the uniform and drift-aware weighting schemes introduced in Section~\ref{sec:numerical-experiments}. Throughout, the probability-reweighting ambiguity set uses the Kullback--Leibler entropy function $\phi(x)=x\log x-x+1$, which satisfies the local regularity conditions of Theorems~\ref{thm:two_budget_ot_phi_taylor} and~\ref{thm:structured_phi_ot_first_order_attack_main}. 
The architecture, optimization settings, and ground-metric parameters are reported in Table~\ref{tab:experiment_hyperparameters}. 
Following \citet{goodfellow2015explaining}, we train with a mixture of clean and adversarial losses:
\begin{equation}
\label{eq:mixed-adversarial-objective}
\mathcal J_{\mathrm{mix}}^{\mathcal I}(\vartheta)
=
(1-\lambda_{\mathrm{adv}})
\sum_{n\in\mathcal I}\bar w_n\ell_{\mathrm{DH}}(\vartheta;X_n)
+\lambda_{\mathrm{adv}}\mathcal J_{\mathrm{adv}}^{\mathcal I}(\vartheta).
\end{equation}
Here \(\lambda_{\mathrm{adv}}\in[0,1]\) balances the two losses after the Empirical Risk Minimization (ERM) warmup reported in Table~\ref{tab:experiment_hyperparameters}. Hyperparameters, ambiguity budgets, and model selection are determined using validation performance only.
\begin{remark}
In Section~\ref{sec:short-two-budget}, we construct a first-order adversary using linear tilting of the form $a_n=1+\alpha_\tau(L_n-\bar L_w)$. While this construction is first-order accurate as $\tau\downarrow0$, at larger budgets the nonnegativity constraint can prevent further increases in $\alpha_\tau$, leaving part of the divergence budget unused. In the numerical experiments, we therefore replace linear tilting with exponential tilting, $a_n=\exp(sL_n)/\sum_j w_j\exp(sL_j)$, where $s\ge0$ is chosen by scalar bisection. The two constructions have the same small-budget leading behavior, preserving the first-order expansion while permitting the divergence budget to remain active beyond the range allowed by linear tilting.
\end{remark}

\begin{table}[ht]
    \centering
    \small
    \caption{Hyperparameter settings for the Heston and GAD experiments.}
    \label{tab:experiment_hyperparameters}
    \begin{tabular}{l c c}
        \toprule
        \textbf{Parameter} & \textbf{Heston} & \textbf{GAD} \\
        \midrule
        \multicolumn{3}{l}{\textit{Architecture and Optimization}} \\
        Base architecture & \multicolumn{2}{c}{Per-date MLP hedger} \\
        Hidden units & \multicolumn{2}{c}{$20$} \\
        Hidden layers per date & \multicolumn{2}{c}{$2$} \\
        Activation & \multicolumn{2}{c}{ReLU} \\
        Normalization & \multicolumn{2}{c}{Batch normalization} \\
        Optimizer & \multicolumn{2}{c}{Adam} \\
        Learning rate & \multicolumn{2}{c}{$5\times 10^{-3}$} \\
        Scheduler & \multicolumn{2}{c}{Cosine annealing, $\eta_{\min}=10^{-6}$} \\
        Batch size & \multicolumn{2}{c}{$10{,}000$} \\
        \midrule
        \multicolumn{3}{l}{\textit{Adversarial Training and Geometry}} \\
        ERM warmup epochs & $300$ & $100$ \\
        Total epochs & $600$ & $300$ \\
        Attack algorithms & \multicolumn{2}{c}{\textsc{WRAP}; Multistep \textsc{WRAP}} \\
        Clean/adversarial loss weights & \multicolumn{2}{c}{$1:1$} \\
        Ground metric exponent & \multicolumn{2}{c}{$p=2$} \\
        Ground metric scales & $(\lambda_S,\lambda_V)=(1,100)$ & $\lambda_S=1$ \\
         Ground metric time weights  & \multicolumn{2}{c}{$\kappa_t = 1$ for all $t\in\{0,\dots,T\}$} \\
        \bottomrule
    \end{tabular}
\end{table}

\subsubsection{Heston Experiments}
\label{subsec:heston-results}

We first consider the Heston model in \eqref{eq:heston-model-short}, which permits controlled comparisons between stationary and nonstationary market environments. The claim is the European call defined in Appendix~\ref{app:claims}. The Heston attacks perturb both components of the observed state, corresponding to the stronger SV-Attack in \citet{he2025distributional}.

For the stationary and nonstationary benchmarks, CVaR as defined in \eqref{eq:app-cvar-oce} serves as both the training objective and the model-selection criterion. We also report mean PnL and turnover, as defined in Appendix~\ref{subsec:performance-metrics}. CVaR is computed on the hedging loss $-\mathrm{PnL}$, so lower values indicate better tail-risk performance. Mean PnL is the average realized profit and loss, while lower turnover indicates less active rebalancing.

\paragraph{Stationary Heston experiments.}
\label{par:heston-stationary}

In the stationary benchmark, the training, validation, and test trajectories are generated independently under the same fixed Heston dynamics \eqref{eq:heston-model-short}. We split the data equally into training, validation, and test sets. This setting isolates the effects of adversarial probability reweighting and trajectory perturbations in the absence of temporal distribution shift. The stationary benchmark uses uniform reference weights.

We train all adversarial specifications with the one-step \textsc{WRAP} procedure in Algorithm~\ref{alg:drift-aware-one-step-attack}. For each method, training epochs and ambiguity budgets are selected by minimizing validation CVaR. The full budget grid is $\varepsilon,\tau\in\{0,0.05,0.10,0.15,0.20,0.25,0.30\}$.

\emph{Results.}~Table~\ref{tab:heston-stationary-optimal} shows that transport drives the improvement in the stationary benchmark. Validation selects \((\varepsilon,\tau)=(0.10,0)\), so \textsc{WRAP-U} coincides with \textsc{OT-U} and reduces test CVaR by \(2.4\%\) relative to \textsc{ERM-U}. Reweighting alone provides no improvement at the displayed precision. Mean PnL is nearly unchanged across methods, while turnover remains lowest under \textsc{ERM-U}. Thus, without temporal distribution shift, validation favors robustness to path perturbations but provides no evidence that additional adversarial reweighting improves performance.

\begin{table*}[!ht]
\centering
\scriptsize
\begin{threeparttable}
\caption{Results for the stationary Heston model with the CVaR objective.}
\label{tab:heston-stationary-optimal}
\setlength{\tabcolsep}{3.5pt}
\begin{tabular}{@{}lcccccc@{}}
\toprule
Method & $\varepsilon$ & $\tau$
& \shortstack{Val. CVaR $(\downarrow)$}
& \shortstack{Test CVaR $(\downarrow)$}
& \shortstack{Mean PnL$(\uparrow)$}
& \shortstack{Turnover $(\downarrow)$} \\
\midrule
\textsc{ERM-U} & -- & -- & $3.634 \pm 0.013$ & $3.629 \pm 0.018$ & $-2.199 \pm 0.005$ & $\underline{3.084} \pm 0.124$ \\
$\phi$\textsc{-U} & -- & $0.25$ & $3.627 \pm 0.014$ & $3.629 \pm 0.025$ & $\underline{-2.198} \pm 0.004$ & $3.283 \pm 0.266$ \\
\textsc{OT-U} & $0.10$ & -- & $\underline{3.550} \pm 0.011$ & $\underline{3.543} \pm 0.016$ & $-2.199 \pm 0.004$ & $3.156 \pm 0.064$ \\
\textsc{WRAP-U} & $0.10$ & -- & $\underline{3.550} \pm 0.011$ & $\underline{3.543} \pm 0.016$ & $-2.199 \pm 0.004$ & $3.156 \pm 0.064$ \\
\bottomrule
\end{tabular}
\begin{tablenotes}[flushleft]
\scriptsize
\item \textit{Note.} Values are mean $\pm$ population standard deviation over five seeds.
\end{tablenotes}
\end{threeparttable}
\end{table*}

\paragraph{Stationary Heston with transaction costs and M--CVaR.}
\label{par:heston-transaction-costs}
We next extend the stationary Heston benchmark by introducing proportional transaction costs at rate \(c\ge0\), which penalize rebalancing in the training objective. Setting \(\delta_{-1}^\theta=0\), the cumulative cost of opening and rebalancing the hedge is
\begin{equation}
\label{eq:transaction-cost}
\operatorname{TC_c}(\theta,X)
=c\sum_{t=0}^{T-1}|S_t|^\top
\bigl|\delta_t^\theta(X_{0:t})-\delta_{t-1}^\theta(X_{0:t-1})\bigr|.
\end{equation}
The resulting terminal PnL is
\begin{equation}
\label{eq:pnl-transaction-cost}
\mathrm{PnL}_{c}(\theta,X)
=p_0 +\sum_{t=0}^{T-1}\delta_t^\theta(X_{0:t})^\top(S_{t+1}-S_t)
-\Phi(X)-\operatorname{TC_c}(\theta,X).
\end{equation}

We retain the fixed Heston dynamics and European call from the stationary benchmark. We consider both CVaR and M--CVaR objectives, defined in \eqref{eq:app-cvar-oce} and~\eqref{eq:oce-mean-cvar}, with transaction-cost rate \(c=0.005\), confidence level \(\alpha=0.95\), and \(\omega=0.5\) for M--CVaR. For each objective, we compare \textsc{ERM-U}, \(\phi\)\textsc{-U}, \textsc{OT-U}, and \textsc{WRAP-U} under uniform reference weights. Training epochs and ambiguity budgets are selected separately for each method using validation performance.

\emph{Results.} Table~\ref{tab:heston_stationary_transaction_cost_optimal} shows that \textsc{WRAP-U} selects \((\varepsilon,\tau)=(0.05,0.05)\) and attains the lowest validation and test CVaR. Relative to \textsc{ERM-U}, test CVaR decreases from \(4.494\) to \(4.481\) (\(0.3\%\)), at the cost of higher transaction costs and lower mean PnL. \textsc{OT-U} attains the highest mean PnL and the lowest transaction costs. The additional reduction from joint training is marginal, with test CVaR of \(4.481\) under \textsc{WRAP-U} versus \(4.484\) under \textsc{OT-U}.

\begin{table*}[!ht]
\centering
\scriptsize
\begin{threeparttable}
\caption{Results for the stationary Heston model with transaction costs and the CVaR objective.}
\label{tab:heston_stationary_transaction_cost_optimal}
\setlength{\tabcolsep}{3.5pt}
\begin{tabular}{@{}lcccccc@{}}
\toprule
Method & $\varepsilon$ & $\tau$ & \shortstack{Val. CVaR $(\downarrow)$} & \shortstack{Test CVaR $(\downarrow)$} & \shortstack{Mean PnL$(\uparrow)$} & \shortstack{Total TC$(\downarrow)$} \\
\midrule
\textsc{ERM-U} & -- & -- & $4.500 \pm 0.011$ & $4.494 \pm 0.016$ & $-3.003 \pm 0.006$ & $0.803 \pm 0.004$ \\
$\phi$\textsc{-U} & -- & $0.05$ & $4.491 \pm 0.009$ & $4.485 \pm 0.015$ & $-3.019 \pm 0.004$ & $0.819 \pm 0.004$ \\
\textsc{OT-U} & $0.05$ & -- & $4.488 \pm 0.010$ & $4.484 \pm 0.013$ & $\underline{-3.002} \pm 0.010$ & $\underline{0.802} \pm 0.008$ \\
\textsc{WRAP-U} & $0.05$ & $0.05$ & $\underline{4.487} \pm 0.009$ & $\underline{4.481} \pm 0.012$ & $-3.014 \pm 0.006$ & $0.815 \pm 0.007$ \\
\bottomrule
\end{tabular}
\begin{tablenotes}[flushleft]
\scriptsize
\item \textit{Note.} Values are mean $\pm$ population standard deviation over five seeds.
\end{tablenotes}
\end{threeparttable}
\end{table*}

Figure~\ref{fig:heston-stationary-budget-sensitivity} clarifies the role of the two ambiguity channels. Without transaction costs, transport accounts for nearly all of the improvement, and validation selects \((\varepsilon,\tau)=(0.10,0)\). With transaction costs, gains are smaller and concentrated at low budgets, while larger transport budgets deteriorate validation performance. This is consistent with the rebalancing penalty in \eqref{eq:transaction-cost} already regularizing the hedge, as also suggested by the smoother \textsc{ERM-U} policy in Figure~\ref{fig:heston-stationary-tc-policy}. In contrast, reweighting contributes more substantially under nonstationarity (Figure~\ref{fig:heston-nonstationary-budget-sensitivity}). Thus, the two-budget formulation adapts across settings: validation selects transport alone when reweighting adds no value and positive budgets for both channels when their combination improves the objective.

\begin{figure}[htbp]
\centering
\begin{subfigure}[b]{0.49\textwidth}
    \centering
    \includegraphics[width=\linewidth]{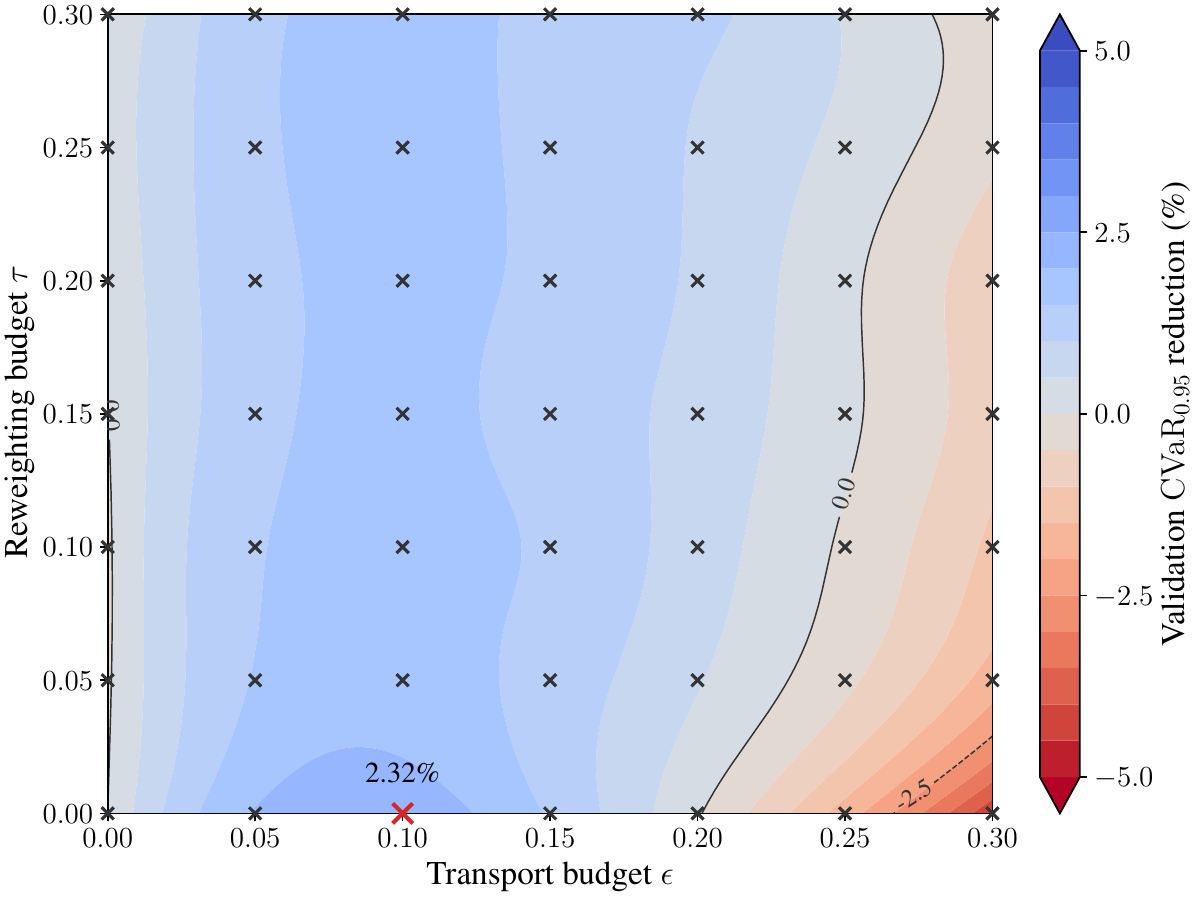}
    \caption{Without transaction costs}
    \label{fig:heston-stationary-cvar-reduction}
\end{subfigure}
\hfill
\begin{subfigure}[b]{0.49\textwidth}
    \centering
    \includegraphics[width=\linewidth]{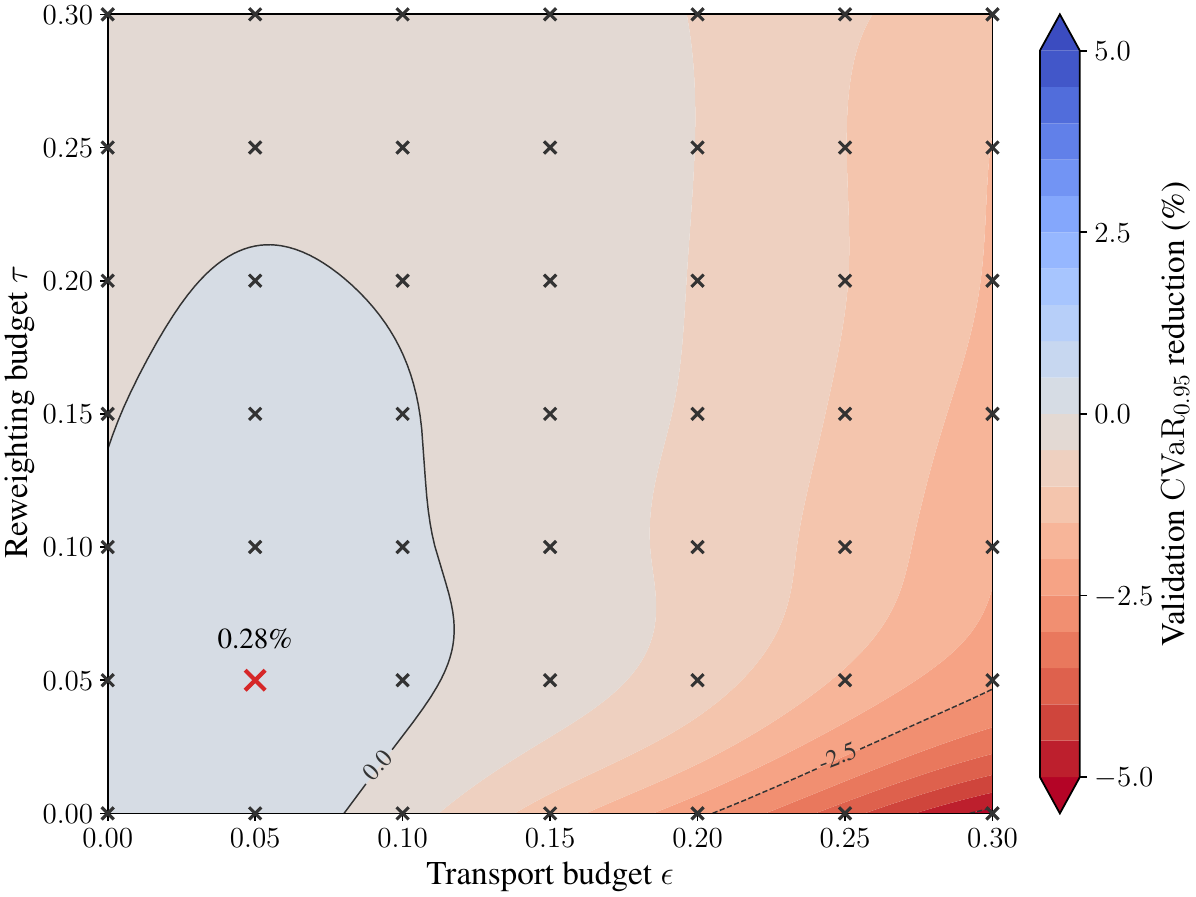}
    \caption{With transaction costs ($c=0.005$)}
    \label{fig:heston-stationary-tc-cvar-reduction}
\end{subfigure}
\caption{Stationary Heston validation grids with uniform reference weights and the CVaR objective, corresponding to Tables~\ref{tab:heston-stationary-optimal} and~\ref{tab:heston_stationary_transaction_cost_optimal}. The surface reports the percentage reduction in validation CVaR relative to the corresponding \textsc{ERM-U} baseline. Crosses indicate the evaluated budget pairs, with contours interpolating between them; the selected \textsc{WRAP} budgets are \((\varepsilon,\tau)=(0.10,0)\) in (a) and \((0.05,0.05)\) in (b).}
\label{fig:heston-stationary-budget-sensitivity}
\end{figure}

Under M--CVaR (Table~\ref{tab:heston_stationary_transaction_cost_mean_cvar_optimal}), validation again selects transport alone, with \textsc{WRAP-U} coinciding with \textsc{OT-U} at \((\varepsilon,\tau)=(0.05,0)\). Both attain the lowest validation and test M--CVaR, reducing the test objective from \(3.714\) to \(3.706\), while also achieving the highest mean PnL and lowest transaction costs. Reweighting alone attains the lowest test CVaR, but its lower mean PnL yields a higher M--CVaR objective. Hence, the preferred adversarial mechanism depends on the optimization objective, not tail risk alone.

\begin{table*}[!ht]
\centering
\scriptsize
\begin{threeparttable}
\caption{Results for the stationary Heston model with transaction costs and the M--CVaR objective.}
\label{tab:heston_stationary_transaction_cost_mean_cvar_optimal}
\setlength{\tabcolsep}{3.5pt}
\begin{tabular}{@{}lccccccc@{}}
\toprule
Method & $\varepsilon$ & $\tau$ & \shortstack{Val. M--CVaR $(\downarrow)$} & \shortstack{Test M--CVaR $(\downarrow)$} & \shortstack{Test CVaR $(\downarrow)$} & \shortstack{Mean PnL$(\uparrow)$} & \shortstack{Total TC$(\downarrow)$} \\
\midrule
\textsc{ERM-U} & -- & -- & $3.717 \pm 0.010$ & $3.714 \pm 0.008$ & $4.500 \pm 0.014$ & $-2.927 \pm 0.004$ & $0.727 \pm 0.003$ \\
$\phi$\textsc{-U} & -- & $0.05$ & $3.716 \pm 0.007$ & $3.713 \pm 0.009$ & $\underline{4.486} \pm 0.014$ & $-2.940 \pm 0.004$ & $0.740 \pm 0.002$ \\
\textsc{OT-U} & $0.05$ & -- & $\underline{3.708} \pm 0.009$ & $\underline{3.706} \pm 0.007$ & $4.502 \pm 0.013$ & $\underline{-2.909} \pm 0.004$ & $\underline{0.709} \pm 0.003$ \\
\textsc{WRAP-U} & $0.05$ & -- & $\underline{3.708} \pm 0.009$ & $\underline{3.706} \pm 0.007$ & $4.502 \pm 0.013$ & $\underline{-2.909} \pm 0.004$ & $\underline{0.709} \pm 0.003$ \\
\bottomrule
\end{tabular}
\begin{tablenotes}[flushleft]
\scriptsize
\item \textit{Note.} Values are mean $\pm$ population standard deviation over five seeds.
\end{tablenotes}
\end{threeparttable}
\end{table*}

\paragraph{Nonstationary Heston experiments.}
\label{par:heston-nonstationary}

In the nonstationary experiments, the stock-price drift and long-run variance evolve chronologically, so the validation, near-future, and far-future test distributions depart progressively from the training regime. This setting allows us to examine both the drift-aware reference weights from Section~\ref{subsec:short-coverage} and the two adversarial ambiguity channels under temporal distribution shift.

We generate $B=8$ chronological blocks of 4,000 independent trajectories each. For block $b=1,\ldots,B$, the drift and long-run variance are
\[
\mu_b
=
\mu_0+\Delta\mu\,\frac{b-1}{B-1},
\qquad
\theta_{\mathrm H,b}
=
\theta_{\mathrm H,0}
+\Delta\theta_{\mathrm H}\,\frac{b-1}{B-1},
\]
with $\mu_0=0$, $\Delta\mu=0.20$, $\theta_{\mathrm H,0}=0.04$, and $\Delta\theta_{\mathrm H}=0.12$. Each block starts at $v_{0,b}=\theta_{\mathrm H,b}$, so both initial and long-run variance increase from $0.04$ to $0.16$, while the drift increases from $0$ to $0.20$. All remaining Heston parameters retain their baseline values and are fixed across blocks.

The blocks are split chronologically into training (Blocks \(1,\dots,5\)), validation (Block 6), near-future test (Block 7), and far-future test (Block 8). Weight selection precedes adversarial-budget selection. For the drift-aware methods, we train clean hedges over a grid of eleven candidate ratios \(\varrho/\varepsilon_\varrho\) and select the ratio minimizing mean validation CVaR over five seeds. The selected ratio \(0.00853\) is then held fixed across all weighted adversarial specifications. Here \(\varepsilon_\varrho\) denotes the reference-weighting radius from Section~\ref{subsec:short-coverage}, distinct from the adversarial transport budget \(\varepsilon\). The resulting weights are assigned at the block level and distributed uniformly across paths within each block, while \(\varrho/\varepsilon_\varrho=0\) recovers uniform reference weights.

The nonstationary one-step budget grid is
\begin{equation}
\label{eq:heston-nonstationary-grid}
\varepsilon\in\{0,0.1,\ldots,0.5\},
\qquad
\tau\in\{0,0.2,\ldots,1.2\}.
\end{equation}

\emph{Results.}~Table~\ref{tab:short-nonstationary} shows that temporal distribution shift materially degrades the clean hedge: the CVaR of \textsc{ERM-U} increases from \(6.868\) on validation to \(7.952\) and \(8.992\) on the near- and far-future samples. In contrast to the stationary benchmark, both ambiguity channels now improve performance under uniform reference weights, and their combination yields the largest reduction. \textsc{WRAP-U} selects \((\varepsilon,\tau)=(0.30,0.60)\) and lowers far-future CVaR by \(19.2\%\) relative to \textsc{ERM-U}. This contrast with the stationary setting is central: adversarial reweighting becomes useful once the training distribution no longer represents future market conditions, complementing the protection provided by transport.

Drift-aware weighting provides an additional layer of robustness by improving the reference distribution before adversarial perturbations are applied. Across all four training methods, drift-aware weights lower CVaR in every evaluation period, with the largest relative far-future improvement under clean training. \textsc{WRAP-W} selects \((\varepsilon,\tau)=(0.40,1.00)\) and attains the lowest CVaR in validation and both test periods. The additional gain over \textsc{WRAP-U} is smaller, consistent with drift-aware weighting and adversarial training addressing partly overlapping sources of temporal misspecification. The improvement is concentrated in tail risk rather than average performance: \textsc{ERM-U} retains the highest mean PnL, while the robust methods achieve substantially lower CVaR.

Figure~\ref{fig:heston-nonstationary-budget-sensitivity} provides further evidence for the complementarity of the two ambiguity channels. In both panels, the largest validation gains occur in the interior of the budget grid, where both \(\varepsilon\) and \(\tau\) are positive. The selected joint budgets reduce validation CVaR by \(12.6\%\) under uniform weights and \(10.5\%\) under drift-aware weights, compared with at most \(9.1\%\) and \(6.8\%\), respectively, along the single-budget axes. Moreover, sufficiently large transport budgets can worsen performance when applied alone, whereas combining transport with reweighting produces a broader region of improved validation CVaR. The benefit of the two-budget formulation therefore lies not only in attaining a better selected solution, but also in providing a more robust region of effective adversarial regularization.

The smaller percentage reduction under drift-aware weighting reflects the stronger starting point: weighting alone lowers validation CVaR from \(6.868\) to \(6.653\). Panel~(b) therefore measures adversarial gains relative to an already improved reference distribution. Despite the smaller relative reduction, \textsc{WRAP-W} attains lower validation CVaR than \textsc{WRAP-U} (\(5.954\) versus \(6.005\)). This pattern shows that drift-aware reference construction and adversarial robustification are complementary rather than substitutes: the former reduces temporal misspecification in the empirical reference, while the latter protects against residual probability and path uncertainty.

\begin{figure}[htbp]
\centering
\begin{subfigure}[b]{0.49\textwidth}
    \centering
    \includegraphics[width=\linewidth]{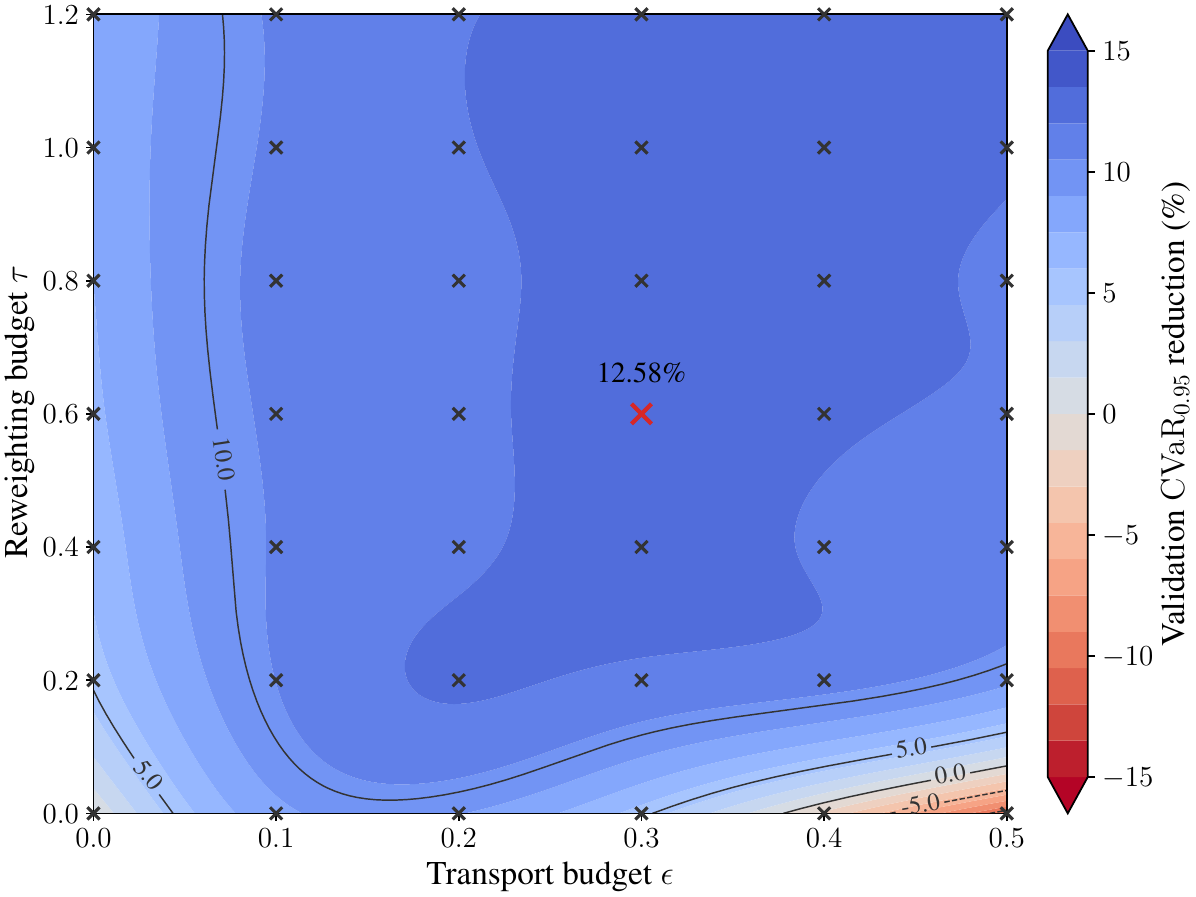}
    \caption{Uniform reference weights (\textsc{-U})}
    \label{fig:heston-nonstationary-uniform-cvar-reduction}
\end{subfigure}
\hfill
\begin{subfigure}[b]{0.49\textwidth}
    \centering
    \includegraphics[width=\linewidth]{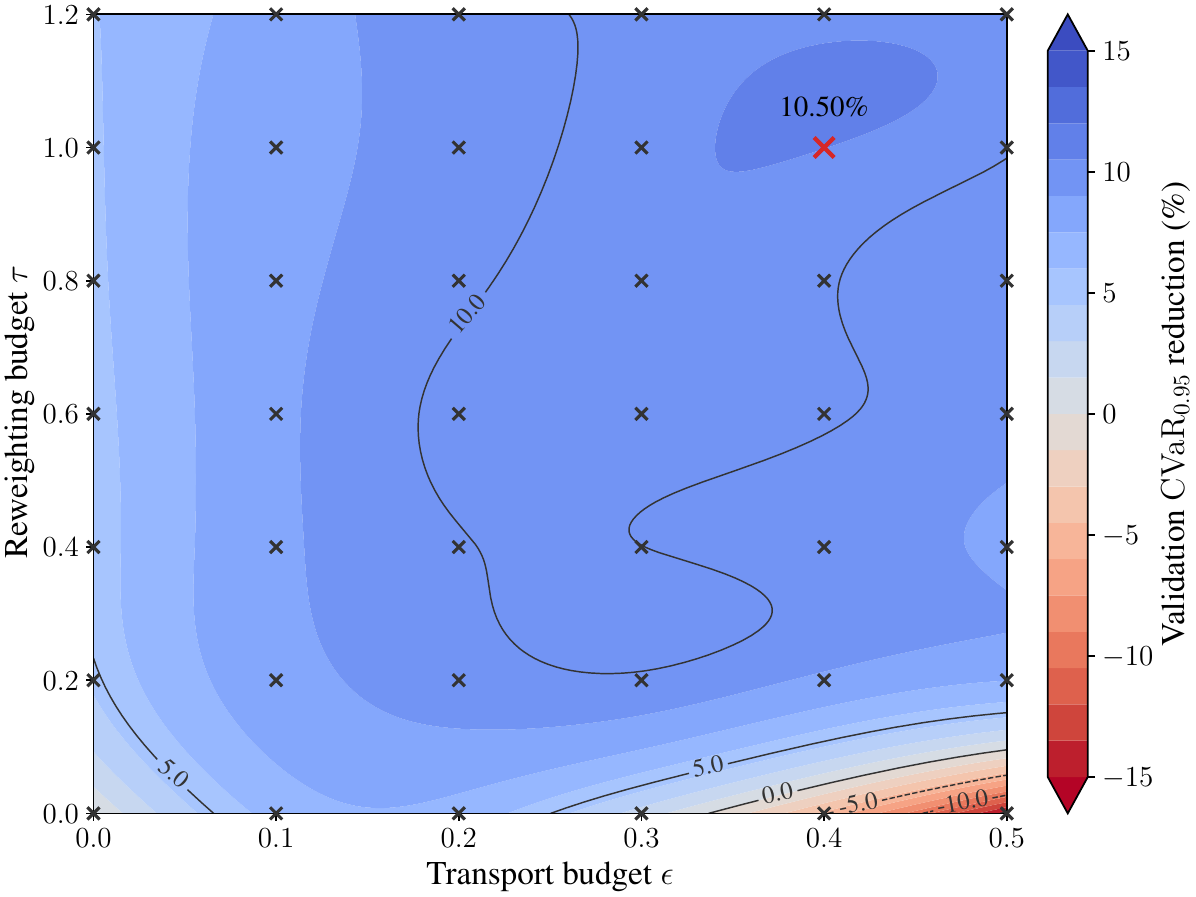}
    \caption{Drift-aware reference weights (\textsc{-W})}
    \label{fig:heston-nonstationary-weighted-cvar-reduction}
\end{subfigure}
\caption{Nonstationary Heston validation grids for Table~\ref{tab:short-nonstationary}. Colours show percentage reductions Val.~CVaR relative to \textsc{ERM-U} in (a) and \textsc{ERM-W} in (b); Black crosses mark the budget pairs in \eqref{eq:heston-nonstationary-grid}; contours interpolate between them. Red crosses mark the selected \textsc{WRAP} budgets: $(\varepsilon,\tau)=(0.30,0.60)$ in (a) and $(0.40,1.00)$ in (b)}
\label{fig:heston-nonstationary-budget-sensitivity}
\end{figure}

\paragraph{Nonstationary Heston with transaction costs and M--CVaR.}
\label{par:heston-nonstationary-transaction-costs}
We extend the nonstationary Heston experiment to proportional transaction costs using the European call in Appendix~\ref{app:claims} and the chronological block split described above. Costs follow \eqref{eq:transaction-cost} at rate $c=0.005$ and enter the terminal PnL through \eqref{eq:pnl-transaction-cost}. Reported mean PnL is net of these costs, and ``Total TC'' denotes the mean cumulative transaction cost per path, including opening trades and excluding terminal liquidation.

We use uniform reference weights and compare \textsc{ERM-U}, \(\phi\)\textsc{-U}, \textsc{OT-U}, and \textsc{WRAP-U} with the one-step attack in Algorithm~\ref{alg:drift-aware-one-step-attack}. The first experiment uses the CVaR objective in \eqref{eq:app-cvar-oce} with \(\alpha=0.95\), while the second uses the M--CVaR objective in \eqref{eq:oce-mean-cvar} with the same confidence level and \(\omega=1/2\). Epochs and ambiguity budgets are selected on the validation block and subsequently evaluated on the near- and far-future test blocks.

\emph{Results.}~Table~\ref{tab:heston_ns_u_tc_optimal} shows that transaction costs alter the relative contribution of the two ambiguity channels. Reweighting alone substantially improves CVaR across all three periods, whereas transport alone worsens the objective despite producing the lowest transaction costs. Joint training combines the tail-risk benefit of reweighting with a smaller transport component and attains the lowest validation, near-future, and far-future CVaR. In the far-future sample, \textsc{WRAP-U} reduces CVaR by \(6.6\%\) relative to \textsc{ERM-U}. The improvement comes with lower mean PnL and higher trading costs than clean training, but \textsc{WRAP-U} improves both quantities relative to reweighting alone.

\begin{table*}[!ht]
\centering
\scriptsize
\begin{threeparttable}
\caption{Results for the nonstationary Heston model with transaction costs and the CVaR objective.}
\label{tab:heston_ns_u_tc_optimal}
\setlength{\tabcolsep}{3pt}
\begin{tabular}{@{}lccccccc@{}}
\toprule
Method & $\varepsilon$ & $\tau$ & \shortstack{Val. CVaR$(\downarrow)$} & \shortstack{Near CVaR$(\downarrow)$} & \shortstack{Far CVaR$(\downarrow)$} & \shortstack{Far Mean PnL$(\uparrow)$} & \shortstack{Far Total TC$(\downarrow)$} \\
\midrule
\textsc{ERM-U} & -- & -- & $7.592 \pm 0.077$ & $8.250 \pm 0.155$ & $8.935 \pm 0.260$ & $\underline{-5.562} \pm 0.078$ & $1.002 \pm 0.073$ \\
$\phi$\textsc{-U} & -- & $1.20$ & $7.122 \pm 0.042$ & $7.778 \pm 0.131$ & $8.434 \pm 0.225$ & $-5.714 \pm 0.082$ & $1.163 \pm 0.076$ \\
\textsc{OT-U} & $0.10$ & -- & $7.752 \pm 0.073$ & $8.456 \pm 0.134$ & $9.170 \pm 0.195$ & $-5.565 \pm 0.101$ & $\underline{0.976} \pm 0.095$ \\
\textsc{WRAP-U} & $0.10$ & $0.80$ & $\underline{7.082} \pm 0.041$ & $\underline{7.691} \pm 0.111$ & $\underline{8.349} \pm 0.189$ & $-5.695 \pm 0.071$ & $1.130 \pm 0.065$ \\
\bottomrule
\end{tabular}
\begin{tablenotes}[flushleft]
\scriptsize
\item \textit{Note.} Values are mean $\pm$ population standard deviation over five seeds.
\end{tablenotes}
\end{threeparttable}
\end{table*}

Under M--CVaR, the same qualitative pattern persists (Table~\ref{tab:heston_ns_u_tc_mc_optimal}). Reweighting remains beneficial, transport alone again worsens the risk objective, and \textsc{WRAP-U} attains the lowest validation, near-future, and far-future M--CVaR. The joint method also improves mean PnL and transaction costs relative to reweighting alone, yielding a better balance between average loss and tail risk. Most importantly, the nonstationary experiment selects positive budgets for both perturbation channels under both objectives. This contrasts with the stationary M--CVaR benchmark, where validation selects transport alone, and reinforces the role of reweighting when the market distribution evolves over time.

\begin{table*}[!ht]
\centering
\scriptsize
\begin{threeparttable}
\caption{Results for the nonstationary Heston model with transaction costs and the M--CVaR objective.}
\label{tab:heston_ns_u_tc_mc_optimal}
\setlength{\tabcolsep}{3pt}
\begin{tabular}{@{}lccccccc@{}}
\toprule
Method & $\varepsilon$ & $\tau$ & \shortstack{Val. M--CVaR$(\downarrow)$} & \shortstack{Near M--CVaR$(\downarrow)$} & \shortstack{Far M--CVaR$(\downarrow)$} & \shortstack{Far Mean PnL$(\uparrow)$} & \shortstack{Far Total TC$(\downarrow)$} \\
\midrule
\textsc{ERM-U} & -- & -- & $6.281 \pm 0.048$ & $6.766 \pm 0.077$ & $7.251 \pm 0.149$ & $\underline{-5.552} \pm 0.084$ & $0.994 \pm 0.077$ \\
$\phi$\textsc{-U} & -- & $1.00$ & $6.078 \pm 0.011$ & $6.581 \pm 0.068$ & $7.090 \pm 0.149$ & $-5.721 \pm 0.089$ & $1.170 \pm 0.079$ \\
\textsc{OT-U} & $0.10$ & -- & $6.388 \pm 0.074$ & $6.898 \pm 0.078$ & $7.401 \pm 0.151$ & $-5.562 \pm 0.122$ & $\underline{0.974} \pm 0.112$ \\
\textsc{WRAP-U} & $0.20$ & $0.80$ & $\underline{6.045} \pm 0.013$ & $\underline{6.523} \pm 0.062$ & $\underline{7.019} \pm 0.148$ & $-5.681 \pm 0.098$ & $1.112 \pm 0.092$ \\
\bottomrule
\end{tabular}
\begin{tablenotes}[flushleft]
\scriptsize
\item \textit{Note.} Values are mean $\pm$ population standard deviation over five seeds.
\end{tablenotes}
\end{threeparttable}
\end{table*}

\paragraph{Hedging surfaces.}
Figure~\ref{fig:heston-stationary-policy} shows the stock position \(\delta_t^\theta\) as a function of moneyness and trading date. Across methods, the learned hedge has the expected structure: the stock position is close to zero when the call is far out of the money, approaches one when the option is well in the money, and transitions more sharply around the strike as maturity approaches. The main difference across training schemes is therefore not the overall shape of the hedge, but its local regularity. In particular, \textsc{OT-U} produces a smoother policy than \textsc{ERM-U}, with fewer local fluctuations across nearby moneyness levels and trading dates. This coincides with the regularization effect of Wasserstein described by \citet{gao2022Wasserstein}.

Transaction costs reinforce this regularization effect. Relative to Figure~\ref{fig:heston-stationary-policy}, the policies in Figure~\ref{fig:heston-stationary-tc-policy} vary more gradually across states and time, including under \textsc{ERM-U}. The transaction-cost term in \eqref{eq:transaction-cost} directly penalizes changes in position, whereas OT regularizes the hedge indirectly through sensitivity of the loss to path perturbations, as characterized by Theorem~\ref{thm:two_budget_ot_phi_taylor}. The smoother \textsc{OT-U} and \textsc{WRAP-U} surfaces therefore reflect a distinct robustness mechanism that remains visible even when rebalancing is already penalized.

Under nonstationarity, Figures~\ref{fig:heston-nonstationary-uniform-policy} and~\ref{fig:heston-nonstationary-policy} show the corresponding hedges under uniform and drift-aware reference weights. The later Heston blocks generate a wider range of moneyness states, but the qualitative structure of the hedge remains stable. The main effect of robust training is again local: \(\phi\)\textsc{-W} and \textsc{OT-W} reduce irregularities around the strike relative to \textsc{ERM-W}, while \textsc{WRAP-W} produces the smoothest transition among the weighted specifications. The same pattern appears under uniform weights, where \textsc{WRAP-U} also reduces local fluctuations relative to \textsc{ERM-U}. These surfaces suggest that the robustness gains documented above are accompanied by more regular hedging policies rather than a qualitatively different trading rule.

\begin{figure}[H]
\centering
\includegraphics[width=\textwidth]{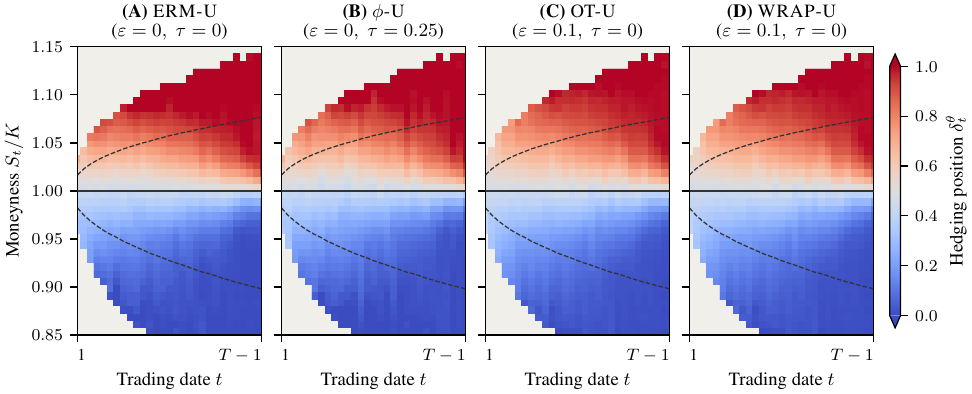}
\caption{Stock-hedging policies for stationary Heston with uniform reference weights and the CVaR objective, using the selections in Table~\ref{tab:heston-stationary-optimal}: (A) \textsc{ERM-U}, (B) $\phi$\textsc{-U}, (C) \textsc{OT-U}, and (D) \textsc{WRAP-U}. Colours show stock holdings averaged over validation paths and five seeds within each date--moneyness bin. Dashed curves mark the 5th and 95th moneyness percentiles; the solid line marks the strike.}
\label{fig:heston-stationary-policy}
\end{figure}

\begin{figure}[H]
\centering
\includegraphics[width=\textwidth]{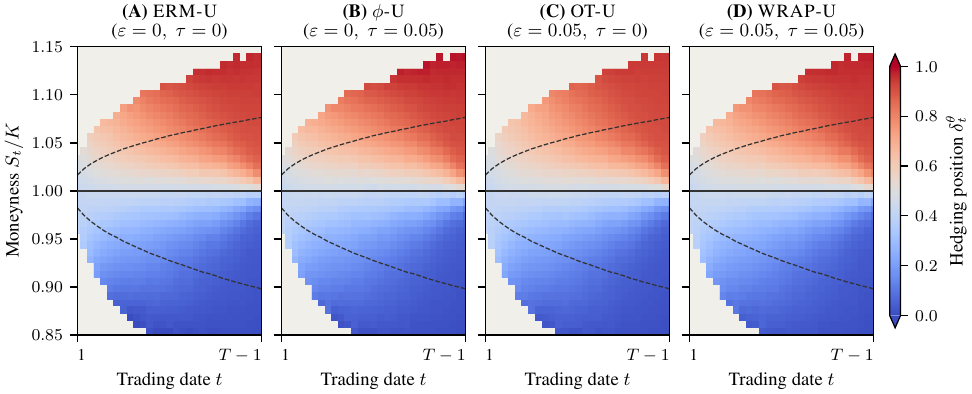}
\caption{Stock-hedging policies for stationary Heston with uniform reference weights, the CVaR objective, and proportional transaction costs $c=0.005$, using the selections in Table~\ref{tab:heston_stationary_transaction_cost_optimal}. Panels, averaging, and plotting conventions follow Figure~\ref{fig:heston-stationary-policy}, with the same colour scale.}
\label{fig:heston-stationary-tc-policy}
\end{figure}

\begin{figure}[H]
\centering
\includegraphics[width=\textwidth]{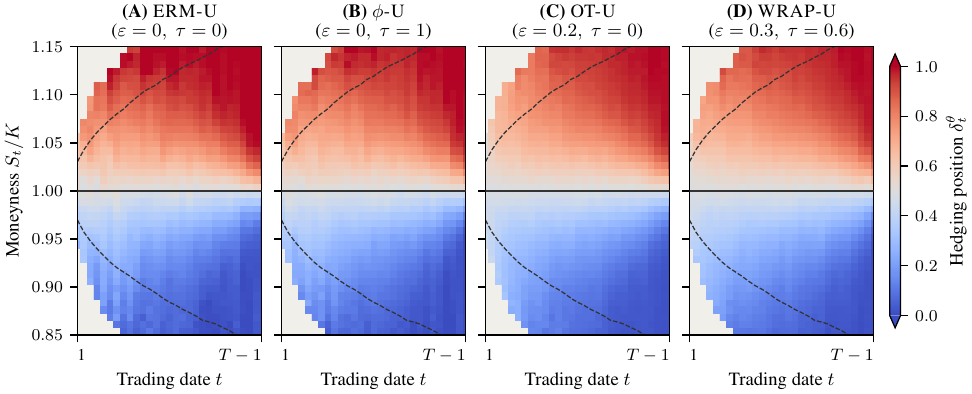}
\caption{Stock-hedging policies for nonstationary Heston with uniform reference weights and the CVaR objective, using the selections in Table~\ref{tab:short-nonstationary}: (A) \textsc{ERM-U}, (B) $\phi$\textsc{-U}, (C) \textsc{OT-U}, and (D) \textsc{WRAP-U}. Averaging and plotting conventions follow Figure~\ref{fig:heston-stationary-policy}, with the same colour scale.}
\label{fig:heston-nonstationary-uniform-policy}
\end{figure}

\begin{figure}[H]
\centering
\includegraphics[width=\textwidth]{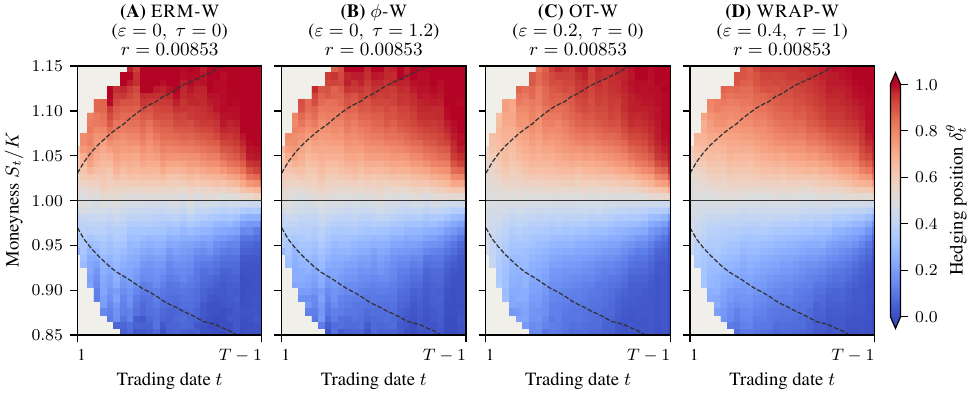}
\caption{Stock-hedging policies for nonstationary Heston with drift-aware reference weights and the CVaR objective, using the selections in Table~\ref{tab:short-nonstationary}: (A) \textsc{ERM-W}, (B) $\phi$\textsc{-W}, (C) \textsc{OT-W}, and (D) \textsc{WRAP-W}.  The displayed ratio $r=\varrho/\varepsilon_\varrho\approx0.00853$ controls reference weighting. Averaging and plotting conventions follow Figure~\ref{fig:heston-stationary-policy}, with the same colour scale.}
\label{fig:heston-nonstationary-policy}
\end{figure}

\subsubsection{GAD Experiments}
\label{subsec:market-results}

We complement the controlled Heston experiments with generalized affine diffusion (GAD) models calibrated to historical equity prices, using the dynamics introduced in Appendix~\ref{app:financial-time-series-models}. Following the market-data benchmark of \citet[Section~5.4]{he2025distributional} and the parameter-uncertainty approach of \citet{lutkebohmert2022robust}, we compare training under a fixed calibration with training that samples parameters from rolling estimates. Evaluation on simulated paths from models fitted to later, held-out periods assesses whether adversarial probability reweighting and path perturbations improve hedging performance under changes in market conditions estimated from real prices.

\paragraph{Data and descriptive statistics.}
\label{appsub:descriptive-statistics}
We use daily auto-adjusted closing prices from Yahoo Finance for AAPL, MSFT, AMZN, GOOGL, and BRK-B. Training calibration uses prices from 26 September 2008 through 31 December 2019, while validation and test models are calibrated separately to 2020--2023 and 2024--2025 prices. Each period is divided into overlapping windows of 250 prices, shifted by 100 observations, with no window crossing a period boundary. Each window is rescaled to start at \(S_0=10\). Calibration settings are selected using the first 200 prices for fitting and the following 50 returns for validation, after which the coefficients are refitted on the full window. Test-period calibration is performed only after epoch and budget selection has been frozen. The resulting experiment is therefore a retrospective simulation benchmark conditional on period-specific fitted market models.

\begin{table}[ht]
\centering
 \small
\caption{Descriptive statistics of daily simple adjusted-close returns across the training, validation, and test periods. Mean, median, volatility, and maximum drawdown are reported in percentages. The \(p_{\mathrm{KS}}\) column reports two-sample Kolmogorov--Smirnov \(p\)-values relative to the corresponding training-period returns.}
% \resizebox{\textwidth}{!}{%
\label{tab:gad_market_descriptive_stats}
\begin{tabular}{lccccccc}
\toprule
Stock & Split & $\mu$ (\%) & Median (\%) & $\sigma$ (\%) & MaxDD $(\%)$ & AC(1) & $p_\text{KS}$ \\
\midrule
AAPL & Train & $30.24$ & $25.39$ & $29.33$ & $-43.80$ & $-0.000$ &  \\
 & Val & $29.90$ & $24.90$ & $33.55$ & $-31.43$ & $-0.133$ & $0.016$ \\
 & Test & $23.51$ & $29.92$ & $27.87$ & $-33.36$ & $0.045$ & $0.046$ \\
\addlinespace
MSFT & Train & $21.56$ & $16.11$ & $26.86$ & $-43.95$ & $-0.084$ &  \\
 & Val & $27.58$ & $24.65$ & $32.60$ & $-37.15$ & $-0.197$ & $<0.001$ \\
 & Test & $16.54$ & $28.00$ & $22.16$ & $-23.73$ & $-0.018$ & $0.682$ \\
\addlinespace
AMZN & Train & $35.49$ & $23.10$ & $36.18$ & $-51.86$ & $-0.018$ &  \\
 & Val & $18.90$ & $19.69$ & $37.67$ & $-56.15$ & $-0.039$ & $0.007$ \\
 & Test & $26.62$ & $16.34$ & $31.37$ & $-30.88$ & $-0.030$ & $0.918$ \\
\addlinespace
GOOGL & Train & $19.92$ & $16.17$ & $27.21$ & $-40.27$ & $-0.010$ &  \\
 & Val & $23.52$ & $26.61$ & $33.52$ & $-44.32$ & $-0.108$ & $<0.001$ \\
 & Test & $46.09$ & $73.66$ & $30.23$ & $-29.81$ & $-0.014$ & $<0.001$ \\
\addlinespace
BRK-B & Train & $10.84$ & $5.87$ & $22.36$ & $-50.54$ & $-0.032$ &  \\
 & Val & $13.87$ & $15.01$ & $23.19$ & $-29.57$ & $-0.179$ & $0.201$ \\
 & Test & $17.85$ & $21.29$ & $16.71$ & $-14.95$ & $0.016$ & $0.256$ \\
\bottomrule
\end{tabular}
\begin{tablenotes}[flushleft]
\small
\item \textit{Note.} Daily mean and median returns are annualized by a factor of \(252\), and volatility by \(\sqrt{252}\). MaxDD reports the maximum decline from a previous peak, and AC(1) the lag-one autocorrelation. KS \(p\)-values are nominal and are not adjusted for serial dependence or multiple testing.
\end{tablenotes}
\end{table}

Table~\ref{tab:gad_market_descriptive_stats} reports annualized mean and median returns, annualized volatility, maximum drawdown, lag-one autocorrelation, and two-sample Kolmogorov--Smirnov \(p\)-values relative to the training period. The three calibration periods exhibit different market characteristics. In particular, test-period volatility is lower than validation-period volatility for all five stocks, with the largest declines for MSFT, AMZN, and BRK-B. Maximum drawdowns are also generally smaller in the test period. Hence, the test period represents a different market regime, but not simply a more severe version of the validation period.

The KS statistics provide additional evidence that distributional changes differ across stocks. AAPL and GOOGL exhibit small test-period \(p\)-values relative to the training sample, while the corresponding values for MSFT, AMZN, and BRK-B are larger. The market-data experiment therefore covers heterogeneous changes in return distributions rather than a single common stress scenario. This provides a useful setting for testing whether adversarial training remains beneficial across different forms of market change.

\paragraph{Calibration.}
\label{par:gad-calibration}
The full five-parameter GAD model in \eqref{eq:gad-model-short} can be weakly identified on short calibration windows: nearly flat likelihood directions allow substantially different coefficients to fit the data similarly. To reduce this instability and obtain a positive-state model, we set $a_0=0$ and impose $a_1,b_0>0$ and $\gamma\in[1/2,1]$, with $2b_0>a_1$ when $\gamma=1/2$. We fit this restricted model using regularized Gaussian quasi-likelihood for log returns, pooling the volatility exponent across stocks and windows within each period and regularizing the drift. For each candidate exponent and drift penalty, we fit the coefficients on the first 200 prices of each window and evaluate the negative log-likelihood on the following 50 returns. We select the pair with the lowest average validation negative log-likelihood across all five stocks and windows in that period. The exponent grid is $\gamma\in\{0.5,0.6,0.7,0.8,0.9,1.0\}$ and the drift-penalty grid is $\{0,0.01,0.1,1,10,100\}$. With the selected pair fixed, we refit $(b_0,b_1,a_1)$ separately for each stock and window using all 250 prices.

We repeat this procedure independently for the three historical datasets: 2008--2019 prices determine the training parameters, 2020--2023 prices determine the validation parameters, and 2024--2025 prices determine the test parameters. Each period has a separately selected exponent and drift penalty and a corresponding collection of fitted coefficient vectors. These collections generate the respective simulated samples; validation and test parameters do not enter the training generator. Test calibration takes place only after model and adversarial-budget selection is complete.

\paragraph{Simulation: \textsc{Fix} and \textsc{Robust}.}
\label{par:gad-simulation}
Following \citet{buehler2019deep,he2025distributional}, we train on simulated trajectories. The \textsc{Fix} specification uses the coefficient vector from the final training-period calibration window throughout every training path. The \textsc{Robust} specification draws one complete coefficient vector uniformly from the stock's rolling estimates, independently for each path and trading step. This differs from the implementation of \citet{he2025distributional}. Sampling the coefficients jointly preserves their dependence within each fitted vector and follows the empirical-parameter alternative discussed by \citet[Remark~4.2]{lutkebohmert2022robust}. The pooled exponent remains fixed within each period. Both specifications use the log-Euler scheme, consistent with the calibration, and produce positive stock prices.

For each stock, each model replicate is trained on 20,000 simulated paths, which are shared across methods and budget candidates. Validation and testing each use a separate sample of 20,000 independent \textsc{Robust} paths per stock, generated from the corresponding period's calibration and shared across methods, training specifications, and model seeds. At each trading step, the evaluation generators sample coefficient vectors only from the corresponding period's rolling estimates, following the same sampling rule as \textsc{Robust} training. The large evaluation samples reduce Monte Carlo variance, since a small number of extreme losses can otherwise dominate the entropic-risk estimate and make comparisons unstable. This improves precision conditional on the fitted models, while calibration uncertainty remains.

\paragraph{Hedging setup and budget selection.}
We hedge the arithmetic Asian call defined in Appendix~\ref{app:claims}, with \(S_0=K=10\), \(T=30\) trading steps, and \(\Delta t=1/250\), using the entropic OCE in Appendix~\ref{app:oce-risk-measures}. The stock is the sole tradable asset and observed state, so adversarial transport acts only on the stock-price path. Under both \textsc{Fix} and \textsc{Robust} training, we compare \textsc{ERM-U}, \(\phi\)\textsc{-U}, \textsc{OT-U}, and \textsc{WRAP-U} using uniform reference weights and the one-step attack in Algorithm~\ref{alg:drift-aware-one-step-attack}. Architecture and optimization settings follow Table~\ref{tab:experiment_hyperparameters}, with 300 training epochs and a 100-epoch clean warmup for adversarial methods.

Training and evaluation use the entropic OCE in \eqref{eq:app-entropic-oce} with \(\lambda=13\) and PnL floored at \(-1\); reported mean PnL is computed without this floor. Clean and adversarial losses receive equal weight after warmup. For each candidate, we retain the epoch with the lowest validation entropic risk and select ambiguity budgets separately for each stock, training distribution, and adversarial specification using mean validation performance over five model seeds. We consider
\[
\tau\in\{0,0.05,0.10,0.15,0.20,0.30,0.40\},
\qquad
\varepsilon\in\{0,0.025,0.05,0.10,0.15,0.20\}.
\]
The selected models are then evaluated on the shared held-out \textsc{Robust} test paths. Reported standard deviations therefore reflect training variability across model seeds rather than market-sampling or calibration uncertainty.

\paragraph{Results.}
Table~\ref{tab:gad_synthetic_validation_optimal} shows that joint adversarial training is consistently stronger than either ambiguity channel in isolation. \textsc{WRAP-U} attains lower validation and test entropic risk than both \(\phi\)\textsc{-U} and \textsc{OT-U} in all ten stock--training comparisons. Relative to \textsc{ERM-U}, \textsc{WRAP-U} lowers test entropic risk in six of ten comparisons. In three additional cases, validation selects the clean origin, so \textsc{WRAP-U} recovers \textsc{ERM-U}. Only \textsc{Fix} BRK-B selects positive adversarial budgets on validation without improving test risk. The results therefore suggest that joint adversarial training can add robustness beyond parameter uncertainty in the training generator while reverting to empirical training when additional robustness is not supported by validation.

The gains vary substantially across stocks and market periods. The largest improvements under \textsc{Robust} occur for AAPL and GOOGL, whereas several comparisons show little or no benefit over \textsc{ERM-U}. This heterogeneity is consistent with the descriptive statistics in Table~\ref{tab:gad_market_descriptive_stats}: validation and test periods differ in both the magnitude and direction of distribution shift across stocks. In particular, lower test-period volatility for MSFT, AMZN, and BRK-B reduces entropic risk for all methods, so lower test scores alone do not indicate a better hedge. Performance should instead be assessed relative to \textsc{ERM-U} on the same test sample.

The remaining performance measures reveal a clear trade-off. \textsc{WRAP-U} reduces turnover relative to \textsc{ERM-U} in all ten comparisons and attains the lowest turnover in nine. In the six cases with lower test entropic risk, mean PnL is also lower. The improvements therefore reflect better control of risk rather than uniformly better average performance, accompanied by less active rebalancing.

\begin{table*}[!ht]
\centering
\scriptsize
\begin{threeparttable}
\caption{Hedging under GAD dynamics calibrated to successive periods of real equity prices.}

\label{tab:gad_synthetic_validation_optimal}
\setlength{\tabcolsep}{3.5pt}
\begin{tabular}{@{}lllcccccc@{}}
\toprule
Training & Stock & Method & $\varepsilon$ & $\tau$ & \shortstack{Val. OCE $(\downarrow)$} & \shortstack{Test OCE $(\downarrow)$} & \shortstack{Mean PnL $(\uparrow)$} & \shortstack{Turnover $(\downarrow)$} \\
\midrule
\textsc{Fix} & AAPL & \textsc{ERM-U} & -- & -- & $\underline{0.287} \pm 0.002$ & $\underline{0.258} \pm 0.001$ & $-0.227 \pm 0.000$ & $0.072 \pm 0.002$ \\
 &  & $\phi$\textsc{-U} & -- & $0.05$ & $0.308 \pm 0.012$ & $0.272 \pm 0.008$ & $-0.227 \pm 0.002$ & $0.089 \pm 0.005$ \\
 &  & \textsc{OT-U} & $0.025$ & -- & $0.299 \pm 0.003$ & $0.267 \pm 0.002$ & $\underline{-0.226} \pm 0.001$ & $0.072 \pm 0.005$ \\
 &  & \textsc{WRAP-U} & $0.05$ & $0.15$ & $0.289 \pm 0.002$ & $0.261 \pm 0.001$ & $\underline{-0.226} \pm 0.001$ & $\underline{0.066} \pm 0.001$ \\
\midrule
\textsc{Fix} & AMZN & \textsc{ERM-U} & -- & -- & $\underline{0.310} \pm 0.001$ & $\underline{0.289} \pm 0.001$ & $-0.255 \pm 0.000$ & $0.069 \pm 0.001$ \\
 &  & $\phi$\textsc{-U} & -- & $0.05$ & $0.328 \pm 0.004$ & $0.301 \pm 0.003$ & $\underline{-0.252} \pm 0.002$ & $0.083 \pm 0.003$ \\
 &  & \textsc{OT-U} & $0.025$ & -- & $0.322 \pm 0.002$ & $0.300 \pm 0.002$ & $-0.254 \pm 0.001$ & $0.068 \pm 0.005$ \\
 &  & \textsc{WRAP-U} & $0.05$ & $0.30$ & $0.314 \pm 0.002$ & $0.294 \pm 0.002$ & $-0.255 \pm 0.001$ & $\underline{0.067} \pm 0.001$ \\
\midrule
\textsc{Fix} & BRK-B & \textsc{ERM-U} & -- & -- & $0.204 \pm 0.005$ & $\underline{0.170} \pm 0.003$ & $-0.142 \pm 0.004$ & $0.100 \pm 0.004$ \\
 &  & $\phi$\textsc{-U} & -- & $0.05$ & $0.263 \pm 0.044$ & $0.203 \pm 0.025$ & $\underline{-0.131} \pm 0.013$ & $0.148 \pm 0.037$ \\
 &  & \textsc{OT-U} & $0.025$ & -- & $0.219 \pm 0.011$ & $0.186 \pm 0.007$ & $-0.151 \pm 0.004$ & $0.086 \pm 0.019$ \\
 &  & \textsc{WRAP-U} & $0.025$ & $0.10$ & $\underline{0.197} \pm 0.005$ & $0.172 \pm 0.006$ & $-0.147 \pm 0.004$ & $\underline{0.085} \pm 0.007$ \\
\midrule
\textsc{Fix} & GOOGL & \textsc{ERM-U} & -- & -- & $0.292 \pm 0.008$ & $0.301 \pm 0.009$ & $-0.243 \pm 0.003$ & $0.086 \pm 0.006$ \\
 &  & $\phi$\textsc{-U} & -- & $0.05$ & $0.327 \pm 0.015$ & $0.340 \pm 0.016$ & $\underline{-0.238} \pm 0.008$ & $0.114 \pm 0.010$ \\
 &  & \textsc{OT-U} & $0.025$ & -- & $0.303 \pm 0.008$ & $0.303 \pm 0.008$ & $-0.248 \pm 0.002$ & $0.074 \pm 0.009$ \\
 &  & \textsc{WRAP-U} & $0.05$ & $0.10$ & $\underline{0.286} \pm 0.002$ & $\underline{0.292} \pm 0.003$ & $-0.248 \pm 0.002$ & $\underline{0.071} \pm 0.004$ \\
\midrule
\textsc{Fix} & MSFT & \textsc{ERM-U} & -- & -- & $0.318 \pm 0.019$ & $0.213 \pm 0.004$ & $-0.175 \pm 0.001$ & $0.088 \pm 0.003$ \\
 &  & $\phi$\textsc{-U} & -- & $0.05$ & $0.363 \pm 0.029$ & $0.226 \pm 0.006$ & $-0.181 \pm 0.005$ & $0.108 \pm 0.009$ \\
 &  & \textsc{OT-U} & $0.025$ & -- & $0.323 \pm 0.011$ & $0.217 \pm 0.004$ & $\underline{-0.174} \pm 0.001$ & $0.076 \pm 0.009$ \\
 &  & \textsc{WRAP-U} & $0.025$ & $0.05$ & $\underline{0.280} \pm 0.002$ & $\underline{0.205} \pm 0.000$ & $-0.176 \pm 0.001$ & $\underline{0.070} \pm 0.004$ \\
\midrule
\textsc{Robust} & AAPL & \textsc{ERM-U} & -- & -- & $0.328 \pm 0.019$ & $0.288 \pm 0.012$ & $-0.223 \pm 0.001$ & $0.091 \pm 0.005$ \\
 &  & $\phi$\textsc{-U} & -- & $0.05$ & $0.352 \pm 0.023$ & $0.296 \pm 0.008$ & $-0.227 \pm 0.002$ & $0.108 \pm 0.005$ \\
 &  & \textsc{OT-U} & $0.025$ & -- & $0.332 \pm 0.015$ & $0.289 \pm 0.011$ & $\underline{-0.222} \pm 0.001$ & $0.077 \pm 0.008$ \\
 &  & \textsc{WRAP-U} & $0.025$ & $0.10$ & $\underline{0.297} \pm 0.002$ & $\underline{0.267} \pm 0.002$ & $-0.224 \pm 0.001$ & $\underline{0.071} \pm 0.002$ \\
\midrule
\textsc{Robust} & AMZN & \textsc{ERM-U} & -- & -- & $0.329 \pm 0.005$ & $0.296 \pm 0.003$ & $\underline{-0.243} \pm 0.001$ & $0.075 \pm 0.002$ \\
 &  & $\phi$\textsc{-U} & -- & $0.05$ & $0.336 \pm 0.011$ & $0.302 \pm 0.006$ & $-0.245 \pm 0.003$ & $0.086 \pm 0.006$ \\
 &  & \textsc{OT-U} & $0.025$ & -- & $0.336 \pm 0.006$ & $0.303 \pm 0.003$ & $-0.248 \pm 0.001$ & $\underline{0.065} \pm 0.004$ \\
 &  & \textsc{WRAP-U} & $0.05$ & $0.15$ & $\underline{0.321} \pm 0.001$ & $\underline{0.294} \pm 0.002$ & $-0.249 \pm 0.001$ & $0.067 \pm 0.002$ \\
\midrule
\textsc{Robust} & BRK-B & \textsc{ERM-U} & -- & -- & $\underline{0.190} \pm 0.003$ & $\underline{0.161} \pm 0.002$ & $\underline{-0.137} \pm 0.002$ & $0.095 \pm 0.002$ \\
 &  & $\phi$\textsc{-U} & -- & $0.05$ & $0.223 \pm 0.013$ & $0.183 \pm 0.008$ & $-0.139 \pm 0.007$ & $0.142 \pm 0.019$ \\
 &  & \textsc{OT-U} & $0.025$ & -- & $0.206 \pm 0.005$ & $0.173 \pm 0.006$ & $-0.142 \pm 0.004$ & $0.094 \pm 0.014$ \\
 &  & \textsc{WRAP-U} & $0.025$ & $0.05$ & $0.193 \pm 0.004$ & $0.167 \pm 0.004$ & $-0.143 \pm 0.003$ & $\underline{0.087} \pm 0.009$ \\
\midrule
\textsc{Robust} & GOOGL & \textsc{ERM-U} & -- & -- & $0.325 \pm 0.023$ & $0.339 \pm 0.030$ & $\underline{-0.233} \pm 0.006$ & $0.105 \pm 0.010$ \\
 &  & $\phi$\textsc{-U} & -- & $0.05$ & $0.392 \pm 0.028$ & $0.420 \pm 0.033$ & $-0.246 \pm 0.019$ & $0.137 \pm 0.013$ \\
 &  & \textsc{OT-U} & $0.025$ & -- & $0.332 \pm 0.024$ & $0.333 \pm 0.018$ & $-0.242 \pm 0.004$ & $0.092 \pm 0.016$ \\
 &  & \textsc{WRAP-U} & $0.10$ & $0.40$ & $\underline{0.297} \pm 0.007$ & $\underline{0.302} \pm 0.005$ & $-0.245 \pm 0.003$ & $\underline{0.077} \pm 0.005$ \\
\midrule
\textsc{Robust} & MSFT & \textsc{ERM-U} & -- & -- & $0.310 \pm 0.017$ & $0.212 \pm 0.005$ & $\underline{-0.176} \pm 0.001$ & $0.089 \pm 0.005$ \\
 &  & $\phi$\textsc{-U} & -- & $0.10$ & $0.369 \pm 0.012$ & $0.235 \pm 0.004$ & $-0.178 \pm 0.003$ & $0.124 \pm 0.008$ \\
 &  & \textsc{OT-U} & $0.025$ & -- & $0.316 \pm 0.008$ & $0.213 \pm 0.003$ & $\underline{-0.176} \pm 0.001$ & $0.080 \pm 0.005$ \\
 &  & \textsc{WRAP-U} & $0.05$ & $0.15$ & $\underline{0.285} \pm 0.004$ & $\underline{0.208} \pm 0.003$ & $-0.178 \pm 0.003$ & $\underline{0.075} \pm 0.004$ \\
\bottomrule
\end{tabular}
\begin{tablenotes}[flushleft]
\scriptsize
\item \textit{Note.} Values are mean $\pm$ sample standard deviation over five model seeds. The \textsc{WRAP-U} rows exclude the clean origin $(\varepsilon,\tau)=(0,0)$, reported separately as \textsc{ERM-U}.

\end{tablenotes}
\end{threeparttable}
\end{table*}

\subsubsection{One-step and Multistep \textsc{WRAP}}
\label{app:additional-results}

We compare one-step and multistep \textsc{WRAP} (Algorithms~\ref{alg:drift-aware-one-step-attack} and~\ref{alg:drift-aware-wbpgd}) using the stationary Heston setup above and the Heston training settings in Table~\ref{tab:experiment_hyperparameters}. Budgets are selected separately for each attack and specification according to validation CVaR.

\emph{Results.} Table~\ref{tab:heston-stationary-multi-step-WRAP} shows that both attacks select the same budgets within each specification. Both \textsc{OT-U} and \textsc{WRAP-U} select $(\varepsilon,\tau)=(0.10,0)$, so the selected WRAP configuration uses transport alone for both attacks. One-step training achieves slightly lower test CVaR than multistep training ($3.543$ versus $3.545$), with a difference smaller than the reported seed standard deviations. The additional inner iterations therefore provide no observed improvement at the selected budgets, supporting the use of the cheaper one-step attack in this setting.

\begin{table*}[!ht]
\centering
\scriptsize
\begin{threeparttable}
\caption{Results for one-step and multistep \textsc{WRAP} in the stationary Heston setting.}
\label{tab:heston-stationary-multi-step-WRAP}
\setlength{\tabcolsep}{3.5pt}
\begin{tabular}{@{}llcccccc@{}}
\toprule
Attack & Method & $\varepsilon$ & $\tau$ & \shortstack{Val. CVaR $(\downarrow)$} & \shortstack{Test CVaR $(\downarrow)$} & \shortstack{Mean PnL $(\uparrow)$} & \shortstack{Turnover $(\downarrow)$} \\
\midrule
Baseline & \textsc{ERM-U} & -- & -- & $3.634 \pm 0.013$ & $3.629 \pm 0.018$ & $-2.199 \pm 0.005$ & $\underline{3.084} \pm 0.124$ \\
\midrule
One-step WRAP & $\phi$\textsc{-U} & -- & $0.25$ & $3.627 \pm 0.014$ & $3.629 \pm 0.025$ & $\underline{-2.198} \pm 0.004$ & $3.283 \pm 0.266$ \\
 & \textsc{OT-U} & $0.10$ & -- & $\underline{3.550} \pm 0.011$ & $\underline{3.543} \pm 0.016$ & $-2.199 \pm 0.004$ & $3.156 \pm 0.064$ \\
 & \textsc{WRAP-U} & $0.10$ & -- & $\underline{3.550} \pm 0.011$ & $\underline{3.543} \pm 0.016$ & $-2.199 \pm 0.004$ & $3.156 \pm 0.064$ \\
\midrule
Multistep-WRAP & $\phi$\textsc{-U} & -- & $0.25$ & $3.627 \pm 0.014$ & $3.629 \pm 0.025$ & $\underline{-2.198} \pm 0.004$ & $3.283 \pm 0.266$ \\
 & \textsc{OT-U} & $0.10$ & -- & $3.552 \pm 0.011$ & $3.545 \pm 0.015$ & $-2.200 \pm 0.004$ & $3.167 \pm 0.060$ \\
 & \textsc{WRAP-U} & $0.10$ & -- & $3.552 \pm 0.011$ & $3.545 \pm 0.015$ & $-2.200 \pm 0.004$ & $3.167 \pm 0.060$ \\
\bottomrule
\end{tabular}
\begin{tablenotes}[flushleft]
\scriptsize
\item \textit{Note.} Values are mean $\pm$ standard deviation over five seeds.
\end{tablenotes}
\end{threeparttable}
\end{table*}

\newpage

\subsection{Pathwise Effects of Reweighting and Optimal Transport}
\label{subsec:attack-anatomy}

Figure~\ref{fig:attack-anatomy}, reproduced from Figure~\ref{fig:attack-anatomy-main}, illustrates how the two budgets in Section~\ref{sec:short-two-budget} affect scenario probabilities and market trajectories in the nonstationary Heston experiment with joint drift. Panels~(A)--(C) show a random sample of reference trajectories, together with the highest-weight and most-displaced trajectories highlighted separately; panel~(C) shows their transported versions. Stock prices are normalized by the strike price. Panel~(A) shows the original paths and the pointwise 5th and 95th percentiles under uniform reference probabilities $w_n=1/N$. The quantiles in panels~(A)--(C) are computed from all $N=4000$ validation trajectories using the corresponding scenario probabilities.

Panel~(B) isolates reweighting with divergence budget $\tau=0.6$ and transport budget $\varepsilon=0$. The adversary replaces $w_n$ by $\widehat w_n$, allocating greater probability to trajectories with larger loss integrands $L_n(\vartheta)$. This changes the weighted quantiles while leaving the trajectories unchanged. The path labeled ``Highest-weight path'' has the largest adversarial probability $\widehat w_n$. The resulting probability concentration is summarized by the effective sample size,
\[
N_{\mathrm{eff}}(\widehat w)
=\frac{1}{\sum_{n=1}^N(\widehat w_n)^2}
\approx691,
\]
compared with $N_{\mathrm{eff}}(w)=4000$ before reweighting. The reduced effective sample size reflects the concentration of adversarial probability on a subset of scenarios. The same construction accommodates nonuniform reference probabilities $w_n$ in the \textsc{-W} specification.

Panel~(C) adds optimal transport with budget $\varepsilon=0.3$, holding the reweighted probabilities fixed. The attack perturbs trajectories in directions that increase the linearized loss, with perturbations scaled according to local loss sensitivity and the transport geometry. Both stock prices and variances are perturbed; the figure displays the stock-price component. Panel~(D) provides a magnified view of the region highlighted in (C), comparing the stock-price components of the original and transported versions of the most displaced trajectory. This trajectory has the largest displacement among all $N$ validation trajectories with respect to the ground metric $d_{\kappa,\lambda}$ in \eqref{eq:kappa-lambda-metric}.

\begin{figure}[ht]
\centering
\includegraphics[width=\textwidth]{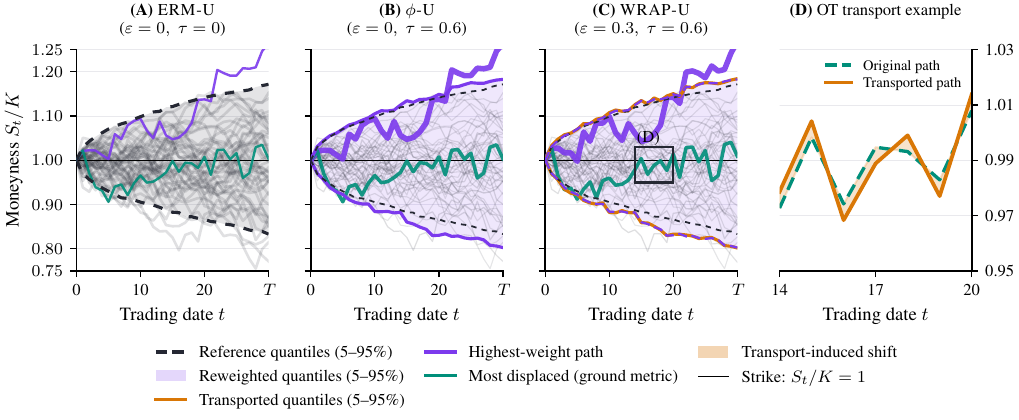}
\caption{Pathwise effects in nonstationary Heston:
(A) \textsc{ERM-U}, (B) $\phi$\textsc{-U}, (C) \textsc{WRAP-U}, and (D) magnified view of the original and transported paths highlighted in (C). Panels~(A)--(C) show sampled trajectories and highlight the highest-weight and most-displaced paths.}
\label{fig:attack-anatomy}
\end{figure}

\end{document}